\documentclass[12pt]{article}
\pdfoutput=1
\usepackage{mathtools}
\usepackage{amsthm}
\usepackage{array}
\usepackage{relsize}
\usepackage{amssymb}
\usepackage{amsmath}
\usepackage{natbib}
\usepackage{comment}
\usepackage{tabularx}
\usepackage[margin=1in]{geometry}

\providecommand{\norm}[1]{\lVert#1\rVert}

\newcommand\numberthis{\addtocounter{equation}{1}\tag{\theequation}}
\usepackage{booktabs}
\usepackage{hyperref}

\newtheorem{thm}{Theorem}
\newtheorem{problem}{Problem}

\newtheorem{assum}{Assumption}

\newtheorem{lemma}{Lemma}
\newtheorem{definition}{Definition}
\newtheorem{proposition}[thm]{Proposition}

\usepackage{xcolor}        
\usepackage{hyperref}

\DeclareMathOperator*{\argmin}{arg\,min}

\usepackage{algorithm}
\usepackage[noend]{algpseudocode}

\renewcommand{\P}{\mathbb{P}} 
\usepackage{amsmath}
\usepackage{amssymb}
\usepackage{amsfonts}   

\DeclareMathOperator*{\E}{\mathbb{E}}
\DeclareMathOperator*{\Var}{Var}
\DeclareMathOperator*{\var}{Var}
\newcommand{\Sprime}{S''}

\newcommand{\cond}{\; \middle| \;}

\title{Foundations of Safe Online Reinforcement Learning in the Linear Quadratic Regulator: Generalized Baselines}

\author{
	Benjamin Schiffer 
 and
    Lucas Janson 
}

\date{Department of Statistics, Harvard University}

\begin{document}
\maketitle

\begin{abstract}
Many practical applications of online reinforcement learning require the satisfaction of safety constraints while learning about the unknown environment. In this work, we establish theoretical foundations for reinforcement learning with safety constraints by studying the canonical problem of Linear Quadratic Regulator learning with unknown dynamics, but with the additional constraint that the position must stay within a safe region for the entire trajectory with high probability. Our primary contribution is a general framework for studying stronger baselines of nonlinear controllers that are better suited for constrained problems than linear controllers. Due to the difficulty of analyzing non-linear controllers in a constrained problem, we focus on 1-dimensional state- and action- spaces, however we also discuss how we expect the high-level takeaways can generalize to higher dimensions.  Using our framework, we show that for \emph{any} non-linear baseline satisfying natural assumptions, $\tilde{O}_T(\sqrt{T})$-regret is possible when the noise distribution has sufficiently large support, and $\tilde{O}_T(T^{2/3})$-regret is possible for \emph{any} subgaussian noise distribution. In proving these results, we introduce a new uncertainty estimation bound for nonlinear controls which shows that enforcing safety in the presence of sufficient noise can provide “free exploration” that compensates for the added cost of uncertainty in safety-constrained control.

\end{abstract}

\newpage
\section{Introduction}

\subsection{Background and Motivation}

Recent advances in reinforcement learning (RL) have led to many successes in applying RL algorithms to a variety of practical online applications, from robotics to personalized health \citep{levine2016end, lillicrap2015continuous,tewari2017ads}. A core concept behind online RL algorithms is the careful balance between exploration (proactively learning about the unknown environment) and exploitation (using what is already known to maximize reward). In practice, however, RL algorithms are restricted in the possible actions and states by safety constraints. For example, a drone using an RL algorithm must have safety constraints restricting possible states that would result in the drone crashing into a building or injuring a bystander. Therefore, the drone cannot explore the environment by accelerating directly into a building, and instead must explore in a safe manner. To deploy more RL algorithms to practical applications, instead of just balancing exploration versus exploitation, the optimal algorithm must now balance exploration versus exploitation versus safety. In many applications, the safety constraints must be obeyed at all time steps (even at the beginning), which does not allow for any violation of safety even during the initial learning period. Therefore, this component of ``safety" involves both learning safely as well as learning how to be safe in the future. Studying simple canonical problems in RL can give insights into how to develop safe RL algorithms in more complex practical settings. In this paper, we address safety in the context of online LQR with unknown dynamics. Online LQR with unknown dynamics can be viewed as one of the simplest RL problems with a continuous decision space, and this problem has recently gained significant attention within the RL community both with and without safety constraints (see e.g. \citet{abbasi2011regret, dean2018regret, dean2019safely}).

\subsection{Setting and Motivation}

In order to better understand the interaction between safety and the balance of exploration/exploitation, we study the classic problem of controlling a discrete-time linear dynamical system with unknown dynamics while minimizing a quadratic cost. In our problem setting, the position at the next time step depends on the current position, the current control input, and a random noise. The goal is to choose controls (actions) that keep the position as close to the origin as possible while using as little control as possible. An example application of this problem is controlling a drone around a target \citep{rubio2016optimal}. In this scenario, the goal is to maintain a safe distance from the target while preserving fuel despite random disturbances from air currents. In this paper, we are interested in the setting where the dynamics are unknown. When the dynamics are unknown, LQR becomes an online RL problem of balancing exploration (controls that learn about the dynamics) and exploitation (controls that minimize the cost). Extending the previous example of controlling a drone around a target, the dynamics could for example be determined by the weather pattern that is unknown in advance. The goal in this paper is to design an algorithm that can learn the dynamics safely while not incurring significantly more cost than the best safe algorithm when the dynamics are known.

To quantify safety in this setting, we will consider constraints on the position of the controller which restrict the position to stay within a safe region. Continuing the previous example, a drone control must be safe in that it must avoid positions that are currently occupied by walls or other objects. We focus on position constraints rather than control constraints because position constraints have the added difficulty that, at the time of choosing the control, the next position for any given control is unknown due to the noise and uncertainty about the dynamics. In contrast, the algorithm has perfect information about (and control over) the choice of control. We therefore consider the LQR setting with only position constraints. See Section \ref{sec:discussion} for discussion on how our results extend to the setting with control constraints. While the optimal policy with known dynamics and without position constraints is the well-understood Linear Quadratic Regulator, with constraints the optimal policy even for known dynamics no longer has a closed-form \citep{rawlings2012postface}. Due to the substantially increased complexity of the constrained LQR problem with both known and unknown dynamics, we will focus on the setting when both the positions and controls are one-dimensional. We focus on the one-dimensional setting to highlight the main ways in which learning unknown dynamics changes in the presence of constraints, without the additional technical overhead that comes with proving results for higher dimensions. However, we do predict that many of the results in this paper can be generalized to higher dimensions, and we discuss this further in Section \ref{sec:discussion}. Other works have also taken the same approach of first studying only the one-dimensional case of LQR, see e.g. \citet{fefferman2021optimal, abeille2017thompson}. The one-dimensional setting of safe LQR does have its own applications, for example maintaining a fixed temperature of a room \citep{oldewurtel2008tractable}. In this application, the goal is to maintain a certain range of safe temperatures with high probability while using as little energy as possible. Taking temperature as the position, this problem can be formulated as a one-dimensional LQR problem with safety ``position" constraints on the temperature.

\subsection{Our Contribution}
The main theorems of this paper each establish new regret results for safety-constrained LQR learning. We improve prior works' regret bounds for this setting along three dimensions, the regret rate, the regret baseline, and the types of noise distributions. In contrast to prior works, we focus on one-dimensional LQR with only positional constraints, but in this setting we study more general non-linear baselines. The following table summarizes our different results relative to prior works across these three dimensions:
\vspace{5mm}

\scalebox{.9}{
\begin{tabular}{@{}llll@{}}
\toprule
               & Regret Rate                               & Regret Baseline                    & Noise Distributions                          \\ \midrule
Previous works & $\tilde{O}_T(T^{2/3})$                    & Best Safe Linear Controller           & Bounded                              \\
Theorem \ref{sufficiently_large_error}        & $\tilde{O}_T(\sqrt{T})$                   & Best General Baseline Controller                      & subgaussian+Large Support \\
Theorem \ref{performance}       & $\tilde{O}_T(T^{2/3})$ & Best General Baseline Controller                      & subgaussian                         \\ \bottomrule
\end{tabular}
}

\vspace{5mm}

The main contribution of this paper is a general framework for analyzing different baselines for the safety-constrained LQR learning problem. Using this framework, we show that a $\tilde{O}_T(\sqrt{T})$ rate of regret is possible for safety-constrained LQR learning in one-dimension for noise distributions with large support, improving on $\tilde{O}_T(T^{2/3})$ regret of previous works \citep{li2021safe, dean2019safely}. This rate of regret for constrained LQR learning matches the optimal regret rate for \emph{unconstrained} LQR learning \citep{ziemann2024regret}. In addition to improving the rate of regret, this result is also with respect to a stronger and more general baseline than studied in previous works. The regret for this result is defined with respect to the best controller from general classes of baseline controllers satisfying only minimal regularity condition. To the best of our knowledge, this is the first work on constrained LQR learning with respect to any baseline other than the best safe linear controller. Our result also holds for any subgaussian noise distribution, which is the (to the best of our knowledge) first safety-constrained LQR learning result for unbounded distributions. A key technical tool used to prove $\tilde{O}_T(\sqrt{T})$ regret is a new bound for estimating unknown dynamics with non-linear controllers, which may be of independent interest.

 In addition to showing that $\tilde{O}_T(\sqrt{T})$ regret is possible when the noise distribution has sufficiently large support, we also show that a certainty equivalence algorithm can achieve a regret rate of $\tilde{O}_T(T^{2/3})$ relative to the best controller from these general classes of baseline controllers for any subgaussian noise distribution. All of the proofs of the regret results in this paper are constructive and provide certainty equivalence algorithms for achieving the guaranteed rates of regret.

\subsection{Related Work}\label{sec:related_work}
RL has been recognized as being a powerful tool in a broad array of applications \citep{silver2016mastering, kiran2021deep, levine2016end}, but there is still a need to better understand RL in the presence of safety constraints. There exists a wide array of definitions of safety in RL, many of which focus on some notion of reachability or stability  \citep{ganai2024iterative, garg2024learning, gu2022review, moldovan2012safe, wachi2018safe, wachi2024survey, yao2024constraint}. However, these notions of safety are less directly related to our problem setting. More related to our problem, there is also a body of literature on algorithms for RL for control with constraints that maintain safety for the entire trajectory  \citep{fulton2018safe, cheng2019end, marvi2021safe, fisac2018general}. These works study different broad definitions of safety in control, which can apply to a wider variety of models and settings than our results. However, the technical contribution of these works focuses specifically on developing safe algorithms, without proving theoretical results about the rates of regret or the optimality of the proposed safe algorithms. 

The LQR problem has many applications despite the simplicity of the problem statement \citep{priess2014solutions, choi1999lqr, shabaani2003application}. There has recently been a large body of work focusing on minimizing regret in the unconstrained LQR setting with unknown dynamics, beginning with \citet{abbasi2011regret} which gave the first algorithm for $\tilde{O}_T(\sqrt{T})$ regret for unconstrained LQR learning. This was followed by many works that study variations of both the infinite and finite time problem including (but not limited to) \citet{dean2018regret, mania2019certainty, mania2020active, simchowitz2018learning,cohen2019learning,wang2021exact,wang2022rate,mania2019certainty,abeille2017thompson,zheng2020non,sun2020finite,khosravi2020nonlinear,sattar2022non,faradonbeh2018input,faradonbeh2017finite,oymak2019non,ye2024online,athrey2024regret,ziemann2024regret,lee2024nonasymptotic}. Certainty Equivalence (CE) algorithms estimate the unknown dynamics and find an optimal policy under the estimated dynamics. Later works on LQR learning showed that CE algorithms are in fact (rate) optimal for the unconstrained learning problem \citep{simchowitz2020naive,faradonbeh2018optimality,mania2019certainty,wang2022rate}. 
There are also some connections between our work and the areas of model predictive control and system identification, but we defer these to the appendix (Appendix \ref{app:related_works}) in the interest of space because the connections to our work are not as strong as the works surveyed in the rest of this subsection.

The two previous works that are most closely related to this paper are \citet{dean2019safely} and \citet{li2021safe}, which both study safety-constrained LQR learning with unknown dynamics. Both works study the regret with respect to the baseline of the best linear controllers of the form $u_t = -Kx_t$ and derive an upper bound of $\tilde{O}_T(T^{2/3})$ on the regret. In \citet{dean2019safely}, they use system level synthesis to develop an algorithm that can safely inject noise into the system to give statistical guarantees on the learning rate.  \citet{li2021safe} provide the first adaptive learning algorithm for constrained LQR learning with unknown dynamics using a CE approach. While their results hold for higher dimensional LQR, our results improve on theirs in two ways. First, we are able to show a regret rate of $\tilde{O}_T(\sqrt{T})$ for some noise distributions, an improvement over their regret rate of $\tilde{O}_T(T^{2/3})$. Second, our regret results are with respect to a significantly stronger and more general baseline. These previous works focused on regret with respect to the best safe linear controller. However, the class of safe linear controllers is a relatively weak class of safe controllers, and the best safe linear controller can be far worse than the best overall safe controller. See Section \ref{sec:baseline} for more discussion on the importance of the choice of baseline. Note that these works allow constraints on both control and positions, while our results focus only on positional constraints. See Section \ref{sec:discussion} for more discussion on control constraints.

This work is the first part of a two part series of papers on safe LQR learning. In this paper, we provide a general framework for studying baselines of non-linear controllers in this problem. The second part of the series \citep{schiffer2025stronger} uses the general framework from this paper to study a specific baseline of non-linear controllers known as the truncated linear controllers. \citet{schiffer2025stronger} shows that the assumptions proposed in this paper do actually hold for a very natural baseline, and therefore our general framework can be applied to that baseline. Using the current paper's framework, \citet{schiffer2025stronger} shows an even stronger result holds for that paper's specific baseline using a more complex algorithm. The additional results in \citet{schiffer2025stronger} required significant extra technical work that would have pushed this paper's already lengthy technical appendix to an inaccessible length, which is why we split the results into a two part series.

\section{Preliminaries}

\subsection{Outline of Preliminaries}

In order to formally state our problem, the preliminaries section is organized as follows. First, in Section \ref{sec:dynamics}, we outline the dynamics of the system and the notation we will use for controllers. In Section \ref{sec:constraints}, we define and motivate the expected-position safety constraints we use to represent safety throughout the paper. In order for it to be possible to learn safely, we also need some initial information. In Section \ref{sec:initial_uncertainty}, we outline the exact assumptions we make on the initial uncertainty. Finally, in Section \ref{sec:problem_statement}, we put everything from the previous sections together with a definition of regret to formally state our problem.

\subsection{Problem Dynamics}\label{sec:dynamics}

Denote the state of the system at time $t$ for $t \in [T]$ as $x_t \in \mathbb{R}$ and the control at time $t$ as $u_t \in \mathbb{R}$. For simplicity, we will assume that the system starts at position $x_0 = 0$. The position at time $t+1$ follows dynamics $x_{t+1} = a^*x_t + b^*u_t + w_t$,
where $a^* \in \mathbb{R}$ and $b^* \in \mathbb{R}$ determine the dynamics and $w_t \stackrel{\text{i.i.d.}}{\sim} \mathcal{D}$ is the noise term drawn from a continuous, mean-$0$ probability distribution $\mathcal{D}$ with cumulative distribution function $F_{\mathcal{D}}$ and variance $\sigma_{\mathcal{D}}^2 = 1$.  We will consider the quadratic cost at time $t$ as $qx_t^2 + ru_t^2$ for $q, r \in \mathbb{R}_{>0}$, and consider the sum of cost over the first $T$ steps. Throughout this paper, we will assume that the dynamics $a^*,b^*$ are unknown, while all other problem parameters are known (e.g. $\mathcal{D}, q,r$, etc.). For simplicity, we will denote the unknown dynamics as $\theta^* = (a^*,b^*)\in\mathbb{R}^2$.

We will also use the following controller notation. Define $H_t = (x_0,u_0,x_1,..., u_{t-1}, x_t)$, and $\mathcal{F}_t = \sigma(H_t)$, the sigma algebra generated by $H_t$. We define a (possibly time-dependent and randomized) controller $C$ such that the control chosen at time $t$ is $u_t = C(H_t)$. Note that any randomness in the controller $C$ must be independent of the noise random variables $\{w_t\}_{t=0}^{T-1}$. Define the $T$-step \textit{cost} of a controller $C$  starting at position $x_0$ under dynamics $\theta$ with noise random variables $W = \{w_t\}_{t=0}^{T-1}$ as
\begin{equation}\label{eq:cost_def}
J(\theta, C, T, x_0, W)  = \frac{1}{T}\left(qx_T^2 +  \sum_{t=0}^{T-1} qx_t^2 + ru_t^2\right),
\end{equation}
\[
 \text{where }  u_t = C(H_t),\;\,  x_{t+1} = ax_t + bu_t + w_t,\;\, w_t \stackrel{i.i.d.}{\sim} \mathcal{D}.
\]
Notice that $J$ outputs an average cost. We will refer to $T \cdot J(\theta, C, T, x_0, W)$ as the \textit{total cost}. We denote $J^*(\theta, C, T, x_0)$ as the expectation of $J(\theta, C, T, x_0, W)$ with respect to only the randomness in $W$. Formally, this means that $J^*(\theta, C, T, x_0)  = \E\left[J(\theta, C, T, x_0, W) \mid \theta, C, T, x_0\right]$ in case any of $\theta$, $C$, $T$, and $x_0$ are random, but in the typical setting when $\theta$, $C$, $T$, and $x_0$ are all deterministic, $J^*(\theta, C, T, x_0)$ will be non-random. For notational simplicity, we also define $J^*(\theta,C,T) = J^*(\theta,C,T,0)$.

\subsection{Constraints}\label{sec:constraints}

Now we will formalize our positional constraints.  Both \citet{dean2019safely} and \citet{li2021safe} formulate their positional constraints as \textit{realized-position constraints} of the form 
\begin{equation}\label{eq:intro_safety_hard}
    D_{\mathrm{L}}^{x} \le x_t \le D_{\mathrm{U}}^{x},
\end{equation}
which must be satisfied with probability $1$ when the dynamics are known. Realized-position constraints that hold with probability $1$ have the easy interpretation that the realized position must never exceed the realized-position boundaries given by the user of the algorithm. However, in the case of unbounded noise distributions (for example Gaussian noise), having the realized position never exceed any compact set with probability $1$ is impossible even with known dynamics. This is because with Gaussian noise, there is always a strictly positive probability that $x_t$ will be outside of the safe region $[D_{\mathrm{L}}^{x}, D_{\mathrm{U}}^{x}]$ for any choice of control $u_{t-1}$. Therefore, in order to allow for unbounded noise distributions, we must relax the requirement of never exceeding the constraints with probability $1$, and instead allow the position $x_t$ to exceed the realized-position boundaries $D_{\mathrm{L}}^{x}$ and $D_{\mathrm{U}}^{x}$ with probability at most $\delta_{\mathrm{traj}}$ throughout the entire trajectory. Using a union bound, one way to achieve this relaxation for $T$ steps is to require that for every $t$,
\begin{equation}\label{eq:Dconvert}
    D_{\mathrm{L}}^{x} - F_{\mathcal{D}}^{-1}\left(\frac{\delta_{\mathrm{traj}}}{2T}\right) \le a^*x_t + b^*u_t \le D_{\mathrm{U}}^{x} - F_{\mathcal{D}}^{-1}\left(1-\frac{\delta_{\mathrm{traj}}}{2T}\right).
\end{equation}
Motivated by this result, we will formulate our problem in terms of \textit{expected-position constraints} of the form
\begin{equation}\label{eq:intro_safety}
    D_{\mathrm{L}}^{\E[x]} \le a^*x_t + b^*u_t \le D_{\mathrm{U}}^{\E[x]}.
\end{equation}
Because $\mathcal{D}$ is mean-$0$, this expected-position constraint has the easy interpretation of constraining the expected position, conditional on the history, at every time point (hence the $\E[x]$ superscript). By constraining the expected position, we are also implicitly constraining the realized position $x_{t}$ to be within the random interval $[D_{\mathrm{L}}^{\E[x]} + w_{t-1}, D_{\mathrm{U}}^{\E[x]} + w_{t-1}]$. Furthermore, if the noise distribution has support  $[-\bar{w}, \bar{w}]$ and $\delta_{\mathrm{traj}} = 0$ (as in \citet{dean2019safely} and \citet{li2021safe}), then realized-position constraints are a special case of the expected-position constraints:
Equation \eqref{eq:intro_safety_hard} with realized-position boundaries $D^{x} := \big(D_{\mathrm{L}}^{x}, D_{\mathrm{U}}^{x}\big)$ is equivalent to Equation \eqref{eq:intro_safety} with expected-position boundaries $D^{\E[x]} := (D^{\E[x]}_{\mathrm{L}},D^{\E[x]}_{\mathrm{U}}) = \big(D_{\mathrm{L}}^{x} + \bar{w}, D_{\mathrm{U}}^{x} - \bar{w}\big)$. 
For unbounded noise, Equation \eqref{eq:intro_safety_hard} is impossible to satisfy with probability $1$, while Equation \eqref{eq:intro_safety} is possible to satisfy and is directly related to the problem of satisfying the realized-position constraints with high probability. Therefore, the constraints in Equation \eqref{eq:intro_safety} can in some sense be thought of as a generalization of the realized-position constraints in Equation \eqref{eq:intro_safety_hard}. To maintain that $0$ is a safe position, we will also require that $D_{\mathrm{L}}^{\E[x]} < 0 < D_{\mathrm{U}}^{\E[x]}$ (see Assumption \ref{problem_specifications}).

In order to satisfy the realized position constraints in Equation \eqref{eq:intro_safety_hard} for all $T$ steps with constant probability, the magnitude of the boundaries must scale with the max position, which scales with the magnitude of the largest realized noise. For an unbounded distribution $\mathcal{D}$, this means that the realized-position boundaries must be a function of $T$ that grows with $T$. Now looking at Equation \eqref{eq:Dconvert}, the implied expected-position constraints include both $D^{x}$ (which may be a function of $T$) and a quantile of the noise distribution (which is explicitly a function of $T$). Therefore, we will allow the expected-position boundary $D^{\E[x]}$ of Equation \eqref{eq:intro_safety} to depend on $T$. However, in the typical feasible safe RL problem we will have expected-position boundaries that are $O_T(1)$. The reason for this is that the expected-position constraints only bound the position in expectation. Therefore, unlike the realized-position boundary which must scale with the maximum position in order to be feasible, the expected-position boundary is feasible as long as it scales with the largest product of position and dynamics estimation error (uncertainty in $\theta$). Under the assumptions in this paper, we will achieve an estimation error that decreases at a rate that is much faster than the rate at which the maximum position grows. Thus, while we allow the expected-position boundaries to be functions of $T$, the reader should generally think of them as not growing with $T$ in a typical problem, and indeed some of our results will explicitly require 
the expected-position boundaries to be $O_T(1)$.

Formally, we define safety as follows. Note that when the boundaries $(D_{\mathrm{L}}^{\E[x]},D_{\mathrm{U}}^{\E[x]})$ are clear in context, we will drop the constraints and simply refer to algorithms that are safe for a specific dynamics $\theta^*$.

\begin{definition}\label{def:safe}
    A series of controls $\{u_t\}_{t=0}^{T-1}$ are \textit{safe} for dynamics $\theta^*$ and boundaries $(D_{\mathrm{L}}^{\E[x]},D_{\mathrm{U}}^{\E[x]})$  if every control satisfies Equation \eqref{eq:intro_safety}. Similarly, a controller $C$ is \textit{safe} for dynamics $\theta^*$ and boundaries $(D_{\mathrm{L}}^{\E[x]},D_{\mathrm{U}}^{\E[x]})$ if the resulting controls $\{C(H_t)\}_{t=0}^{T-1}$ under true dynamics $\theta^*$ are safe for dynamics $\theta^*$.
\end{definition}

\subsection{Initial Uncertainty Assumptions}\label{sec:initial_uncertainty}
Without any prior knowledge about the unknown dynamics $\theta^*$, it is impossible to choose a first action that is guaranteed to be safe for all $\theta^* \in \mathbb{R}^{2}$. Therefore, to learn anything about the unknown dynamics while maintaining safety, we require some initial information about the unknown dynamics. Before getting into our main results, we will therefore formalize our assumptions about the initial uncertainty in our problem. As is standard in previous works \citep{abbasi2011regret, li2021safe}, we will assume the following:
\begin{assum}\label{assum_init}
    The algorithm has access to some $\Theta = \Theta_a \times \Theta_b =  [\underline{a}, \bar{a}] \times [\underline{b}, \bar{b}]$ such that $\theta^* \in \Theta$ and $\bar{b} \ge \underline{b} > 0$ and $\bar{a} \ge \underline{a} > 0$.
\end{assum}
$\Theta$ can be thought of as the initial uncertainty set for $\theta^*$. Define the size of such a set $\Theta$ as $\mathrm{size}(\Theta) = \max(\bar{a} - \underline{a}, \bar{b} - \underline{b})$. Note that depending on the size of $\Theta$, maintaining safety with respect to the expected-position boundaries for any $\theta^* \in \Theta$ may be infeasible. Infeasible in our setting means that there does not exist any adaptive controller $C$ such that for all $\theta^* \in \Theta$, the controller is safe with high probability, i.e. $\P\left(\forall t < T : D_{\mathrm{L}}^{\E[x]} \le a^*x_t + b^*C(H_t) \le D_{\mathrm{U}}^{\E[x]} \right) \ge 1-\delta$. Clearly feasibility of $\Theta$ (for some appropriate choice of $\delta$) is a necessary condition for our problem to have a solution. The assumptions we make are only slightly stronger than just feasibility, which we discuss further in Appendix \ref{app:infeasibility}. As described in Section \ref{sec:related_work}, many previous works have developed algorithms that maintain guaranteed safety, but to the best of our knowledge the exact amount of prior information needed has not been quantified. 

The assumption that $a^*,b^* > 0$ is for algebraic convenience, and the same results can be shown for any constant $a^*,b^* \in \mathbb{R}$. The assumption that $\underline{a}, \underline{b} > 0$ can actually be removed given the next assumption, and we discuss this more in Appendix \ref{app:relax_assum_init}. 

The other main assumption about prior information that we make is that we have sufficient information to not violate the safety constraint for some initial period of the algorithm.

\begin{assum}\label{assum:initial}
    There is a known controller $C^{\mathrm{init}}$ such that $\forall x \in \left[ D_{\mathrm{L}}^{\E[x]} + F_{\mathcal{D}}^{-1}(\frac{1}{T^4}), D_{\mathrm{U}}^{\E[x]} + F_{\mathcal{D}}^{-1}(1-\frac{1}{T^4})\right] $,
    \begin{equation}\label{eq:assum_initial}
        D_{\mathrm{L}}^{\E[x]} + \frac{b^*}{\log(T)} \le  a^*x + b^*C^{\mathrm{init}}(x) \le D_{\mathrm{U}}^{\E[x]} - \frac{b^*}{\log(T)}.
    \end{equation}
\end{assum}

To get a sense of how strong Assumption \ref{assum:initial} is, note that if we ignore the vanishing log terms in Equation \eqref{eq:assum_initial}, then Assumption \ref{assum:initial} is equivalent to assuming that we can identify any safe controller. If this is not the case, then safe learning is clearly impossible. We further discuss Assumption \ref{assum:initial} and how it relates to the concept of feasibility in Appendix \ref{app:infeasibility}. In Appendix \ref{app:assum_initial}, we also provide further interpretation of Assumption \ref{assum:initial} in the case of bounded noise.

\subsection{Problem Statement}\label{sec:problem_statement}

We define $\mathcal{C}^{\theta^*}$ as a baseline class of controllers if every controller  $C \in \mathcal{C}^{\theta^*}$ is safe with respect to dynamics $\theta^*$ with probability $1$. \textbf{If $\theta^*$ were known}, then the safe LQR problem with $\mathcal{C}^{\theta^*}$ as the baseline would simply be to minimize the expected total cost for all controllers in this baseline, i.e. to solve
\begin{equation}\label{eq:fullformula}
\begin{array}{ll@{}ll}
\displaystyle\min_{C \in \mathcal{C}^{\theta^*}}  & T \cdot J^*(\theta^*, C, T).
\end{array}
\end{equation}
We will use the expression in Equation \eqref{eq:fullformula} as the baseline cost to which we compare the cost of our algorithms. We will often consider families of controller classes $\{\mathcal{C}^{\theta}\}_{\theta \in \Theta}$ such that for any dynamics $\theta$, every controller in the class $\mathcal{C}^{\theta}$ is safe for dynamics $\theta$ with probability $1$. For example, the baseline class $\mathcal{C}^{\theta}$ could be the class of linear controllers that are safe for dynamics $\theta$, the class of affine controllers that are safe for dynamics $\theta$, all controllers that are safe for dynamics $\theta$, etc.

The \textit{regret} of an algorithm with corresponding controller $C_{\text{alg}}$ with respect to baseline $\mathcal{C}^{\theta^*}$ is the random variable
\begin{equation}\label{eq:regret_rv}
  \text{Regret } := T \cdot  J(\theta, C_{\text{alg}}, T, 0, W) - \min_{C \in \mathcal{C}^{\theta^*}}  T \cdot J^*(\theta^*, C, T).
\end{equation}
Note that this regret random variable compares the realized cost of the algorithm with the expected cost of a controller from the baseline class, and this definition of regret is typical in the LQR learning literature \citep{abbasi2011regret,li2021safe}. We also could have defined regret comparing the realized cost of an algorithm to the realized cost of the best (in expectation) controller from the baseline class. Due to standard concentration inequalities, the realized total cost of the baseline controller will be within $\tilde{O}(\sqrt{T})$ of the expected total cost of the baseline controller. Therefore, considering a realized total cost for both terms in the regret would change our regret bounds by at most $\tilde{O}(\sqrt{T})$ and therefore not change any of the results.

The overarching goal of this paper is to find a controller $C_{\mathrm{alg}}$ that achieves low regret as defined in Equation \eqref{eq:regret_rv} and such that for any true dynamics $\theta^* \in \Theta$, the controller $C^{\mathrm{alg}}$ is safe for $\theta^*$ with probability $1-o_T(1/T)$. Note that we only require that the algorithm $C_{\text{alg}}$ is safe with probability $1-o_T(1/T)$, while we require the baseline to be safe with probability $1$. This (slightly unfair) mismatch is necessary to allow the algorithm to use information ``learned" from historical observations when trying to satisfy the safety constraints.  For example, if $\mathcal{D}$ is an unbounded distribution, then it is impossible to conclude anything with probability $1$ based on any amount of historical information. We want to allow our algorithm to use information about $\theta^*$ learned from previous time steps to choose better future safe controls, and therefore we only require safety with respect to $\theta^*$ with probability $1-o_T(1/T)$. We chose $1-o_T(1/T)$ for the safety probability because this is strictly stronger than $1-o_T(1)$ or $1-\delta$ for constant $\delta > 0$, and therefore our results hold for these larger probabilities of satisfying safety as well. In principle, we could also compare to a baseline that allows some probability of error. However, because the baseline does not need to learn $\theta^*$, allowing it to be safe with probability slightly less than 1 would not significantly impact its cost, while it would significantly increase the mathematical complexity of the analysis.

Finally, we will make the following assumptions about the problem specifications throughout this paper.

\begin{assum}[Problem Specifications]\label{problem_specifications}
    The noise distribution $\mathcal{D}$ is mean-$0$, variance $1$, and subgaussian with bounded density. The boundaries $D^{\E[x]}_{\mathrm{L}}, D^{\E[x]}_{\mathrm{U}}$ (which may be functions of $T$) satisfy that $-\log^2(T) \le D_{L}^{\E[x]} < 0 < D_{U}^{\E[x]} \le \log^2(T)$ and that $D^{\E[x]}_{\mathrm{U}} - D_{\mathrm{L}}^{\E[x]} \ge \frac{1}{\log(T)}$. 
\end{assum}
For exposition purposes, we also assume that $\log_2(T^{1/12})$ is an integer. The assumption of variance $1$ gives a simpler uncertainty bound, but as in \citet{abbasi2011regret} this can be relaxed. We assume that $\max(|D_{\mathrm{L}}^{\E[x]}|, D_{\mathrm{U}}^{\E[x]}) \le \log^2(T)$ because if the constraints are greater than $\log^2(T)$, then the constraints have very little impact on the optimal controller. This is because with subgaussian noise, with high probability the noise random variables have magnitude less than $o(\log(T))$, and so reasonable controllers will with high probability never hit the constraint. Therefore, if both boundaries are greater than $\log^2(T)$ then the problem becomes similar to the unconstrained problem, and if one  boundary is large, then the problem becomes one sided which is an easier version of our problem. The assumption of mean-$0$ and subgaussian noise is also standard in the LQR literature \citep{abbasi2011regret,dean2019safely, li2021safe}.

Putting everything together, the formal problem we are considering is the following.

\begin{problem}[Safe LQR Learning]\label{problem_main}
    Suppose we are given $D, \mathcal{D}, \Theta, T$ that satisfy Assumption \ref{assum_init}--\ref{problem_specifications} and a set of baseline classes of controllers $\{\mathcal{C}^\theta\}_{\theta \in \Theta}$. Then the goal of safe LQR learning is to find an algorithm $C^{\mathrm{alg}}$ that achieves a regret with respect to baseline $\mathcal{C}^{\theta^*}$ that is as low as possible, while also satisfying $\sup_{\theta \in \Theta} \P\left(C^{\mathrm{alg}} \text{ is safe with respect to $\theta$}\right) = 1-o_T(1/T)$.
\end{problem}
Note that $\sup_{\theta \in \Theta} \P\left(C^{\mathrm{alg}} \text{ is safe with respect to $\theta$}\right) = 1-o_T(1/T)$ is equivalent to requiring that there exists some probability $p = 1 - o_T(1/T)$ such that for any $\theta \in \Theta$, if the true dynamics $\theta^* = \theta$ then the controls used by $C^{\mathrm{alg}}$ are safe with respect $\theta^*$ with probability $p$.

\subsection{Notation}

To simplify notation, we use $\theta = (a,b)$ to represent an arbitrary set of dynamics and $\theta^* = (a^*,b^*)$ to represent the true (unknown) dynamics. We will also use $D:= (D_{\mathrm{L}}, D_{\mathrm{U}}) := (D_{\mathrm{L}}^{\E[x]}, D_{\mathrm{U}}^{\E[x]})$ (i.e., drop the superscripts).
We will use $\tilde{O}_T$ and $O_T$ notation to represent $\tilde{O}$ and $O$ with respect to $T$, where the values of the hidden constants and $\log$ terms may depend on the values of problem inputs such as $q,r, D, \mathcal{D}, \Theta$. Because the nature of our problem requires us to define a significant amount of notation in this paper, we have a table in Appendix \ref{app:notation} that lists the common notation used throughout the paper that the reader can use as a reference if needed.

\section{Theoretical Results}
The goal of this paper is to provide a general framework for studying the regret with respect to non-linear baselines of controllers. We first introduce a general class of baselines satisfying regularity conditions in Section \ref{sec:baseline}. We then present our two main theorems in Section \ref{sec:baseline_theorems}.

\subsection{Regret Rates for General Baselines}\label{sec:baseline}
In order to present our main theorem, we first need a baseline class of controllers $\mathcal{C}^{\theta^*}$ to define the regret in Equation~\eqref{eq:regret_rv}. In both \citet{li2021safe} and \citet{dean2019safely}, the regret baseline for the $\tilde{O}_T(T^{2/3})$ results is the cost of the best stationary linear controller of the form $u_t = -Kx_t$ that is safe for $\theta^*$ with probability $1$. We will refer to the class of stationary linear controllers that are safe for $\theta^*$ with probability $1$ as the class of safe linear controllers. Since not all linear controllers are safe for dynamics $\theta^*$, this is restricted to $K$ that will maintain safety for $\theta^*$ for any realization of the noise, and therefore can be a very weak baseline. Linear controllers are not always well-suited for safety constrained LQR because linear controllers only have one degree of freedom ($K$), but in safety constrained LQR the controller must balance keeping regret low with being safe. For example, when $D_{\mathrm{U}}$ and $D_{\mathrm{L}}$ are not symmetric, the best linear controller must still behave symmetrically. However, symmetric behavior may be far from optimal for $D_{\mathrm{U}}$ and $D_{\mathrm{L}}$ that are not symmetric, and linear controllers lack the flexibility to behave non-symmetrically. Therefore, there exist much stronger baselines than the safe linear controllers studied in \citet{li2021safe,dean2019safely}. 

In Section \ref{sec:baseline_theorems}, we present two results that hold for a wide range of stronger baseline classes of controllers. Before stating the theorems, we will outline a few assumptions on the controllers in these general baseline classes. 

Let $\{\mathcal{C}^{\theta}\}_{\theta \in \Theta}$ be the set of baseline classes of controllers for dynamics $\theta \in \Theta$. For the rest of this paper, we will assume that the baseline class of controllers satisfies Assumption \ref{assum:baseline_cont}. 
\begin{assum}\label{assum:baseline_cont}
    All of the controllers in the baseline class $\mathcal{C}^{\theta}$ for all $\theta \in \Theta$ are stationary, Markovian,  deterministic, and safe for dynamics $\theta$ with probability $1$.
\end{assum}

Note that the assumption that every controller in $\mathcal{C}^{\theta}$ is safe for dynamics $\theta$ with probability $1$ is consistent with the baselines of \citet{li2021safe} and \citet{dean2019safely}. Additionally, this means that the baseline class of controllers could change depending on the dynamics $\theta$, as the class of controllers that is safe for one dynamics will not necessarily be safe for a different dynamics. One option is to construct the baseline class from another class of controllers $\tilde{\mathcal{C}}$ (for example the class of linear controllers), as follows:
\begin{equation}\label{eq:safe_baseline}
    \{C \in \tilde{\mathcal{C}} : \text{$C$ is safe for dynamics $\theta$}\}.
\end{equation}
If $\tilde{\mathcal{C}}$ is a rich enough class of controllers (e.g. all controllers), then Equation \eqref{eq:safe_baseline} would result in a good safe baseline. However, if $\tilde{\mathcal{C}}$ is a relatively small class of controllers (e.g. linear controllers), then the restriction in Equation \eqref{eq:safe_baseline} to only controllers in the class that are safe for $\theta$ may result in a weak safe baseline. Therefore, instead of simply subsetting the class of controllers $\tilde{\mathcal{C}}$ as in Equation \eqref{eq:safe_baseline}, we will preserve the complexity of the function class $\tilde{\mathcal{C}}$ by transforming \emph{every} controller in $\tilde{\mathcal{C}}$ into a controller that is safe for $\theta$. We generalize even further by allowing the starting class of controllers $\tilde{\mathcal{C}}^\theta$ to be different for each $\theta$. 
\begin{assum}[Truncation]\label{weak_depend}
     For any $\theta$, there exists a controller class $\tilde{\mathcal{C}}^\theta$  of deterministic controllers such that the baseline class $\mathcal{C}^{\theta}$ consists of all controllers of the following form for $C \in \tilde{\mathcal{C}}^\theta$:
    \begin{equation}\label{eq:weak_depend}
    C^{\theta}(x)=
        \begin{cases}
            C(x) & \text{if } D_{\mathrm{L}} \le ax + bC(x) \le D_{\mathrm{U}}\\
            \frac{D_{\mathrm{U}}-ax}{b} & \text{if } ax + bC(x) > D_{\mathrm{U}} \\
            \frac{D_{\mathrm{L}} - ax}{b} & \text{if } ax + bC(x) < D_{\mathrm{L}}.
        \end{cases}
    \end{equation}
\end{assum}

By this construction, every controller $C^\theta \in \mathcal{C}^{\theta}$ is safe for dynamics $\theta$. We will also assume that $\mathcal{C}^{\theta}$ is parameterizable by a scalar parameter $K \in \mathbb{R}$. This allows us to choose the optimal controller in $\mathcal{C}^{\theta}$ in terms of the parameter $K$.

\begin{assum}[Parametrization]\label{safety_assum}
For any $\theta$, there exists $K_{\mathrm{L}}^\theta,K_{\mathrm{U}}^\theta \in \mathbb{R}$ such that the $\mathcal{C}^\theta$ in Assumption \ref{weak_depend} can be parameterized as $\mathcal{C}^\theta = \{C^\theta_K : K \in [K_{\mathrm{L}}^\theta, K_{\mathrm{U}}^\theta]\}$. Furthermore, for any $\theta$, $T$ there exists a $K_{\mathrm{opt}}(\theta,T)$ satisfying
\[
     K_{\mathrm{opt}}(\theta,T) = \arg\min_{K \in [K_{\mathrm{L}}^\theta, K_{\mathrm{U}}^\theta]} J^*(\theta, C^{\theta}_K, T).
\]

\end{assum}

Our results require two more key assumptions on the class of controllers.

\begin{assum}[Average Cost Lipschitz in Optimal Controller]\label{parameterization_assum2}
    There exists $\epsilon_{\mathrm{A}\ref{parameterization_assum2}} = \tilde{\Omega}_T(1)$ such that for any $\norm{\theta - \theta^*}_{\infty}  \le \epsilon_{\mathrm{A}\ref{parameterization_assum2}}$ and $t \le T$,
    \[
        |J^*(\theta^*,C_{K_{\mathrm{opt}}(\theta, t)}^\theta,t) - J^*(\theta^*,C_{K_{\mathrm{opt}}(\theta^*, t)}^{\theta^*},t)| \le \tilde{O}_T\left(\norm{\theta - \theta^*}_{\infty} + \frac{1}{T^2}\right).
    \]
    
\end{assum}

Assumption \ref{parameterization_assum2} relates the expected cost under dynamics $\theta^*$ of the optimal controller for dynamics $\theta^*$ to the expected cost of the optimal controller for some other dynamics $\theta$ close to $\theta^*$. Intuitively, this is a form of Lipschitz continuity which implies that the performance of the optimal controller is not too sensitive to the choice of $\theta$. Some sort of continuity assumption is intuitively required for any form of certainty equivalence algorithm to achieve low regret guarantees.
        
\begin{assum}[Total Cost Lipschitz in Initial Position]\label{parameterization_assum3}
     There exist $\epsilon_{\mathrm{A}\ref{parameterization_assum3}}, \delta_{\mathrm{A}\ref{parameterization_assum3}} = \tilde{\Omega}_T(1)$ such that for any $\theta$ satisfying $\norm{\theta - \theta^*}_{\infty} \le \epsilon_{\mathrm{A}\ref{parameterization_assum3}}$ the following holds. For $t < T$, let $W' = \{w_i\}_{i=0}^{t-1}$. Then for any $K \in [K_{\mathrm{L}}^\theta, K_{\mathrm{U}}^\theta]$, there exists a set $\mathcal{Y}_{\mathrm{A}\ref{parameterization_assum3}} \in \mathbb{R}^{t}$ that depends only on $C_K^\theta$ such that the following holds. Define $E_{\mathrm{A}\ref{parameterization_assum3}}\left(C_K^\theta, W'\right)$ as the event that $W' \in \mathcal{Y}_{\mathrm{A}\ref{parameterization_assum3}}$. Then $\P(E_{\mathrm{A}\ref{parameterization_assum3}}\left(C_K^\theta, W'\right)) \ge 1-o_T(1/T^{10})$ and for any  $|x|, |y| \le 4\log^2(T)$ such that $|x-y| \le \delta_{\mathrm{A}\ref{parameterization_assum3}}$, conditional on event $E_{\mathrm{A}\ref{parameterization_assum3}}\left(C_K^\theta, W'\right)$,
    \begin{equation}\label{eq:param_assum3}
        \left|t \cdot J(\theta^*,C_{K}^\theta,t,x, W') - t \cdot J(\theta^*,C^{\theta}_{K},t, y, W')\right| \le \tilde{O}_T(|x-y| + \norm{\theta - \theta^*}_{\infty}).
    \end{equation}

\end{assum}

Assumption \ref{parameterization_assum3} relates the random variables of cost when starting at two different positions, $x$ and $y$, but with the same noise random variables $W'$. Intuitively, this implies that making a small non-optimal control will not have significant long-term impact on the total cost. Therefore, this assumption can be thought of as assuming the total cost is Lipschitz in the initial position.

The final assumption we consider is an assumption that the noise distribution has sufficiently large support, which we require for Theorem \ref{sufficiently_large_error} but not for Theorem \ref{performance}. Note that we only need the noise distribution to have large support relative to one of the two boundaries ($D_{\mathrm{U}}$ or $D_{\mathrm{L}}$). We will w.l.o.g. state Assumption \ref{assum_sufficiently_large_error} relative to boundary $D_{\mathrm{U}}$, however an equivalent assumption swapping $D_{\mathrm{L}}$ and $D_{\mathrm{U}}$ would also be sufficient for Theorem \ref{sufficiently_large_error}. 
\begin{assum}\label{assum_sufficiently_large_error}
     For any $z$,  define $P(\theta,K,z)$ as the largest $x$ such that $ ax + bC^{\theta}_{K}(x) \le z$. There exists a constant $\epsilon_{\mathrm{A}\ref{assum_sufficiently_large_error}} > 0$ such that the following equation holds for all $t \ge \sqrt{T}$ for sufficiently large $T$:
     \begin{equation}\label{eq:sufficiently_large_errora}
        \P_{w \sim \mathcal{D}}\left(w \ge P(\theta^*,K_{\mathrm{opt}}(\theta^*, t), D_{\mathrm{U}}) - D_{\mathrm{L}}\right) \ge \epsilon_{\mathrm{A}\ref{assum_sufficiently_large_error}} > 0.
    \end{equation}
\end{assum}
The quantity $P(\theta^*,K_{\mathrm{opt}}(\theta^*, t), D_{\mathrm{U}})$ will often be proportional to and greater than $D_{\mathrm{U}}$. Because $\mathcal{D}$ is constant relative to $T$, Assumption \ref{assum_sufficiently_large_error} implies that the boundary $D$ must satisfy $\norm{D}_{\infty} = O_T(1)$. When $\norm{D}_{\infty} = O_T(1)$, Assumption \ref{assum_sufficiently_large_error} effectively requires that the noise distribution $\mathcal{D}$ has a constant probability of spanning the distance between $D_{\mathrm{L}}$ and $D_{\mathrm{U}}$. Note that Assumption \ref{assum_sufficiently_large_error} is automatically satisfied for any $\norm{D}_{\infty} = O_T(1)$ when the noise distribution is Gaussian, unbounded, or bounded with a high enough variance. This assumption will be necessary to achieve regret of $\tilde{O}_T(\sqrt{T})$ in Theorem \ref{sufficiently_large_error}, as the variance from the noise distribution of Assumption \ref{assum_sufficiently_large_error} provides the controller with extra exploration that leads to better estimation. We will also provide a result for general classes of controllers that does not require this assumption, but achieves a worse regret rate (Theorem \ref{performance}).

\subsection{Theorems}\label{sec:baseline_theorems}
 We are now ready to present our first theorem. 

\begin{thm}\label{sufficiently_large_error}
In the setting of Problem \ref{problem_main} and under further Assumptions \ref{assum:baseline_cont}--\ref{assum_sufficiently_large_error}, there exists an algorithm $C^{\mathrm{alg}}$ (Algorithm \ref{alg:cap_large}) that with probability $1-o_T(1/T)$ achieves  $\tilde{O}_T(\sqrt{T})$ regret with respect to baseline $\mathcal{C}^{\theta^*}$ while also satisfying $\sup_{\theta \in \Theta} \P\left(C^{\mathrm{alg}} \text{ is safe with respect to $\theta$}\right) = 1-o_T(1/T)$.
\end{thm}

The key lemma in proving Theorem \ref{sufficiently_large_error} is a new estimation bound for the unknown system dynamics $\theta^*$ (Lemma \ref{boundary_uncertainty_cont}). Informally, this estimation bound shows that simply by obeying safety constraints, the unknown dynamics can be estimated at a rate of $1/\sqrt{t}$ without injecting any additional randomness into the controller. This faster rate of learning is because in order to be safe, the controller must frequently be non-linear, which in turn helps learn the unknown dynamics. This result of safe behavior leading to faster learning rates may also be of independent interest in other safe RL problems.

The more general result of this paper is Theorem \ref{performance}, which achieves a weaker regret rate of $\tilde{O}_T(T^{2/3})$ but applies for any subgaussian noise distribution (in particular, it drops Assumption~\ref{assum_sufficiently_large_error}).

\begin{thm}\label{performance}
    In the setting of Problem \ref{problem_main} and under further Assumptions \ref{assum:baseline_cont}--\ref{parameterization_assum3}, there exists an algorithm $C^{\mathrm{alg}}$ (Algorithm \ref{alg:cap}) that with probability $1-o_T(1/T)$ achieves  $\tilde{O}_T(T^{2/3})$ regret with respect to baseline $\mathcal{C}^{\theta^*}$ while also satisfying $\sup_{\theta \in \Theta} \P\left(C^{\mathrm{alg}} \text{ is safe with respect to $\theta$}\right) = 1-o_T(1/T)$.
\end{thm}

Theorem \ref{performance} is an improvement on existing results in that it bounds the regret of constrained LQR learning for any subgaussian noise distribution. See Section \ref{proof_sketch_performance} and Section \ref{proof_sketch_sufficiently_large_error} for the proof sketches of Theorem \ref{performance} and  Theorem \ref{sufficiently_large_error} respectively. Previous works focus on linear controller baselines, and linear controllers have properties that allow for easier regret analysis. Theorems \ref{performance} and \ref{sufficiently_large_error} reduce these ``useful" properties of linear controllers to Assumptions  \ref{parameterization_assum2} and \ref{parameterization_assum3}. Therefore, many classes of non-linear controllers can be constructed as described in this section, and all that needs to be done to show that the result of the theorems hold with such a class of controllers as a baseline is to show that this class of controllers satisfies Assumptions \ref{parameterization_assum2} and \ref{parameterization_assum3}. Both of Assumptions \ref{parameterization_assum2} and \ref{parameterization_assum3} are simply Lipschitz conditions on the cost function (one with respect to the optimal controller and one with respect to the starting position), and therefore are likely to hold for many classes of controllers. In particular, \citet{schiffer2025stronger} shows that both of these assumptions are satisfied for the class of truncated linear controllers, and therefore Theorems \ref{sufficiently_large_error} and \ref{performance} apply for this baseline class of controllers. The properties in Assumptions \ref{parameterization_assum2} and \ref{parameterization_assum3} are the main tools that allow us to analyze the regret of nonlinear general baselines, and therefore these properties may be of independent interest outside of these theorems.

 The algorithms that achieve the regret bounds of Theorems \ref{sufficiently_large_error} and \ref{performance} follow the same general form. We outline the algorithm that achieves Theorem \ref{performance} below in Algorithm \ref{alg:cap_large_intuit}.
\begin{algorithm}
\caption{Outline of Algorithm \ref{alg:cap} for proof of Theorem \ref{performance} }\label{alg:cap_large_intuit}
\begin{algorithmic}[1]
\State Explore for $\tilde{\Theta}_T(T^{2/3})$ steps using controller $C^{\mathrm{init}}$ from Assumption \ref{assum:initial} with random noise.
\For{$s \in [0:\log(T^{1/3})-1]$} \label{line:intuit_start}
\State $\hat{\theta}_s \gets $ regularized least-squares estimate of $\theta^*$ using data seen so far 
\State $\epsilon_s \gets$ high probability bound on $\norm{\theta^* - \hat{\theta}_s}_{\infty}$
\State $C_s^{\mathrm{alg}} \gets$ optimal controller from baseline class for dynamics $\hat{\theta}_s$
\State For next $T^{2/3}2^s$ steps, use controller $C_s^{\mathrm{alg}}$ modified at each step to be safe for all dynamics $\theta$ satisfying $\norm{\theta - \hat{\theta}_s}_{\infty} \le \epsilon_s$ \label{line:intuit_end}
\EndFor
\end{algorithmic}
\end{algorithm}

This algorithm mostly behaves like a standard certainty equivalence algorithm, first calculating the regularized least-squares estimate of $\theta^*$ and then finding the best controller for this estimated dynamics. This algorithm deviates from standard certainty equivalence in the final line, where the algorithm enforces safety by modifying the controller $C_s^{\mathrm{alg}}$. Because $\theta^*$ with high probability satisfies $\norm{\theta^* - \hat{\theta}_s}_{\infty} \le \epsilon_s$, the modification in the final line guarantees safety for dynamics $\theta^*$ with high probability. The bulk of the theoretical work in proving Theorem \ref{alg:cap} is upper bounding the regret contributed by these safety modifications. Theorem \ref{sufficiently_large_error} follows a similar pattern with a slightly more complicated choice of $\hat{\theta}_s$. In the setting of Theorem \ref{sufficiently_large_error}, the large support of the noise distribution leads to the controls used by controller $C_s^{\mathrm{alg}}$ being non-linear by a constant amount for a constant fraction of the steps. This non-linearity allows the algorithm to learn at a faster rate than in Theorem \ref{performance} and results in the lower regret bound of $\tilde{O}_T(\sqrt{T})$. Note also that the length of the exploration period and the number of steps in each round of the loop are chosen differently for Algorithm \ref{alg:cap_large} than for Algorithm \ref{alg:cap}. See the proof sketches in the following section for more details.

\section{Proof Sketches of Main Results}
We will present the proof sketches (and formal proofs) of the main results in \textit{reverse} of the order in which they were stated in the previous section. We present the proofs in this manner because the result of Theorem \ref{performance} is a weaker result in a more general setting. We therefore build off of this proof in the subsequent proof of Theorem \ref{sufficiently_large_error} by strengthening the result of Theorem \ref{performance} in less general settings. 

\subsection{Proof Sketch of Theorem \ref{performance}}\label{proof_sketch_performance}
The full proof of Theorem \ref{performance} can be found in Appendix \ref{app:proof_of_performance}. 

First we state Algorithm~\ref{alg:cap}, which is the algorithm that achieves the guarantee of Theorem~\ref{performance}. But before presenting the algorithm, we need some additional notation. Fix a constant $\lambda > 0$. Define $z_t = (x_t, u_t)^{\top}$ and $V_t = \lambda I + \sum_{i=0}^{t-1} z_iz_i^\top$, where $I$ is the identity matrix. Define $X_t$ as the column vector $(x_1,...,x_{t})^{\top}$ and $Z_t$ as the matrix with rows $z_0^{\top},...,z_{t-1}^{\top}$. Define $B_{t} = \alpha\sqrt{\log(\det(V_t)) + \log(\lambda^2) + 2\log(T^2)} + \sqrt{\lambda}(\bar{a}^2 + \bar{b}^2) = \tilde{O}_T(1)$ where $\alpha$ is the subgaussian parameter of $\mathcal{D}$. The algorithm that achieves the regret bound of Theorem \ref{performance} is given as Algorithm \ref{alg:cap}.

\paragraph{Algorithm \ref{alg:cap} Intuition}
Algorithm \ref{alg:cap} can be broken into two phases: a warm-up exploration phase (Lines \ref{line:warmup_start}--\ref{line:warmup_end}) and a safe certainty equivalence phase (Lines \ref{line:exploit_start}--\ref{line:exploit_end}). In the warm-up phase, the controls are random which allows for sufficient exploration and learning of the unknown dynamics. In the certainty equivalence phase, $\hat{\theta}_s$ is the regularized least-square estimate of $\theta^*$ based on the data seen so far. $\epsilon_s$ is an upper bound on the distance between $\hat{\theta}_s$ and $\theta^*$ that holds with high probability. $C_s^{\mathrm{alg}}$ is the optimal controller from the baseline class for dynamics $\hat{\theta}_s$. Because $C_s^{\mathrm{alg}}$ is not guaranteed to be safe for dynamics $\theta^*$, we calculate $u_t^{\mathrm{safeU}}$ and $u_t^{\mathrm{safeL}}$ which are respectively the largest and smallest possible controls that satisfy the constraints for all dynamics $\theta$ within $\epsilon_s$ distance of $\hat{\theta}_s$ (which will with high probability include $\theta^*$). We then censor the control $C_s^{\mathrm{alg}}(x_t)$ with these two controls to guarantee with high probability that the final chosen control is safe with respect to dynamics $\theta^*$.
In order to show Theorem \ref{performance}, we must show that with probability $1-o_T(1/T)$, Algorithm \ref{alg:cap} is safe with respect to $\theta^*$ and that Algorithm \ref{alg:cap} has $\tilde{O}_T(T^{2/3})$ regret. To show the latter, we will decompose the regret into four main components and consider each separately.

\begin{algorithm}[H]
\caption{Safe LQR for General Baselines}\label{alg:cap}
\begin{algorithmic}[1]
\Require $D, \mathcal{D}, \Theta, C^{\mathrm{init}}, \{\mathcal{C}^{\theta}\}_{\theta \in \Theta}, T, \lambda$
\State{$\nu_T \gets T^{-1/3}$} \label{line:nu}
\For{$t \gets 0$ to $\frac{1}{\nu_T^2}-1$} \Comment{Safe warm-up exploration phase} \label{line:warmup_start}
        \State{$\phi_t \sim \mathrm{Rademacher}(0.5)$}
        \State{Use control $u_t = C^{\mathrm{init}}(x_t) + \frac{\phi_t}{\log(T)}$}
\EndFor\label{line:warmup_end}
\For{$s \gets 0$ to $\log_2(T\nu_T^2) - 1$} \label{line:exploit_start}\Comment{Safe certainty equivalence phase}
    \State{$T_s \gets \frac{2^s}{\nu_T^2}$}
    \State{$\hat{\theta}_s \gets (Z_{{T_s}}^{\top}Z_{{T_s}}+\lambda I)^{-1}Z_{{T_s}}^{\top}X_{{T_s}}$}
    \State{$C_s^{\mathrm{alg}} \gets C^{\hat{\theta}_s}_{K_{\mathrm{opt}}(\hat{\theta}_s, T_s)}$}
    \State{$\epsilon_s \gets B_{T_s}\sqrt{\frac{\max\left(V_{T_s}^{22}, V_{T_s}^{11}\right)}{V_{T_s}^{11}V_{T_s}^{22} - (V_{T_s}^{12})^2}}$}  \label{line:epsilon}
    \For{$t \gets T_s$ to $2T_s - 1$}
        \State{$u_t^{\mathrm{safeU}} \gets \max \left\{u :  \displaystyle\max_{\norm{\theta - \hat{\theta}_s}_{\infty} \le \epsilon_s} ax_t + bu \le D_{\mathrm{U}}\right\}$} \label{line:safeU}
        \State{$u_t^{\mathrm{safeL}} \gets \min\left\{ u: \displaystyle\min_{\norm{\theta - \hat{\theta}_s}_{\infty} \le \epsilon_s} ax_t + bu \ge D_{\mathrm{L}}\right\}$} \label{line:safeL}
        \State{Use control $u_t = \max\left(\min\left(C_s^{\mathrm{alg}}(x_t), u^{\mathrm{safeU}}_t\right),u^{\mathrm{safeL}}_t\right)$} \label{line:safety}
    \EndFor 
\EndFor \label{line:exploit_end}
\end{algorithmic}
\end{algorithm}

\paragraph{Safety of Algorithm \ref{alg:cap}}
We begin with analyzing the safety of Algorithm \ref{alg:cap}. The first loop (warm-up exploration) of Algorithm \ref{alg:cap} is safe with respect to dynamics $\theta^*$ as a result of Assumption \ref{assum:initial}. In the second loop (safe certainty equivalence), the control in Line \ref{line:safety} is chosen to enforce safety relative to all $\theta$ satisfying $\norm{\theta - \hat{\theta}_s}_{\infty} \le \epsilon_s$. By the choice of $\epsilon_s$, the true dynamics $\theta^*$ satisfy $\norm{\theta^* - \hat{\theta}_s}_{\infty} \le \epsilon_s$ for all $s$ with probability $1-o_T(1/T)$ (Lemma \ref{v_to_use}). Therefore, the control applied in Line \ref{line:safety} is safe with respect to $\theta^*$ for all $t$ with probability $1-o_T(1/T)$. Therefore, Algorithm \ref{alg:cap} is safe with respect to $\theta^*$ with probability $1-o_T(1/T)$.

\paragraph{Regret from warm-up period}
The first component of regret ($R_0$) is the cost of the warm-up exploration phase, which is the first $1/\nu_T^2$ steps of the algorithm. Using Assumption \ref{problem_specifications}, we can show that the positions and controls during this phase are with high probability bounded by $\tilde{O}_T(1)$ (Lemma \ref{bounded_approx}). Therefore, the cost during this phase can be bounded by $\tilde{O}_T(1/\nu_T^2)$ (Proposition \ref{warm_up_regret}). Importantly, after this initial exploration phase, $\epsilon_s = \tilde{O}_T(\nu_T)$ with probability $1-o_T(1/T)$ (Lemma \ref{initial_uncertainty}). This is a result of the Rademacher random variables in the warm-up phase.

\paragraph{Regret from certainty equivalence}
The second source of regret ($R_1$) comes from the certainty equivalence aspect of the algorithm. In other words, $R_1$ is the regret from the fact that $K_{\mathrm{opt}}(\hat{\theta}_s, T_s)$ is the optimal controller for dynamics $\hat{\theta}_s$ and not for dynamics $\theta^*$. By Lemma \ref{initial_uncertainty} and Lemma \ref{v_to_use}, with high probability $\norm{\hat{\theta}_s - \theta^*}_{\infty} \le \epsilon_s = \tilde{O}_T(\nu_T)$, so by Assumption \ref{parameterization_assum2} the expected cost of using controller $C_{K_{\mathrm{opt}}(\hat{\theta}_s, T_s)}^{\hat{\theta}_s}$ for $T_s$ steps is at most $\tilde{O}_T(T_s\norm{\hat{\theta}_s - \theta^*}_{\infty} + 1/T)$ more  than the expected cost of using $C_{K_{\mathrm{opt}}(\theta^*, T_s)}^{\theta^*}$ for $T_s$ steps. Using the aforementioned bound comparing $\hat{\theta}_s$ and $\theta^*$, this source of regret can therefore be upper-bounded by $\tilde{O}_T(T\nu_T)$ with probability $1-o_T(1/T)$ (Proposition \ref{non_optimal_controller}). 

\paragraph{Regret from deviation from expectation}
The third source of regret ($R_2$) comes from the fact that we defined regret as the difference between the cost of the algorithm (which is a random variable) and the expected cost of the best controller in the baseline class (which is nonrandom). To bound this regret term, we show that the cost of the algorithm concentrates within $\tilde{O}_T(\sqrt{T})$ of its expectation with probability $1-o_T(1/T)$ (Proposition \ref{r1b_bound}). For this result, we use a variant of McDiarmid's Inequality that applies to high probability events combined with Assumption \ref{parameterization_assum3} (Lemma \ref{concentration_of_cond_exp}).

\paragraph{Regret from enforcing safety}
The final source of regret ($R_3$) is a result of the times the algorithm ``enforces safety" on the controls by sometimes using controls $u_t^{\mathrm{safeU}}$ and $u_t^{\mathrm{safeL}}$. With probability $1-o_T(1/T)$, when the algorithm enforces safety, the chosen $u_t$ differs from $C_s^{\mathrm{alg}}(x_t)$ by $\tilde{O}_T(\epsilon_s)$ (Lemma \ref{offbyepsiloncontrol_propproof}). By Assumption \ref{parameterization_assum3} and Lemma \ref{initial_uncertainty}, the small differences between $C_s^{\mathrm{alg}}(x_t)$ and $u_t$ each increase the cost by at most $\tilde{O}_T(\nu_T)$  with probability $1-o_T(1/T)$. Therefore, the total cost of enforcing safety with these controls is $\tilde{O}_T(\nu_TT)$ with probability $1-o_T(1/T)$ (Proposition \ref{enforcing_safety}). 

\paragraph{Combining Regret Terms}

Putting these four sources of regret together, the total regret can be upper bounded as follows with probability $1-o_T(1/T)$:
{\small
\begin{equation}
    T \cdot J(\theta^*, C_{\text{alg}}, T,0,W) - T \cdot J^*(\theta^*, C^{\theta^*}_{K_{\mathrm{opt}}(\theta^*, T)}, T) \le R_0 + R_1 + R_2 + R_3 = \tilde{O}_T\left(\sqrt{T} + T\nu_T + \frac{1}{\nu^2_T} \right) = \tilde{O}_T(T^{2/3}),
\end{equation}
}
where the last line comes from the fact that $\nu_T = T^{-1/3}$. See Appendix \ref{app:proof_of_performance} and Equation \eqref{eq:all_decomposition} for a formal description of these four sources of regret.

\subsection{Proof Sketch of Theorem \ref{sufficiently_large_error}}\label{proof_sketch_sufficiently_large_error}
The full proof of Theorem \ref{sufficiently_large_error} can be found in Appendix \ref{app:suff_large_noise_case}. 

    \paragraph{Algorithm and Intuition}
    The algorithm that achieves the regret result of Theorem \ref{sufficiently_large_error} is Algorithm \ref{alg:cap_large}, which is very similar to Algorithm \ref{alg:cap}. Rather than restating the entire algorithm here, we defer the full algorithm to the appendix and instead highlight the main differences between Algorithm \ref{alg:cap_large} and Algorithm \ref{alg:cap}. The first modification is that for Algorithm \ref{alg:cap_large} we choose $\nu_T = T^{-1/4}$, which affects the lengths of the exploration and certainty equivalence periods. The second major difference is that we change how $\hat{\theta}_s$ is defined. Recall that in Algorithm \ref{alg:cap}, $\hat{\theta}_s$ is the regularized least-squares estimate of $\theta^*$. For this algorithm we instead denote the regularized least-squares estimate as
    \begin{equation}
    \hat{\theta}^{\text{pre}}_s = (Z_{{T_s}}^{\top}Z_{{T_s}}+\lambda I)^{-1}Z_{{T_s}}^{\top}X_{{T_s}}.
    \end{equation}
    Recall the function $P$ defined in Assumption \ref{assum_sufficiently_large_error}. We choose $\hat{\theta}_s$ as
    \begin{equation}\label{eq:choice_of_theta}
        \displaystyle  \hat{\theta}_s = \argmin_{\norm{\hat{\theta}_s - \hat{\theta}_s^{\text{pre}}}_{\infty} \le \epsilon_s} \min_{\norm{\theta -\hat{\theta}^{\text{pre}}_s}_{\infty} \le \epsilon_s} P(\theta, K_{\mathrm{opt}}(\hat{\theta}_s, T_s), D_{\mathrm{U}}).
    \end{equation}
    The choice of $\hat{\theta}_s$ described above is a technical way of ensuring that $C^{\mathrm{alg}}_s$ does sufficient exploration, which in turn guarantees a faster learning rate of the unknown dynamics. The key difference between the proof of Theorem \ref{sufficiently_large_error} and the proof of Theorem \ref{performance} is a new upper bound on $\epsilon_s$ which is stronger than Lemma \ref{initial_uncertainty}. Instead of $\epsilon_s = \tilde{O}_T(\nu_T)$ with probability $1-o_T(1/T)$, we show that $\epsilon_s = \tilde{O}_T\left(\frac{1}{\sqrt{T_s}}\right)$  with probability $1-o_T(1/T)$ (Lemma \ref{bounded_st}).  Informally, this means that with high probability, the estimated dynamics at time $t$ are at most $\tilde{O}_T\left(\frac{1}{\sqrt{t}}\right)$ different from $\theta^*$, and this is a faster learning rate than in Theorem \ref{performance}. This faster learning rate gives better upper bounds on the regret terms than in Theorem \ref{performance}. 

     \paragraph{Faster Learning Rate}
     Showing the faster learning rate requires two main results. The first result is that the uncertainty $\epsilon_s$ can be upper-bounded by $\tilde{O}_T(1/\sqrt{|S_{T_s}|})$, where $|S_{T_s}|$ is the number of times Algorithm \ref{alg:cap_large} uses the control $u_t^{\mathrm{safeU}}$ before time $T_s$ (Lemma \ref{boundary_uncertainty}). To prove this result, we prove a more general uncertainty bound in Lemma \ref{boundary_uncertainty_cont}. The key insight is that in order to maintain safety, the control  $u_t^{\mathrm{safeU}}$ will with high probability be sufficiently non-linear. This non-linearity combined with the variance in the position leads to a faster convergence rate of the upper bound in Lemma \ref{v_to_use}. The second result is that Algorithm \ref{alg:cap_large} uses the control $u_t^{\mathrm{safeU}}$ at least $\Omega_T(T_s)$ times before time $T_s$ (Lemma \ref{s_t_lowerbound}). The key insight to this result is that every time the position exceeds $P(\theta^*, K_{\mathrm{opt}}(\theta^*, T_s), D_{\mathrm{U}})$, Algorithm \ref{alg:cap_large} will use control $u_t^{\mathrm{safeU}}$. Assumption \ref{assum_sufficiently_large_error} says that the noise is large enough that (due to the choice of $\hat{\theta}_s$ in Equation \eqref{eq:choice_of_theta}) the position will exceed $P(\theta^*, K_{\mathrm{opt}}(\theta^*, T_s), D_{\mathrm{U}})$ in each round with constant probability. This implies that with probability $1-o_T(1/T)$, for every $s$, the control $u_t^{\mathrm{safeU}}$ is used a constant fraction of the times before time $T_s$. Combining these two results, we have for all $s$ that with probability $1-o_T(1/T)$, $\epsilon_s = \tilde{O}_T(1/\sqrt{|S_{T_s}|})$ and $|S_{T_s}| = \Omega_T(T_s)$. Therefore, we can conclude that with probability $1-o_T(1/T)$, we have $\epsilon_s = \tilde{O}_T\left(\frac{1}{\sqrt{T_s}}\right)$.

     \paragraph{Regret Proof Changes}
     Equipped with this tighter upper bound on $\epsilon_s$, we can bound $R_1$ (the regret of using controller $C_{K_{\mathrm{opt}}(\hat{\theta}_s, T_s)}^{\hat{\theta}_s}$ rather than $C_{K_{\mathrm{opt}}(\theta^*, T_s)}^{\theta^*}$) and $R_3$ (the regret of enforcing safety at every time step with controls $u_t^{\mathrm{safeU}}$ and $u_t^{\mathrm{safeL}}$) by $\tilde{O}_T(\sqrt{T})$ (Proposition \ref{non_optimal_controller_suff} and \ref{enforcing_safety_suff}, respectively). Because $\nu_T = T^{-1/4}$, $R_0$ is $\tilde{O}_T(\sqrt{T})$. Therefore, we can conclude as in the proof sketch of Theorem \ref{performance} that the regret of Algorithm \ref{alg:cap_large} is upper-bounded by $R_0 + R_1 + R_2 + R_3 = \tilde{O}_T(\sqrt{T})$.

\section{Discussion}\label{sec:discussion}

In this paper, we have presented new results for the safety-constrained LQR problem. We conclude by discussing some possible extensions of our work and remaining open questions.

While our results focus on positional constraints, we also expect that similar results would hold for algorithms similar to Algorithms \ref{alg:cap} and \ref{alg:cap_large} when there are also constraints on the controls. While we leave the formal derivations of results for control constraints to future work, we provide a brief discussion of how the algorithm and proofs would change. With the addition of control constraints, the algorithms can no longer use $u_t^{\mathrm{safeU}}$ or $u_t^{\mathrm{safeL}}$ as these constraints may not satisfy the control constraints. To address this, we believe that a slight modification to the way the algorithm chooses the controller $C^{\mathrm{alg}}_s$ will allow the algorithms to satisfy both control and position constraints with high probability and achieve the same regret results as in Theorems  \ref{sufficiently_large_error} and \ref{performance}. We propose choosing $C^{\mathrm{alg}}_s = C_K^{\hat{\theta}_s}$, where $K$ is chosen such that it satisfies positional constraints and control constraints $\tilde{\Theta}_T(\epsilon_s)$ tighter than the actual constraints. As long as $\norm{\hat{\theta}_s - \theta^*}_{\infty} \le \tilde{O}_T(\epsilon_s)$, this will guarantee both types of constraints are satisfied. The main additional result that needs to be assumed is that choosing this $C^{\mathrm{alg}}_s$ will not have significantly more regret than in the existing proofs. See Appendix \ref{sec:control_constraints} for more discussion on the generalization of our results to the setting with both position and control constraints.

Our results also focus on one-dimensional LQR, but we expect that many of the same results will generalize to higher dimensions. In higher dimensions, a natural generalization of our constraints is to consider a compact safe region that is defined as the intersection of a finite number of half-planes. Therefore, the goal would be to choose controls such that the expected position stays within this safe region. We expect that the uncertainty bounds proven in this paper will generalize naturally to higher dimensions, as our bounds are based on results in \citet{abbasi2011regret} that hold for higher dimensions. Therefore, we expect that the result of Theorem \ref{performance} will directly generalize to higher dimensions by replacing the controller $C_s^{\mathrm{alg}}$ with $C_K^{\hat{\theta}_s}$ where $K$ is chosen as the optimal control for constraints that are $\tilde{\Theta}_T(\epsilon_s)$ tighter than the true constraints. Whether Theorem \ref{sufficiently_large_error} generalize to higher dimensions is an open question we leave for future work, though in Appendix \ref{sec:highd}, we discuss stylized settings in which we expect that the $\tilde{O}_T(\sqrt{T})$ regret bounds from Theorem \ref{sufficiently_large_error} will generalize to higher dimensions. 

We also note that our algorithms require knowledge of $T$ in advance, as the value of $T$ determines the length of time spent in the warm-up exploration period. We expect that similar results will hold when $T$ is not known in advance, however this would require periods of exponentially growing length that alternate exploration versus exploitation (similar to as done in, e.g. \citet{li2021safe}). Because this greatly increases the complexity of the algorithm and analysis, we state and prove our results for $T$ known in advance. Finally, while we study general baselines that are more powerful than just safe linear controllers, the question of whether we can achieve $\tilde{O}_T(\sqrt{T})$ (or even $\tilde{O}_T(T^{2/3})$) regret on top of the cost of the best possible among \emph{all} safe controllers is still open.

\section*{Acknowledgements}

The authors would like to thank Na Li and Shahriar Talebi for helpful discussions. B.S. and L.J. received funding from NSF grant CBET-2112085 and B.S. received funding from the National Science Foundation Graduate Research Fellowship grant DGE 2140743.

\bibliographystyle{plainnat}

\bibliography{reference_arxiv}

\appendix

\section{Notation}\label{app:notation}

\subsection{Big O Notation}
Throughout this paper, we use notation such as $o_T(\cdot)$, $O_T(\cdot)$, $\omega_T(\cdot)$, $\Omega_T(\cdot)$. 
\begin{itemize}
    \item $f(T) = O_T(g(T))$ if there exists $T_0$ and $M \in \mathbb{R}$ such that for $T \ge T_0$, $f(T) \le M\cdot g(T)$.
    \item $f(T) = o_T(g(T))$ if for every constant $\epsilon > 0$ there exists $T_0$ such that  for all $T \ge T_0$,  $f(T) \le \epsilon \cdot g(T)$.
    \item $f(T) = \tilde{O}_T(g(T))$ if there exists $T_0$ and $k, M \in \mathbb{R}$ such that for $T \ge T_0$, $f(T) \le M\cdot g(T) \cdot \log^k(T)$.
\end{itemize}
Note that $\Omega_T, \omega_t,$ and $\tilde{\Omega}_T$ are defined in the same way but with the inequality reversed. While this is standard notation, we want to highlight exactly how we are using this notation in our proofs. First, we note that the subscript $T$ is included to indicate that we will always be using this notation with respect to the variable $T$. Furthermore, we note that the constant $M$ that is ``hidden" by the big-O notation will always be a function of known problem specification parameters, such as $q,r,\Theta, \mathcal{D}, D$. Therefore, if an expression includes an $O_T(1)$ term, this constant does not depend on any other variables in the expression. For example, suppose we state that for all $K$,  $f(K) \le O_T(\sqrt{T})$. Then this means that there exists $T_0$ and $M$ (where $M$ is a function of known problem specification parameters) such that for all $K$ and $T \ge T_0$, $f(K) \le M \cdot \sqrt{T}$. Furthermore, we will use notation such as $f(T) = O_T(\epsilon)$ to mean that there exists $T_0$ and $M$ such that $f(T) \le M \cdot \epsilon$ for $T \ge T_0$, where $M$ does not depend on $\epsilon$ and only depends on the problem specification parameters $\{q,r,\Theta, \mathcal{D}, D\}$. Finally, note that we will use the computer science notation of $O_T()$, in that the functions $f(T)$ and $g(t)$  will always be non-negative.

\subsection{Miscellaneous Notation}
Throughout the proofs, any inequalities or equations involving random variables will represent inequality or equality almost surely unless otherwise stated. Throughout the paper, we will use the notation $\{x_i\}_{i=1}^n$ to represent the unordered but indexed set of $x_1,x_2,...,x_n$.

\subsection{Problem Specifications}
The notation below will be used throughout the appendix, however the variables may depend on the algorithm being studied within a section. For example, the event $E$ is defined slightly differently for each of the two algorithms, and therefore the reader should note which algorithm each section addresses. The notation never changes within a single section.
\begin{itemize}
    \item $q,r$ : coefficients for the cost at time $t$ of $qx_t^2 + ru_t^2$.
    \item $W = \{w_t\}_{t=0}^{T-1}$ : The noise random variables for the $T$-length trajectory.
    \item $\mathcal{D}$ : Distribution of $w_t$ with CDF $F_{\mathcal{D}}$ and pdf upper bound $B_P$
    \item $\Theta = [\underline{a}, \bar{a}] \times [\underline{b}, \bar{b}]$ : Given set of dynamics s.t. $\theta^* \in \Theta$ ($\mathrm{size}(\Theta) = \min(\bar{a} - \underline{a}, \bar{b} - \underline{b})$) 
    \item $\theta^* = (a^*,b^*)$ : The true (unknown) dynamics.
    \item $C^{\mathrm{init}}$ : The initial safe controller satisfying Assumption \ref{assum_init}.
    \item $D = (D_{\mathrm{L}}, D_{\mathrm{U}})$ : the expected-position boundary for the safety constraint.
    \item A set of controls $\{u_t\}$ are safe for dynamics $\{\theta_t\}$ if for all $t$, $D_{\mathrm{L}} \le a_tx_t + b_tu_t \le D_{\mathrm{U}}$.
    \item $H_t = (x_0,u_0,x_1,u_1,...,u_{t-1}, x_t)$ and $\mathcal{F}_t = \sigma(H_t)$.
    \item $J(\theta, C,T,x,W)$ : The random variable cost of using controller $C$ starting at position $x_0 = x$ for $T$ time steps under dynamics $\theta$ with noise random variables $W$.
    \item $J^*(\theta, C,T) = J^*(\theta, C,T,0) = \E[J(\theta,C,T,x,W) \mid \theta, C, T, x]$ and $J^*(\theta, C,T) = J^*(\theta, C,T,0)$.
    \item $J^*(\theta, C) = J^*(\theta, C, 0) = \lim_{T \rightarrow \infty} J^*(\theta, C,T,0)$.
    \item $\mathcal{C}^{\theta} = \{C_K^{\theta}\}_{K \in [K^\theta_{\mathrm{L}}, K^\theta_{\mathrm{U}}]}$ : a class of controllers that are safe for dynamics $\theta$
    \item $K_{\mathrm{opt}}(\theta, T)$ : The $K$ that maximizes $J^*(\theta, C_K^\theta, T, 0)$ for $K \in [K^\theta_{\mathrm{L}}, K^\theta_{\mathrm{U}}]$.
    \item $K_{\mathrm{opt}}(\theta)$ : The $K$ that maximizes $J^*(\theta, C_K^\theta)$ for $K \in [K^\theta_{\mathrm{L}}, K^\theta_{\mathrm{U}}]$.
    \item $C_K^{\mathrm{unc}}$ : The unconstrained linear controller with parameter $K$, i.e. $C_K^{\mathrm{unc}}(x) = -Kx$.
    \item $F_{\mathrm{opt}}(\theta)$ : The $K$ that maximizes $J^*(\theta, C_K^{\mathrm{unc}})$.
\end{itemize}

\subsection{Algorithm Notation}
\begin{itemize}
    \item $\nu_T$ : Algorithm specific parameter that is either $T^{-1/4}$ or $T^{-1/3}$.
    \item $s_e$ : The number of the last round of the safe exploitation phase.
    \item $T_s = \frac{2^s}{\nu_T^2}$ : The length \textbf{and} starting time of round $s$ of the safe exploitation phase. 
    \item $\epsilon_s$ : Uncertainty bound for $\theta^*$ used throughout the algorithm.
    \item $\hat{\theta}_s$ : An estimate of $\theta^*$ that is with high probability within $\epsilon_s$ distance of $\theta^*$
    \item $u_t^{\mathrm{safeU}}$ : Largest $u$ such that $\displaystyle\max_{\norm{\theta - \hat{\theta}_s}_{\infty} \le \epsilon_s} ax_t + bu \le D_{\mathrm{U}}$ 
    \item $u_t^{\mathrm{safeL}}$ : Smallest $u$ such that $\displaystyle\max_{\norm{\theta - \hat{\theta}_s}_{\infty} \le \epsilon_s} ax_t + bu \ge D_{\mathrm{L}}$.
    \item $C_s^{\mathrm{alg}}(x_t)$ : the controller that the algorithm uses in round $s$ of the safe exploitation phase with additional safety modifications, i.e. the algorithm in round $s$ of the safe exploitation phase uses control $u_t = \max\left(\min\left(C_s^{\mathrm{alg}}(x_t), u^{\mathrm{safeU}}_t\right),u^{\mathrm{safeL}}_t\right)$.
    \item $C^{\mathrm{alg}}$ : Controller of the corresponding algorithm as described in the previous point.
    \item $P(\theta, K,z)$ : See Assumption \ref{assum_sufficiently_large_error}.
\end{itemize}

\subsection{Proof Notation}
\begin{itemize}
    \item $ W_s = \{w_i\}_{i=   T_s}^{   T_{s+1}-1}$ : Noise random variables in the round $s$ of safe exploitation phase.
    \item $\left(C_{K^*}^{\theta^*}, \{C_{K^*_s}^{\theta^*}\}_{s=0}^{s_e}\right)$ : The expected cost minimizing set of controllers to use if the controller $C_{K^*}^{\theta^*}$ is used for the first $T_0$ steps and for time $t \ge T_0$, the controller used is $C_{K_s}^{\theta^*}$, where $s = \lfloor \log_2\left(t\nu_T^2\right)\rfloor$. The sequence $(x^*_0,x^*_1,...)$ are the corresponding positions of using these controllers.
    \item $(x'_0,x'_1,...)$ and $(u'_0,u'_1,...)$: Unless otherwise specified, these are the positions and controls of the algorithm being discussed in the current proof.
    \item $(\hat{x}_{T_0}, \hat{x}_{T_0+1},...)$ : Unless otherwise defined in the theorem/lemma statement, $\hat{x}_{   T_0},\hat{x}_{   T_0+1},...$ is the sequence of positions if the control at each time $t \ge T_0$ is $C^{\hat{\theta}_s}_{K_{\mathrm{opt}}(\hat{\theta}_s, T_s)}(x_t)$ for $s = \lfloor\log_2\left(t\nu_T^2\right)\rfloor$ and starting at $\hat{x}_{   T_0} = x'_{   T_0}$.
    \item $ E_{\mathrm{safe}} = \left\{ \forall t < T: D_{\mathrm{L}} \le a^*x'_t + b^*u'_t \le D_{\mathrm{U}}\right\}$ : Event that all controls are safe
    \item $E_1 = \left\{\forall t < T : |w_t| \le \log^2(T) \right\}$ : Event that all noise has magnitude less than $\log^2(T)$ 
    \item $E_0 = \left\{\forall s \le s_e : \norm{\theta^* - \hat{\theta}_s}_{\infty} \le \epsilon_s \right\}$ : Event that all estimates of $\theta^*$ are within $\epsilon_s$ of $\theta^*$.
    \item $E_2 = E_0 \bigcap \left\{ \max_{s \in [0:s_e]} \epsilon_s \le \tilde{O}_T(\nu_T)\right\}$.
    \item $E_2^s = \left\{\norm{\hat{\theta}_{s} - \theta^*}_\infty \le \epsilon_s \le c_T \cdot \nu_T \right\}$, where $c_T$ is the coefficient in the $\tilde{O}_T(\nu_T)$ of the definition of event $E_2$.
    \item $E = E_{\mathrm{safe}} \cap E_1 \cap E_2$
    \item $B_x = \log^3(T)$ : Used throughout the appendix to simplify notation.
\end{itemize}

\section{Additional Related Work}\label{app:related_works}
The constrained LQR problem is closely related to the problem of model predict control (MPC) with constraints. For example, there is a large body of work on robust model predictive control with known dynamics \citep{bemporad2007robust}. This is further extended to MPC with model uncertainties in robust adaptive MPC (RAMPC) in works such as \citet{kohler2019linear, lu2021robust}. There have also been significant work on stochastic MPC with soft constraints, for example \citet{mesbah2016stochastic, oldewurtel2008tractable}, which are closely related to the expected position constraints we use in this paper. In the context of constrained LQR with no noise, \citet{bemporad2002explicit} derive the optimal controller as a piece-wise affine function. In a different MPC setting with deterministic dynamics and noisy observations, \citet{muthirayan2022online} provide an algorithm that also achieves $O(T^{2/3})$ regret. Learning based MPC using an initial safe controller was also studied in \citet{koller2018learning}. MPC results on learning constraints include e.g. \citet{lorenzen2019robust, kohler2019linear}.  While these works provide algorithms to solve constrained optimization problems such as LQR, these works do not compare the asymptotic performance of their results to the optimal algorithm. In contrast, our work studies a similar problem but focuses on algorithmic regret analysis from an RL perspective, comparing our algorithm to some baseline representation of the ``best" algorithm.

The results in this paper are also closely related to general system identification, the idea of being able to (in any way) asymptotically estimate the unknown dynamics. There have been multiple works in this area including \citet{simchowitz2018learning, zhao2022adaptive, mania2020active}. A recent work closely related to the results of this paper is \citet{li2023non}, which describes learning rates for non-linear controllers in a similar setting. The results in \citet{li2023non}, however, require i.i.d. noise excitation in every step, while our uncertainty bounds after the warm-up phase actually require no such excitation. These works are most similar to our work in that our results rely on identifying the system dynamics to a high accuracy. However our focus is not simply on learning the system, but also on achieving provably low regret results. The new uncertainty bounds we use to achieve our results also apply to nonlinear controllers as in \citet{li2023non}, but our uncertainty bounds apply specifically to the setting with safety constraints.

\section{Proof of Theorem \ref{performance}}\label{app:proof_of_performance}

Before proving Theorem \ref{performance}, we extend Definition \ref{def:safe} to account for time-dependent dynamics.
\begin{definition}\label{safe_controls_timedep}
    A control $u_t$ and position $x_t$ are safe for dynamics $\theta_t$ if
    \[
        D_{\mathrm{L}} \le a_tx_t + b_tu_t \le D_{\mathrm{U}}.
    \]
    Similarly, a (possibly time-dependent) controller $C_t$ is safe for $T$ steps for dynamics $\{\theta_t\}$ if when the dynamics at time $t$ is $\theta_t$, the sequence of controls $C_0(H_0), C_1(H_1),...,C_{T-1}(H_{T-1})$ and the resulting positions $x_0,...,x_{T-1}$ are safe for dynamics $\theta_t$ at all times $t$.
\end{definition}

Note that in general, a controller being safe is a random event.

Theorem \ref{performance} makes two claims: the first is that Algorithm \ref{alg:cap} is safe for dynamics $\theta^*$ for all $T$ steps with high probability and the second bounds with high probability the regret of Algorithm \ref{alg:cap}. In Appendix \ref{sec:proofofsafety} we will prove the result about the safety of Algorithm \ref{alg:cap} and in Appendix \ref{sec:regret} we will prove the result about the regret of Algorithm \ref{alg:cap}.
\subsection{Proof of Safety of Algorithm \ref{alg:cap}}\label{sec:proofofsafety}

\begin{lemma}\label{safety_append}
    Under Assumptions \ref{assum_init}--\ref{parameterization_assum3} , Algorithm \ref{alg:cap} is safe for $T$ steps for dynamics $\theta^*$ with probability $1-o_T(1/T^2)$.
\end{lemma}

\begin{proof}

We will first analyze the warm-up exploration phase (the first loop in Algorithm \ref{alg:cap} in Lines \ref{line:warmup_start}--\ref{line:warmup_end}). If the control at time $t-1$ was safe for dynamics $\theta^*$ as in Definition \ref{safe_controls_timedep}, then with probability at least $1-O_T(\frac{1}{T^4})$, the next position satisfies 
\[
x_t \in \left[ D_{\mathrm{L}} - F_{\mathcal{D}}^{-1}(1-\frac{1}{T^4}), D_{\mathrm{U}} + F_{\mathcal{D}}^{-1}(1-\frac{1}{T^4})\right].
\]
By Assumption \ref{assum:initial} on the controller $C^{\mathrm{init}}$, $D_{\mathrm{L}} + \frac{b^*}{\log(T)} \le a^*x+b^*C^{\mathrm{init}}(x) \le D_{\mathrm{U}} - \frac{b^*}{\log(T)}$ for all $x \in \left[ D_{\mathrm{L}} - F_{\mathcal{D}}^{-1}(1-\frac{1}{T^4}), D_{\mathrm{U}} + F_{\mathcal{D}}^{-1}(1-\frac{1}{T^4})\right]$. In Lines \ref{line:warmup_start}--\ref{line:warmup_end} of Algorithm \ref{alg:cap} the control is $C^{\mathrm{init}}(x_t) + \frac{\phi_t}{\log(T)}$ and $|\phi_t| = 1$. Therefore, if at time $t-1$ the algorithm's control was safe, then with probability $1-O_T\left(\frac{1}{T^4}\right)$ the control at time $t$ will satisfy $D_{\mathrm{L}} \le a^*x_t + b^*u_t \le D_{\mathrm{U}}$ and be safe. Furthermore, at time $0$, the position is $x_0=0$, therefore the first control is safe. Using this as a base case in a proof by induction with a union bound over all $1/\nu_T^2$ time steps $t$ in this loop, with probability $1-O_T(1/T^3)$, the first $1/\nu_T^2$ steps will be safe for dynamics $\theta^*$.

Now we will analyze the second loop in Algorithm \ref{alg:cap} (Lines \ref{line:exploit_start}--\ref{line:exploit_end}). Define $s_e = \log_2(T\nu_T^2) - 1$. Define the event $E_0$ as
\begin{equation}\label{eq:E0}
    E_0 = \left\{\forall s \le s_e : \norm{\theta^* - \hat{\theta}_s}_{\infty} \le \epsilon_s \right\}.
\end{equation} 
These $\epsilon_s$ are less than the right hand side of the equation in Lemma \ref{v_to_use}, and therefore by Lemma \ref{v_to_use}, under Assumptions \ref{problem_specifications} and \ref{assum_init},
\begin{equation}\label{eq:e0bound}
    \P(E_0) \ge 1-o_T(1/T^2).
\end{equation}
Informally, the next event we define is the combination of event $E_0$ and the event that the $\epsilon_s$ (defined in Line \ref{line:epsilon} of Algorithm \ref{alg:cap}) are decreasing at a sufficiently fast rate, which we will prove in Lemma \ref{initial_uncertainty}. Define
\begin{equation}\label{eq:E2}
    E_2 = E_0 \bigcap \left\{ \max_{s \in [0:s_e]} \epsilon_s \le \tilde{O}_T(\nu_T)\right\}.
\end{equation}
\begin{lemma}\label{initial_uncertainty}
    Under Assumptions \ref{assum_init}--\ref{parameterization_assum3}, with probability $1-o_T(1/T^2)$
    \[
     \max_{s \in [0:s_e]} \epsilon_s \le \tilde{O}_T(\nu_T).
    \]
\end{lemma}
The proof of Lemma \ref{initial_uncertainty} can be found in Appendix \ref{proof:initial_uncertainty}. Combining Lemma \ref{v_to_use} and Lemma \ref{initial_uncertainty} with a union bound gives that
\begin{equation}\label{eq:e2bound}
    \P(E_2) \ge  1-o_T(1/T^2).
\end{equation}
Define the event $E_1$ as 
\begin{equation}\label{eq:E1}
    E_1 = \left\{\forall t < T : |w_t| \le \log^2(T) \right\}.
\end{equation}
By Assumption \ref{problem_specifications}, the noise is sub-Gaussian, and therefore there exists a constant $\alpha$ such that for any $t$ and $x$, $\P(w_t \ge x) \le 2\exp(-x^2/\alpha)$. Taking $x = \log^2(T)$ and a union bound over all $w_t$, we have that 
\begin{equation}\label{eq:e1bound}
    \P(E_1) \ge 1 - \sum_{t=0}^{T-1} 2\exp\left(-\log^4(T)/\alpha\right) = 1 - o_T\left(\frac{1}{T^{\log(T)}}\right).
\end{equation}
We need one last lemma before concluding the proof.
\begin{lemma}\label{lemma:L_less_than_U}
    Under Assumptions \ref{assum_init}--\ref{parameterization_assum3}, conditional on $E_1 \cap E_2$ and for sufficiently large $T$, if $u_{T_0-1}$ is safe for dynamics $\theta^*$, then for all $t \in [T_0, T]$,
    \[
        u_t^{\mathrm{safeL}} \le u_t^{\mathrm{safeU}}.
    \]
\end{lemma}
The proof of Lemma \ref{lemma:L_less_than_U} can be found in Appendix \ref{proof:lemma:L_less_than_U}.

Under event $E_0$, $\hat{\theta}_s$ satisfies $\norm{\theta^* -\hat{\theta}_s}_{\infty} \le \epsilon_s$ for all $s \in [0: s_e]$ (which recall are the $s$ in the second for loop of Algorithm \ref{alg:cap}). Therefore, by the choice of $u_t^{\text{safeU}}$ and $u_t^{\text{safeL}}$ in Lines \ref{line:safeU} and \ref{line:safeL}, it must be the case that $a^*x_t + b^*u_t^{\text{safeU}} \le D_{\mathrm{U}}$ and $a^*x_t + b^*u_t^{\text{safeL}} \ge D_{\mathrm{L}}$. By the choice of $u_t$ in Line \ref{line:safety} of Algorithm \ref{alg:cap}, if $u_t^{\mathrm{safeL}} \le u_t^{\mathrm{safeU}}$  then $u_t^{\mathrm{safeL}} \le u_t \le u_t^{\mathrm{safeU}}$. This implies that
\begin{equation}
    D_{\mathrm{L}} \le a^*x_t + b^*u_t \le D_{\mathrm{U}}.
\end{equation}
Therefore, by Lemma \ref{lemma:L_less_than_U}, under $E_1 \cap E_2 \cap \{\text{$u_{T_0-1}$ is safe for dynamics $\theta^*$}\}$, all controls used in the second for loop (Lines \ref{line:exploit_start}--\ref{line:exploit_end}) in Algorithm \ref{alg:cap} are safe for dynamics $\theta^*$. By a union bound combining Equations \eqref{eq:e2bound} and \eqref{eq:e1bound} and the first paragraph of this proof, we have that 
\[
    \P(E_1 \cap E_2 \cap \{\text{$u_{T_0-1}$ is safe for dynamics $\theta^*$}\}) = 1-o_T(1/T^2).
\]
Because all of the steps in Algorithm \ref{alg:cap} are part of either the first or second loop, and the first loop steps are safe for dynamics $\theta^*$ with probability $1-o_T(1/T^2)$ and the second loop steps are safe for dynamics $\theta^*$ with probability $1-o_T(1/T^2)$, a union bound gives that the overall algorithm is safe for dynamics $\theta^*$ with probability $1-o_T(1/T^2)$.
\end{proof}

\subsection{Proof of Regret Bound of Algorithm \ref{alg:cap}}\label{sec:regret}

\begin{proof}
Define the event $E_{\mathrm{safe}}$ as the event that the controls used by the algorithm are safe at all times. If $x'_0,x'_1,...$ and $u'_0,u'_1,...$ are respectively the positions and controls of the algorithm, we have that
\begin{equation}\label{eq:safeE}
  E_{\mathrm{safe}} = \left\{ \forall t < T: D_{\mathrm{L}} \le a^*x'_t + b^*u'_t \le D_{\mathrm{U}}\right\},
\end{equation}
and by Lemma \ref{safety_append} we have that $\P(E_{\mathrm{safe}}) = 1-o_T(1/T^2)$. 
Now, define the event $E$ as
\begin{equation}\label{eq:E}
    E = E_{\mathrm{safe}} \cap E_1 \cap E_2.
\end{equation}
A union bound combining Equations \eqref{eq:e1bound}  and \eqref{eq:e2bound} gives that
\begin{equation}\label{eq:Ebound}
    \P(E) = \P( E_{\mathrm{safe}} \cap E_1 \cap E_2) \ge 1-o_T(1/T^2).
\end{equation}
The rest of the proof of Theorem \ref{performance} will focus on proving that the regret of Algorithm \ref{alg:cap} is $\tilde{O}_T(T^{2/3})$ with conditional probability at least $1-o_T(1/T)$ given $E$. Let $C^{\mathrm{alg}}$ be the (time-dependent) controller of Algorithm \ref{alg:cap}. Then the total cost of using Algorithm \ref{alg:cap} is $T \cdot J(\theta^*, C^{\mathrm{alg}},T, 0, W)$, and the regret we are trying to bound is (as in Equation \eqref{eq:regret_rv} using the notation $K_{\mathrm{opt}}$ from Assumption \ref{safety_assum}),
\begin{equation}\label{eq:regret_form}
    T\cdot J(\theta^*, C^{\mathrm{alg}},T,0,W) - T \cdot \bar{J}(\theta^*, C^{\theta^*}_{K_{\mathrm{opt}}(\theta^*, T)}, T).
\end{equation}
Define $W_s$ as the noise random variables from time $   T_s$ to $   T_{s+1} - 1$, so
\begin{equation}
    W_s = \{w_i\}_{i=   T_s}^{   T_{s+1}-1}.
\end{equation}
For any tuple $\left(K, \{K_s\}_{0 \le s \le s_e}\right)$ where $K, K_s \in (K_L^{\theta^*}, K_U^{\theta^*})$, define $x_0^{\left(K, \{K_s\}_{0 \le s \le s_e}\right)}, x_1^{\left(K, \{K_s\}_{0 \le s \le s_e}\right)},...$ as the random variable sequence of positions that result from starting at $x_0 = 0$ and using the controller that at each time $t < T_0$ uses controller $C^{\theta^*}_{K}$ and at each time $t \ge T_0$ uses the controller $C_{K_s}^{\theta^*}$, where $s = \lfloor \log_2\left(t\nu_T^2\right)\rfloor$. Define $\left(K^*, \{K^*_s\}_{0 \le s \le s_e}\right)$ as follows:
\begin{align*}
&\left(K^*, \{K_s^*\}_{0 \le s \le s_e}\right) \\
&= \argmin_{\left(K, \{K_s\}_{0 \le s \le s_e}\right)}  \E\left[\frac{1}{\nu_T^2}J\left(\theta^*,C^{\theta^*}_{K}, \frac{1}{\nu_T^2}, 0, \{w_t\}_{t=0}^{T_0-1}\right) + \sum_{s=0}^{s_e} T_sJ(\theta^*,C^{\theta^*}_{K_s}, T_s,  x_{   T_s}^{\left(K, \{K_s\}_{0 \le s \le s_e}\right)}, W_s)\right].
\end{align*}
Here the expectation is taken over both $w_t$ and $W_s$ (and recall that $x_{T_s}$ is a deterministic function of the $w_t$ and $W_s$ because $C_K^\theta$ is non-random for all $K,\theta$). We then define $x_0^*,x_1^*,...$ as the random variable sequence of positions such that $x_t^* = x_t^{\left(K^*, \{K^*_s\}_{0 \le s \le s_e}\right)}$. By construction, we could choose $K, K_s = K_{\mathrm{opt}}(\theta^*, T)$ for every $s$, and therefore it must be the case that
\[
\E\left[\frac{1}{\nu_T^2}J\left(\theta^*,C^{\theta^*}_{K^*}, \frac{1}{\nu_T^2}, 0, \{w_t\}_{t=0}^{T_0-1}\right) + \sum_{s=0}^{s_e} T_sJ(\theta^*,C^{\theta^*}_{K^*_s}, T_s,  x^*_{   T_s}, W_s)\right] \le T \cdot  \bar{J}\left(\theta^*, C^{\theta^*}_{K_{\mathrm{opt}}(\theta^*, T)}, T\right).
\]
Therefore, upper bounding the cost of Algorithm \ref{alg:cap} minus the cost of using $K^*$ for $T_0$ steps and then using the sequence of controllers $\{C^{\theta^*}_{K^*_s}\}$ each for $T_s$ steps is sufficient for upper bounding the regret in Equation \eqref{eq:regret_form}. Now we will bound
\begin{equation}\label{eq:goal_regret}
    T\cdot J(\theta^*, C^{\mathrm{alg}},T, 0, W) - \E\left[\frac{1}{\nu_T^2}J\left(\theta^*,C^{\theta^*}_{K^*}, \frac{1}{\nu_T^2},0,  \{w_t\}_{t=0}^{T_0-1}\right) + \sum_{s=0}^{s_e} T_sJ(\theta^*,C^{\theta^*}_{K^*_s}, T_s,  x^*_{   T_s}, W_s)\right].
\end{equation}
Note that we will upper bound the cost in terms of the parameter $\nu_T = T^{-1/3}$ in Line \ref{line:nu}. In order to bound the quantity in Equation \eqref{eq:goal_regret}, we will break this component of regret into four sources: the regret from the warm-up period (Lines \ref{line:warmup_start}--\ref{line:warmup_end}), the regret from using the estimates $\hat{\theta}_s$ instead of using $\theta^*$, the regret induced by the randomness of the trajectory, and the regret from enforcing safety.

The first source of regret is the regret incurred in the warm-up period of Algorithm \ref{alg:cap} (Lines \ref{line:warmup_start}--\ref{line:warmup_end}). Recall that $C^{\mathrm{alg}}_s$ is the controller used in Algorithm \ref{alg:cap} in the $s$ iteration of the second for loop. We will use Proposition \ref{warm_up_regret} to bound the cost incurred during the warm-up period.

\begin{proposition}[Regret from Warm-up Period]\label{warm_up_regret}
    Define $x'_0, x'_1,...$ as the sequence of random variables that are the positions of the controller $C^{\mathrm{alg}}$ defined in Algorithm \ref{alg:cap}. Define $R_0$ as the cost of the first $1/\nu_T^2$ steps, i.e.
    \begin{equation}
        R_0 = T \cdot J(\theta^*, C^{\mathrm{alg}}, T, 0, W) - \sum_{s=0}^{s_{e}}  T_s \cdot J(\theta^*,C^{\mathrm{alg}}_s, T_s,  x'_{   T_s}, W_s).
    \end{equation}
    Then under Assumptions \ref{assum_init}--\ref{parameterization_assum3} and conditional on event $E$,
    \[
        R_0 \stackrel{\text{a.s.}}{\le} \tilde{O}_T\left(\frac{1}{\nu_T^2} \right).
    \]
\end{proposition}
The proof of Propposition \ref{warm_up_regret} can be found in Appendix \ref{proof:warm_up_regret}.
The second source of regret in Equation \eqref{eq:goal_regret} is that Algorithm \ref{alg:cap} uses a controller $C^{\hat{\theta}_s}_{K_{\mathrm{opt}}(\hat{\theta}_s, T_s)}$ instead of the controller $C_{K^*_s}^{\theta^*}$. This source of regret (denoted $R_1$) can be interpreted as the ``estimation cost" of using the estimated controller instead of the optimal controller, but without enforcing safety. We will use Proposition \ref{non_optimal_controller} to bound this source of regret.

\begin{proposition}[Regret from Non-optimal Controller]\label{non_optimal_controller}
 Define $R_1$ as 
 \[
    R_1 :=  \sum_{s=0}^{s_e}  \E\left[T_sJ(\theta^*,C^{\hat{\theta}_s}_{K_{\mathrm{opt}}(\hat{\theta}_s, T_s)}, T_s, 0, W_s) \cond \hat{\theta}_s \right] -  \E\left[\sum_{s=0}^{s_e} T_sJ(\theta^*,C^{\theta^*}_{K^*_s}, T_s,  x^*_{   T_s}, W_s)  \right].
 \]
 Note that $W_s$ is independent of $\hat{\theta}_s$ by construction. Then under Assumptions \ref{assum_init}--\ref{parameterization_assum3} and conditional on event $E_2$,
    \begin{equation}\label{eq:non_optimal_controller}
       R_1 \stackrel{\text{a.s.}}{\le} \tilde{O}_T\left(T\nu_T\right).
    \end{equation}
\end{proposition}
The proof of Proposition \ref{non_optimal_controller} can be found in Appendix \ref{proof:non_optimal_controller}.
It may appear odd that the starting positions of the two terms do not match in the definition of $R_1$ (or in the definition of $R_2$ below), but we do account for this difference in the proofs of Propositions \ref{non_optimal_controller} and \ref{r1b_bound}. The third source of regret (which we will denote $R_2$) comes from the fact that in Equation \eqref{eq:goal_regret} we are comparing the random variable $T\cdot J(\theta^*, C^{\mathrm{alg}},T,0,W)$ to an expectation. In order to show that this source of regret is small, we need to show a concentration inequality for the cost of repeatedly using controllers of the form $C_{K_{\mathrm{opt}}(\hat{\theta}_s,  T_s)}^{\hat{\theta}_s}$, which we do in Proposition \ref{r1b_bound}.

\begin{proposition}[Regret from Randomness]\label{r1b_bound}
    Define $\hat{x}_{   T_0},\hat{x}_{   T_0+1},...$ as the sequence of random variables representing the sequence of positions if the control at each time $t \ge    T_0$ is $C^{\hat{\theta}_s}_{K_{\mathrm{opt}}(\hat{\theta}_s, T_s)}(x_t)$ for $s = \lfloor\log_2\left(t\nu_T^2\right)\rfloor$ and starting at $\hat{x}_{   T_0} = x'_{   T_0}$. Define $R_2$ as 
    \[
        R_2 := \sum_{s=0}^{s_e}  T_sJ(\theta^*,C^{\hat{\theta}_s}_{K_{\mathrm{opt}}(\hat{\theta}_s, T_s)}, T_s,  \hat{x}_{   T_s}, W_s) -  \sum_{s=0}^{s_e}  \E\left[T_sJ(\theta^*,C^{\hat{\theta}_s}_{K_{\mathrm{opt}}(\hat{\theta}_s, T_s)}, T_s,  0, W_s) \; \cond \; \hat{\theta}_s \right].
    \]
    Then with conditional probability $1-o_T(1/T)$ given event $E$,
    \begin{equation}\label{eq:r1b_bound}
         R_2 \le \tilde{O}_T(\sqrt{T}).
    \end{equation}
\end{proposition}
The proof of Proposition \ref{r1b_bound} can be found in Appendix \ref{proof:r1b_bound}.
 The final source of regret in Equation \eqref{eq:goal_regret} is the extra cost incurred by enforcing safety in Algorithm \ref{alg:cap} (Line \ref{line:safety}) rather than using the control given by $C_{K_{\mathrm{opt}}(\hat{\theta}_s, T_s)}^{\hat{\theta}_s}$. Each time we enforce safety we potentially incur an extra cost, but Proposition \ref{enforcing_safety} bounds this extra cost. 
\begin{proposition}[Regret from Enforcing Safety]\label{enforcing_safety}
   Define $\hat{x}_{   T_0},\hat{x}_{   T_0+1},...$ as the sequence of random variables representing the sequence of positions if the control at each time $t \ge    T_0$ is $C^{\hat{\theta}_s}_{K_{\mathrm{opt}}(\hat{\theta}_s, T_s)}(x_t)$ for $s = \lfloor\log_2\left(t\nu_T^2\right)\rfloor$ and starting at $\hat{x}_{   T_0} = x'_{   T_0}$. Define $R_3$ as (the random variable)
    \[
        R_3 := \sum_{s=0}^{s_e}  T_sJ(\theta^*,C^{\mathrm{alg}}_s, T_s,  x'_{   T_s}, W_s) - \sum_{s=0}^{s_e} T_sJ(\theta^*,C^{\hat{\theta}_s}_{K_{\mathrm{opt}}(\hat{\theta}_s, T_s)}, T_s,  \hat{x}_{   T_s}, W_s).
    \]
    Then under Assumptions \ref{assum_init}--\ref{parameterization_assum3}, with conditional probability $1-o_T(1/T)$ given event $E$,
    \[
        R_3 \le \tilde{O}_T(\nu_T T).
    \]
\end{proposition}
The proof of Proposition \ref{enforcing_safety} can be found in Appendix \ref{proof:enforcing_safety}.
Now we are ready to combine all of the sources of regret. To summarize, we have bounded and broken down the regret into
\begin{align*}
   & T\cdot J(\theta^*, C^{\mathrm{alg}},T,0,W) - T \cdot \bar{J}(\theta^*, C^{\theta^*}_{K_{\mathrm{opt}}(\theta^*, T)}, T) \\
   &\le T\cdot J(\theta^*, C^{\mathrm{alg}},T,0,W) - \E\left[\frac{1}{\nu_T^2}J\left(\theta^*,C^{\theta^*}_{K^*}, \frac{1}{\nu_T^2}, 0, \{w_t\}_{t=0}^{1/\nu_T^2-1}\right) + \sum_{s=0}^{s_e} T_sJ(\theta^*,C^{\theta^*}_{K^*_s}, T_s,  x^*_{   T_s}, W_s)\right]  \\
   &\le T\cdot J(\theta^*, C^{\mathrm{alg}},T,0,W) - \E\left[\sum_{s=0}^{s_e} T_s\bar{J}(\theta^*,C^{\theta^*}_{K^*_s}, T_s,  x^*_{   T_s}, W_s)\right]  \\
    &= \underbrace{\sum_{s=0}^{s_e}  \E\left[T_sJ(\theta^*,C^{\hat{\theta}_s}_{K_{\mathrm{opt}}(\hat{\theta}_s, T_s)}, T_s,  0, W_s) \cond \hat{\theta}_s \right] - \E\left[ \sum_{s=0}^{s_e} T_sJ(\theta^*,C^{\theta^*}_{K^*_s}, T_s,  x^*_{   T_s}, W_s) \right]}_{R_1} \\
    &\quad \quad \quad \quad + \underbrace{\sum_{s=0}^{s_e}  T_sJ(\theta^*,C^{\hat{\theta}_s}_{K_{\mathrm{opt}}(\hat{\theta}_s, T_s)}, T_s,  \hat{x}_{   T_s}, W_s) -  \sum_{s=0}^{s_e}  \E\left[T_sJ(\theta^*,C^{\hat{\theta}_s}_{K_{\mathrm{opt}}(\hat{\theta}_s, T_s)}, T_s,  0, W_s) \cond \hat{\theta}_s\right]}_{R_{2}} \\
    &\quad \quad \quad \quad + \underbrace{\sum_{s=0}^{s_e}  T_sJ(\theta^*,C_s^{\mathrm{alg}}, T_s,  x'_{   T_s}, W_s) - \sum_{s=0}^{s_e} T_sJ(\theta^*,C^{\hat{\theta}_s}_{K_{\mathrm{opt}}(\hat{\theta}_s, T_s)}, T_s,  \hat{x}_{   T_s}, W_s)}_{R_{3}} \\
    &\quad \quad \quad \quad + \underbrace{T \cdot J(\theta^*, C^{\mathrm{alg}}, T,0,W) - \sum_{s=0}^{s_e}  T_sJ(\theta^*,C^{\mathrm{alg}}_s, T_s,  x'_{   T_s}, W_s)}_{R_0}. \numberthis \label{eq:all_decomposition}
\end{align*}

Now we will use Propositions \ref{warm_up_regret}, \ref{non_optimal_controller}, \ref{r1b_bound}, and \ref{enforcing_safety} to bound the above quantity. Conditional on event $E$,  Proposition \ref{warm_up_regret} and Proposition \ref{non_optimal_controller} respectively imply that $R_0 \le \tilde{O}_T(1/\nu_T^2)$ and $R_1 \le \tilde{O}_T(\nu_TT)$. Proposition \ref{r1b_bound} and Proposition \ref{enforcing_safety} respectively imply that conditional on event $E$ with conditional probability $1-o_T(1/T)$, $R_2 \le \tilde{O}_T(\sqrt{T})$ and $R_3 \le \tilde{O}_T(\nu_TT)$. Therefore, applying a union bound gives that the bounds on $R_0,R_1,R_2,R_3$ all hold conditional on event $E$ with probability $1-o_T(1/T)$. Putting these bounds into Equation \eqref{eq:all_decomposition}, we have that conditional on event $E$ with probability $1-o_T(1/T)$,
\[
   T\cdot J(\theta^*, C^{\mathrm{alg}},T,0,W) - T \cdot \bar{J}(\theta^*, C^{\theta^*}_{K_{\mathrm{opt}}(\theta^*, T)}, T) \le R_1 + R_{2} + R_3 + R_0 \le \tilde{O}_T\left(\sqrt{T} + \frac{1}{\nu_T^2} + T\nu_T\right).
\]
Choosing $\nu_T = T^{-1/3}$ (as in Algorithm \ref{alg:cap}) will minimize this regret upper bound giving a total regret upper bound of $\tilde{O}_T(T^{2/3})$. Because the probability of event $E$ is $1-o_T(1/T)$, by a union bound the regret bound holds with unconditional probability $1-o_T(1/T)$.
\end{proof}

\section{Proofs of Propositions from Appendix \ref{app:proof_of_performance}}\label{sec:theorem1_prop}

\subsection{Proof of Proposition \ref{warm_up_regret} (Regret of Warm-up)}\label{proof:warm_up_regret}
\begin{proof}
To bound the cost of the warm-up phase, we need the following lemma. Informally, Lemma \ref{bounded_approx} shows that when the noise is relatively small and the controller is ``close" to being safe with respect to dynamics $\theta^*$, the position stays relatively small. Note that in this lemma we define $B_x := \log^3(T)$, which we will use throughout the proofs in the rest of the appendices.

\begin{lemma}\label{bounded_approx}
    Let $|x_0| \le 4\log^2(T)$. Suppose for all $t < T$, the control used by controller $C_t$ at time $t$ is safe for fixed dynamics $\theta_t$ and for all $t \le T$,
    \begin{equation}\label{eq:bounded_approx1}
        \norm{\theta^* - \theta_t}_{\infty} \le \frac{1}{\log(T)}.
    \end{equation}
    Then under Assumptions \ref{assum_init}--\ref{parameterization_assum3},  for sufficiently large $T$ and conditioned on event $E_1$, using this controller $C_t$ with dynamics $\theta^*$ for $T$ steps starting at $x_0$ will give positions ($x_0,...,x_{T}$) and controls ($u_0,...,u_{T-1}$) satisfying the following equations.
    \begin{equation}\label{eq:bounded_approx2a}
        |x_t|\stackrel{\text{a.s.}}{\le} 4\log^2(T) < \log^3(T) := B_x
    \end{equation}
    \begin{equation}\label{eq:bounded_approx2b}
        |u_t| \stackrel{\text{a.s.}}{\le} O_T(\log^2(T)) < \log^3(T) := B_x.
    \end{equation}
    Furthermore, if $x_0$ and the controller $C_t$ are deterministic, then the  positions ($x_0,...,x_{T}$) and controls ($u_0,...,u_{T-1}$) satisfy
\begin{equation}\label{eq:bounded_approx2c}
        \E[|x_t|] \le 4\log^2(T) < \log^3(T) := B_x
    \end{equation}
    \begin{equation}\label{eq:bounded_approx2d}
        \E[|u_t|] \le O_T(\log^2(T)) < \log^3(T) := B_x.
    \end{equation}

\end{lemma}
The proof of Lemma \ref{bounded_approx} can be found in Appendix \ref{proof:bounded_approx}.

Now we will use this lemma to bound the total cost of the warm-up phase of the algorithm. The controller for the first $1/\nu_T^2$ steps is safe for dynamics $\theta^*$ under event $E$ as shown in Lemma \ref{safety_append}. This means by Lemma \ref{bounded_approx}, conditional on event $E$, the position and controls during this warm-up period are both bounded in magnitude by $B_x$ (defined in Lemma \ref{bounded_approx}) almost surely for sufficiently large $T$. Because the cost at time $t$ is $qx_t^2 + ru_t^2$, this implies that the total cost of the first $1/\nu_T^2$ steps is upper bounded by $O_T((q+r)\frac{B_x^2}{\nu_T^2}) = \tilde{O}_T(1/\nu_T^2)$. 
\end{proof}

\subsection{Proof of Proposition \ref{non_optimal_controller} (Regret of Non-optimal Controller)}\label{proof:non_optimal_controller}

\begin{proof}
 First, we will use Lemma \ref{start_invariant_inexpectation} to rewrite the expression in Proposition \ref{non_optimal_controller} in a form amenable to Assumption \ref{parameterization_assum2}.
    \begin{lemma}\label{start_invariant_inexpectation}
        Under Assumptions \ref{assum_init}--\ref{parameterization_assum3} , for every $s \in [0: s_e]$ the following hold.
    \begin{equation}\label{eq:nonopt1}
        \left|\E\left[T_sJ(\theta^*,C^{\theta^*}_{K^*_s}, T_s,  x^*_{   T_s}, W_s)\right] - \E\left[T_sJ(\theta^*,C^{\theta^*}_{K^*_s}, T_s, 0, W_s)\right]\right| \le \tilde{O}_T(1)
    \end{equation}
    \end{lemma}
    The proof of Lemma \ref{start_invariant_inexpectation} can be found in Appendix \ref{proof:start_invariant_inexpectation}.
   By Lemma \ref{initial_uncertainty}, there exists a $c_T = \tilde{O}_T(1)$ such that under event $E_2$, $\max_s \epsilon_s \le c_T \cdot \nu_T$. For $s \in [0: s_e]$, define
    \begin{equation}
        E_2^s = \left\{\norm{\hat{\theta}_{s} - \theta^*}_\infty \le \epsilon_s \le c_T \cdot \nu_T \right\}.
    \end{equation}
    Informally, the event $E_2^s$ is the event that the bounds in event $E_2$ hold at time $s$. Note that because $E_2^s \subseteq E_2$, by Equation \eqref{eq:e2bound},
    \begin{equation}\label{eq:e2bounds}
        \P(E_2^s) \ge \P(E_2) \ge 1-o_T(1/T^2).
    \end{equation}
    We will also use the following application of Assumption \ref{parameterization_assum2} that holds under event $E_2^s$. Conditional on event $E_2^s$,
\begin{align*}
    &\left|\E\left[T_sJ(\theta^*,C^{\hat{\theta}_s}_{K_{\mathrm{opt}}(\hat{\theta}_s, T_s)}, T_s, 0, W_s) - T_sJ(\theta^*,C^{\theta^*}_{K^*_s}, T_s, 0, W_s)\cond \hat{\theta}_s \right]\right| \\
    &= \left|T_s\bar{J}(\theta^*,C^{\hat{\theta}_s}_{K_{\mathrm{opt}}(\hat{\theta}_s, T_s)}, T_s) - T_s\bar{J}(\theta^*,C^{\theta^*}_{K^*_s}, T_s)\right| \\
    &\le \tilde{O}_T\left(T_s\epsilon_s + \frac{T_s}{T^2}\right). && \text{Assumption \ref{parameterization_assum2}} \numberthis \label{eq:appl_of_7}
\end{align*}
    We can now use the triangle inequality with Equation \eqref{eq:nonopt1} to rewrite the left side of Equation \eqref{eq:non_optimal_controller} and apply Equation \eqref{eq:appl_of_7}. Formally, conditional on event $E_2$,
    {\small
    \begin{align*}
        &\sum_{s=0}^{s_e} \E\left[T_sJ(\theta^*,C^{\hat{\theta}_s}_{K_{\mathrm{opt}}(\hat{\theta}_s, T_s)}, T_s, 0, W_s)\cond \hat{\theta}_s  \right] -\E\left[ \sum_{s=0}^{s_e}  T_sJ(\theta^*,C^{\theta^*}_{K^*_s}, T_s,  x^*_{   T_s}, W_s) \right]  \\
        &=\sum_{s=0}^{s_e} \E\left[T_sJ(\theta^*,C^{\hat{\theta}_s}_{K_{\mathrm{opt}}(\hat{\theta}_s, T_s)}, T_s, 0, W_s)\cond \hat{\theta}_s  \right] - \sum_{s=0}^{s_e}  \E\left[T_sJ(\theta^*,C^{\theta^*}_{K^*_s}, T_s,  x^*_{   T_s}, W_s) \right]  \\
        &\le \tilde{O}_T(1) + \sum_{s=0}^{s_e} \E\left[T_sJ(\theta^*,C^{\hat{\theta}_s}_{K_{\mathrm{opt}}(\hat{\theta}_s, T_s)}, T_s, 0, W_s) \cond \hat{\theta}_s \right] - \sum_{s=0}^{s_e} \E\left[ T_sJ(\theta^*,C^{\theta^*}_{K^*_s}, T_s, 0, W_s) \right] && \text{By Equation  \eqref{eq:nonopt1}}  \\
        &= \tilde{O}_T(1) + \sum_{s=0}^{s_e} \E\left[T_sJ(\theta^*,C^{\hat{\theta}_s}_{K_{\mathrm{opt}}(\hat{\theta}_s, T_s)}, T_s, 0, W_s) - T_sJ(\theta^*,C^{\theta^*}_{K^*_s}, T_s, 0, W_s)\cond \hat{\theta}_s  \right]  \\
        &\le \tilde{O}_T(1) + \sum_{s=0}^{s_e} \left|\E\left[T_sJ(\theta^*,C^{\hat{\theta}_s}_{K_{\mathrm{opt}}(\hat{\theta}_s, T_s)}, T_s, 0, W_s) - T_sJ(\theta^*,C^{\theta^*}_{K^*_s}, T_s, 0, W_s)\cond \hat{\theta}_s  \right]\right|  \\
        &\le \tilde{O}_T(1) + \tilde{O}_T\left(\sum_{s=0}^{s_e} T_s\epsilon_s + \frac{T_s}{T^2}\right) && \text{By Equation \eqref{eq:appl_of_7}} \\
        &\le \tilde{O}_T(T\nu_T).
    \end{align*}
    }
\end{proof}

\subsection{Proof of Proposition \ref{r1b_bound} (Concentration of Cost)}\label{proof:r1b_bound}

\begin{proof}
    The following lemma is a result of McDiarmid's inequality and shows that the random variable corresponding to $T_sJ(\theta^*,C^{\hat{\theta}_s}_{K_{\mathrm{opt}}(\hat{\theta}_s, T_s)}, T_s, 0, W)$ concentrates around a conditional expectation.

    \begin{lemma}\label{concentration_of_cond_exp}
        Under Assumptions \ref{assum_init}--\ref{parameterization_assum3} , for every $s \in [0: s_e]$ there exists an event $E^{\mathrm{M}}_s$ such that $E^{\mathrm{M}}_s$ depends only on the random variables in $W_s$ and $\hat{\theta}_s$, such that $E^{\mathrm{M}}_s \subseteq \{\forall t \in [T_s : T_{s+1}-1], |w_t| \le \log^2(T)\}$, and such that conditional on $E_2^s$, $\P(E^{\mathrm{M}}_s \mid \hat{\theta}_s) \ge 1-o_T(1/T^8)$ and for $\epsilon \ge 1/T$ and for sufficiently large $T$,
        \begin{align*}
       & \P\left(\left|T_sJ(\theta^*,C^{\hat{\theta}_s}_{K_{\mathrm{opt}}(\hat{\theta}_s, T_s)}, T_s, 0, W_s) - \E\left[T_sJ(\theta^*,C^{\hat{\theta}_s}_{K_{\mathrm{opt}}(\hat{\theta}_s, T_s)}, T_s,0,W_s) \cond E^{\mathrm{M}}_s, \hat{\theta}_s\right] \right| \ge \epsilon \cond \hat{\theta}_s \right) \\
        &\le \frac{1}{T^8} + 2\exp\left(-\frac{\epsilon^2}{2T_sc^2}\right)
        \end{align*}
        for some $c = \tilde{O}_T(1)$.
    \end{lemma}

The proof of Lemma \ref{concentration_of_cond_exp} can be found in Appendix \ref{proof:concentration_of_cond_exp}. We also want that taking expectation conditional on $E^{\mathrm{M}}_s$ does not significantly change the expected cost.

    \begin{lemma}\label{uncond_vs_cond_regret}
        Under Assumptions \ref{assum_init}--\ref{parameterization_assum3}, if $E^{\mathrm{M}}_s \subseteq \{\forall t \in [T_s : T_{s+1}-1], |w_t| \le \log^2(T)\}$ and conditional on event $E_2^s$ we have $\P(E^{\mathrm{M}}_s) \ge 1-o_T(1/T^8)$, then conditional on event $E_2^s$,
        \begin{equation}
            \E\left[T_sJ(\theta^*, C^{\hat{\theta}_s}_{K_{\mathrm{opt}}(\hat{\theta}_s, T_s)}, T_s, 0, W_s)\cond \hat{\theta}_s \right] \stackrel{\text{a.s.}}{\ge} \E\left[ T_sJ(\theta^*, C^{\hat{\theta}_s}_{K_{\mathrm{opt}}(\hat{\theta}_s, T_s)}, T_s, 0, W_s) \cond E^{\mathrm{M}}_s, \hat{\theta}_s\right] -  \tilde{O}_T(1),
        \end{equation}
        where the term $\tilde{O}_T(1)$ does not depend on $s$.
    \end{lemma}
    The proof of Lemma \ref{uncond_vs_cond_regret} can be found in Appendix \ref{proof:uncond_vs_cond_regret}.
    Combining Lemma \ref{concentration_of_cond_exp} for $\epsilon = c\sqrt{T_s}\log(T)$ and Lemma \ref{uncond_vs_cond_regret} for sufficiently large $T$, we have the following conditional on event $E_2^s$:
    \begin{align*}
         &\P\left( T_sJ(\theta^*,C^{\hat{\theta}_s}_{K_{\mathrm{opt}}(\hat{\theta}_s, T_s)}, T_s, 0, W_s) -  \E\left[T_sJ(\theta^*,C^{\hat{\theta}_s}_{K_{\mathrm{opt}}(\hat{\theta}_s, T_s)}, T_s, 0, W_s) \cond \hat{\theta}_s \right] \ge c\sqrt{T_s}\log(T) + \tilde{O}_T(1) \cond \hat{\theta}_s  \right)\\
         &\le \frac{1}{T^8} + 2\exp\left(-\frac{\log^2(T)}{2}\right). \numberthis \label{eq:one_Ws}
    \end{align*}
    Now applying a union bound over all $s \in [0:s_e]$ gives the following result:
    {\fontsize{10}{10}
    \begin{align*}
      &\P\left( \sum_{s=0}^{s_e} T_sJ(\theta^*,C^{\hat{\theta}_s}_{K_{\mathrm{opt}}(\hat{\theta}_s, T_s)}, T_s, 0, W_s) -   \sum_{s=0}^{s_e} \E\left[T_sJ(\theta^*,C^{\hat{\theta}_s}_{K_{\mathrm{opt}}(\hat{\theta}_s, T_s)}, T_s, 0, W_s) \cond \hat{\theta}_s \right] \ge \sum_{s=0}^{s_e} \left(c\sqrt{T_s}\log(T) + \tilde{O}_T(1)\right) \right)  \\
      &\le \P\left( \exists s \in [0: s_e] : T_sJ(\theta^*,C^{\hat{\theta}_s}_{K_{\mathrm{opt}}(\hat{\theta}_s, T_s)}, T_s, 0, W_s) -  \E\left[T_sJ(\theta^*,C^{\hat{\theta}_s}_{K_{\mathrm{opt}}(\hat{\theta}_s, T_s)}, T_s, 0, W_s)\cond \hat{\theta}_s  \right] \ge c\sqrt{T_s}\log(T) + \tilde{O}_T(1) \right)  \\
      &\le \sum_{s = 0}^{s_e} \P\left( T_sJ(\theta^*,C^{\hat{\theta}_s}_{K_{\mathrm{opt}}(\hat{\theta}_s, T_s)}, T_s, 0, W_s) -  \E\left[T_sJ(\theta^*,C^{\hat{\theta}_s}_{K_{\mathrm{opt}}(\hat{\theta}_s, T_s)}, T_s, 0, W_s)\cond \hat{\theta}_s  \right] \ge c\sqrt{T_s}\log(T) + \tilde{O}_T(1) \right)  \\
      &\le \sum_{s  = 0}^{s_e}  \left(\frac{1}{T^8} +  2\exp\left(-\frac{\log^2(T)}{2}\right) + \P(\neg E_2^{s}) \right) \quad \quad \quad \quad \quad \quad \quad  \text{Equation \eqref{eq:one_Ws}}\\
      &\le \tilde{O}_T\left(\frac{1}{T^2}\right). \quad \quad \quad \quad \quad \quad \quad \quad \quad \quad \quad \quad \quad \quad \quad \quad \quad \quad \quad \quad \quad \text{Equation \eqref{eq:e2bounds}} \numberthis \label{eq:firstmcdiarmapp}
    \end{align*}   
    }
    Note that
    \begin{equation}\label{eq:sumofsqrtT}
          \sum_{s=0}^{s_e} c\sqrt{T_s}\log(T) = \tilde{O}_T(\sqrt{T}),
    \end{equation}
    therefore combining Equations \eqref{eq:sumofsqrtT} and \eqref{eq:firstmcdiarmapp}, we have that 
    \begin{align*}
      &\P\left(\sum_{s=0}^{s_e}  T_sJ(\theta^*,C^{\hat{\theta}_s}_{K_{\mathrm{opt}}(\hat{\theta}_s, T_s)}, T_s, 0, W_s) - \sum_{s=0}^{s_e}  \E\left[T_sJ(\theta^*,C^{\hat{\theta}_s}_{K_{\mathrm{opt}}(\hat{\theta}_s, T_s)}, T_s, 0, W_s) \cond \hat{\theta}_s \right] \ge  \tilde{O}_T(\sqrt{T}) \right) \\
      &\le \tilde{O}_T\left(\frac{1}{T^2}\right). \numberthis \label{eq:final_concent} 
    \end{align*}
    Equation \eqref{eq:final_concent} differs from the desired result of Proposition \ref{r1b_bound} in that the first summation is over trajectories starting at position $0$ as opposed to $\hat{x}_{T_s}$. Therefore, the last part of this proof is to bound
    \[
     \left|\sum_{s=0}^{s_e}  T_sJ(\theta^*,C^{\hat{\theta}_s}_{K_{\mathrm{opt}}(\hat{\theta}_s, T_s)}, T_s,  \hat{x}_{   T_s}, W_s) - \sum_{s=0}^{s_e}  T_sJ(\theta^*,C^{\hat{\theta}_s}_{K_{\mathrm{opt}}(\hat{\theta}_s, T_s)}, T_s,  0, W_s)\right|.
    \]
     To do this, we will use the following lemma that is a consequence of Assumption \ref{parameterization_assum3}.
    \begin{lemma}\label{offbyepsilon}
    Under Assumptions \ref{assum_init}--\ref{safety_assum} and \ref{parameterization_assum3}, if $\norm{\theta - \theta^*}_{\infty} = \epsilon \le  \epsilon_{\mathrm{A}\ref{parameterization_assum3}}$, then for any $K \in (K_L^{\theta}, K_U^{\theta})$, $t \le T$, and $|x|, |y| \le 4\log^2(T) $ and any noise random variables $W'$, conditional on event $E_{\mathrm{A}\ref{parameterization_assum3}}(C_K^\theta, W')$,
    \[
        \left|t \cdot J(\theta^*,C_K^\theta,t,x, W') - t \cdot J(\theta^*,C_K^\theta,t,y, W') \right| = \tilde{O}_T\left(|x-y| + \epsilon\right).
    \]
    \end{lemma}
    The proof of Lemma \ref{offbyepsilon} can be found in Appendix \ref{proof:offbyepsilon}.

   In order to use Lemma \ref{offbyepsilon}, we must show that $|\hat{x}_{T_s}| \le 4\log^2(T)$. Recall that $\hat{x}_{T_s}$ is the position at time $T_s$ if the position at time $T_0$ is $\hat{x}_{T_0} = x'_{T_0}$, where $x'_{T_0}$ is the position of the controller $C^{\mathrm{alg}}$ at time $T_0$.  Because $E_{\mathrm{safe}} \subseteq E$, under event $E$ we have that $C^{\mathrm{alg}}$ is safe for dynamics $\theta^*$. Therefore by Lemma \ref{bounded_approx},  $|x'_{T_0}| \le 4\log^2(T)$. Because $E_2 \subseteq E$, under event $E$ we also have that $\norm{\hat{\theta}_s - \theta^*}_{\infty} \le \tilde{O}_T(\nu_T)$ for all $s \in [0: s_e]$ and sufficiently large $T$. Therefore, since $\hat{x}_{T_0} = x'_{T_0}$ and the control $C_{K_{\mathrm{opt}}(\hat{\theta}_s, T_s)}^{\hat{\theta}_s}(x)$ is safe with respect to $\hat{\theta}_s$ for any $x$, again by Lemma \ref{bounded_approx} we have that under event $E$ and for sufficiently large $T$, $|\hat{x}_{   T_s}| \le 4\log^2(T)$. 
   Now we can apply Lemma \ref{offbyepsilon} to get that, conditional on event $E \cap \bigcap_{s=0}^{s_e} E_{\mathrm{A}\ref{parameterization_assum3}}(C_{K_{\mathrm{opt}}(\hat{\theta}_s, T_s)}^{\hat{\theta}_s}, W_s)$,
    \begin{align*}
        &\left|\sum_{s=0}^{s_e}  T_sJ(\theta^*,C^{\hat{\theta}_s}_{K_{\mathrm{opt}}(\hat{\theta}_s, T_s)}, T_s,  \hat{x}_{   T_s}, W_s) - \sum_{s=0}^{s_e}  T_sJ(\theta^*,C^{\hat{\theta}_s}_{K_{\mathrm{opt}}(\hat{\theta}_s, T_s)}, T_s,  0, W_s)\right| \\
        &\le \sum_{s=0}^{s_e} \left| T_sJ(\theta^*,C^{\hat{\theta}_s}_{K_{\mathrm{opt}}(\hat{\theta}_s, T_s)}, T_s,  \hat{x}_{   T_s}, W_s) -  T_sJ(\theta^*,C^{\hat{\theta}_s}_{K_{\mathrm{opt}}(\hat{\theta}_s, T_s)}, T_s,  0, W_s)\right| \\
        &\le \sum_{s=0}^{s_e} \tilde{O}_T\left(\hat{x}_{   T_s} + \norm{\hat{\theta}_s - \theta^*}_{\infty}\right)\\ 
        &\le \tilde{O}_T(1). \numberthis \label{eq:application_of_offbyepsilon}
    \end{align*}
   A union bound gives that $\P(\bigcap_{s=0}^{s_e} E_{\mathrm{A}\ref{parameterization_assum3}}(C_{K_{\mathrm{opt}}(\hat{\theta}_s, T_s)}^{\hat{\theta}_s}, W_s)) = 1-o_T(1/T^2)$. Combining Equation \eqref{eq:final_concent} with Equation \eqref{eq:application_of_offbyepsilon} with a union bound gives that conditional on event $E$ with probability $1-o_T(1/T)$, 
    \begin{equation}\label{eq:event}
    \sum_{s=0}^{s_e}  T_sJ(\theta^*,C^{\hat{\theta}_s}_{K_{\mathrm{opt}}(\hat{\theta}_s, T_s)}, T_s, \hat{x}_{   T_s}, W_s) - \sum_{s=0}^{s_e} \E\left[ T_sJ(\theta^*,C^{\hat{\theta}_s}_{K_{\mathrm{opt}}(\hat{\theta}_s, T_s)}, T_s,0, W_s)\cond \hat{\theta}_s  \right] \le  \tilde{O}_T(\sqrt{T}),
    \end{equation}
    which is the desired result of Proposition \ref{r1b_bound}.    
\end{proof}

\subsection{Proof of Proposition \ref{enforcing_safety} (Regret of Enforcing Safety)}\label{proof:enforcing_safety}

\begin{proof}
Intuitively, $R_3$ is the regret caused by enforcing safety and deviating from the controller $C_{K_{\mathrm{opt}}(\hat{\theta}_s, T_s)}^{\hat{\theta}_s}$. Lemma \ref{offbyepsiloncontrol_propproof} bounds the cost of deviating from $C_{K_{\mathrm{opt}}(\hat{\theta}_s, T_s)}^{\hat{\theta}_s}$ as a sum over all times the algorithm deviates.

\begin{lemma}\label{offbyepsiloncontrol_propproof}
   Recall $u_t^{\mathrm{safeU}}$ and $u_t^{\mathrm{safeL}}$ defined in Algorithm \ref{alg:cap} Lines \ref{line:safeU} and \ref{line:safeL}. Let $X^U_t$ and $X^L_t$ be the indicators for the events that at time $t$, 
    $C^{\mathrm{alg}}(x'_t) = u_t^{\mathrm{safeU}}$ or $C^{\mathrm{alg}}(x'_t) = u_t^{\mathrm{safeL}}$, respectively. Under Assumptions \ref{assum_init}--\ref{parameterization_assum3} and conditional on event $E$, with probability $1-o_T(1/T)$
    
\begin{align*}
&\sum_{s=0}^{s_e}  T_sJ(\theta^*,C^{\mathrm{alg}}, T_s,  x'_{   T_s}, W_s) - \sum_{s=0}^{s_e} T_sJ(\theta^*,C^{\hat{\theta}_s}_{K_{\mathrm{opt}}(\hat{\theta}_s, T_s)}, T_s,  \hat{x}_{   T_s}, W_s) \\
&\le \tilde{O}_T\left(\sum_{s=0}^{s_e} \epsilon_sT_s \right) + \sum_{s=0}^{s_e} \sum_{t=T_s}^{T_{s+1}-1} X^U_{t} \cdot \tilde{O}_T\left(\left|u_t^{\mathrm{safeU}} - C_{K_{\mathrm{opt}}(\hat{\theta}_s, T_s)}^{\hat{\theta}_s}(x'_t)\right|\right) + X^L_{t}\cdot \tilde{O}_T\left(\left|u_t^{\mathrm{safeL}} - C_{K_{\mathrm{opt}}(\hat{\theta}_s, T_s)}^{\hat{\theta}_s}(x'_t)\right|\right).
\end{align*}
\end{lemma}
The proof of Lemma \ref{offbyepsiloncontrol_propproof} can be found in Appendix \ref{proof:offbyepsiloncontrol_propproof}. We also remind the reader that the coefficients of the $\tilde{O}_T(\cdot)$ terms in Lemma \ref{offbyepsiloncontrol_propproof} do not depend on $t$ or $s$, and are a function of known problem parameters and $\log(T)$ factors. The next tool we need is to be able to bound the difference in control when applying safety in Algorithm \ref{alg:cap} compared to the control when not applying safety. We can do that as follows.

\begin{lemma}\label{bound_on_cont_diff_propproof}
Under Assumptions \ref{assum_init}--\ref{parameterization_assum3} and conditional on event $E$, for any $t$ such that $1/\nu_T^2 \le t \le T$, if $s = \lfloor \log_2\left(t\nu_T^2\right)\rfloor$ and  $u_t^{\mathrm{safeU}} \le C^{\hat{\theta}_s}_{K_{\mathrm{opt}}(\hat{\theta}_s, T_s)}(x'_t)$ (which is equivalent to $C^{\mathrm{alg}}(x'_t) = u_t^{\mathrm{safeU}}$), then,
\begin{equation}\label{eq:enforcingsafety10app}
            |u_t^{\mathrm{safeU}} - C_{K_{\mathrm{opt}}(\hat{\theta}_s, T_s)}^{\hat{\theta}_s}(x'_t)|  \le \tilde{O}_T(\epsilon_{s}).
        \end{equation}
        Similarly, if $u_t^{\mathrm{safeL}} \ge C^{\hat{\theta}_s}_{K_{\mathrm{opt}}(\hat{\theta}_s, T_s)}(x'_t)$, then conditional on event $E$,
\begin{equation}\label{eq:enforcingsafety10bpp}
            |u_t^{\mathrm{safeL}} - C^{\hat{\theta}_s}_{K_{\mathrm{opt}}(\hat{\theta}_s, T_s)}(x'_t)|   \le \tilde{O}_T(\epsilon_s).
        \end{equation}
\end{lemma}
The proof of Lemma \ref{bound_on_cont_diff_propproof} can be found in section \ref{proof:bound_on_cont_diff_propproof}.
Combining Lemmas \ref{offbyepsiloncontrol_propproof} and \ref{bound_on_cont_diff_propproof}, we have that conditional on event $E$, with probability $1-o_T(1/T)$,

\begin{align*}
    R_{3} &= \sum_{s=0}^{s_e}  T_sJ(\theta^*,C^{\mathrm{alg}}, T_s,  x'_{   T_s}, W_s) - \sum_{s=0}^{s_e} T_sJ(\theta^*,C^{\hat{\theta}_s}_{K_{\mathrm{opt}}(\hat{\theta}_s, T_s)}, T_s,  \hat{x}_{   T_s}, W_s) \\
    & \le \tilde{O}_T\left(\sum_{s=0}^{s_e} \epsilon_sT_s \right)  + \sum_{s=0}^{s_e} \sum_{t=T_s}^{T_{s+1}-1} \Big( X^U_{t}\cdot\tilde{O}_T\left(\left|u_t^{\mathrm{safeU}} - C_{K_{\mathrm{opt}}(\hat{\theta}_s, T_s)}^{\hat{\theta}_s}(x'_t)\right|\right) \\
    &\quad \quad + X^L_{t}\cdot \tilde{O}_T\left(\left|u_t^{\mathrm{safeL}} - C_{K_{\mathrm{opt}}(\hat{\theta}_s, T_s)}^{\hat{\theta}_s}(x'_t)\right|\right)\Big) && \text{Lemma \ref{offbyepsiloncontrol_propproof}} \\
    &  \le \tilde{O}_T\left(\sum_{s=0}^{s_e} \epsilon_sT_s \right)  + \sum_{s=0}^{s_e} \sum_{t=T_s}^{T_{s+1}-1}X_{t}^U \cdot\tilde{O}_T\left(\epsilon_{s }\right) + X_t^L \cdot\tilde{O}_T\left(\epsilon_{s }\right) && \text{Lemma \ref{bound_on_cont_diff_propproof}}\\
    &  \le \tilde{O}_T(T\nu_T) +\sum_{t=1/\nu_T^2}^{T-1} X_{t}^U\cdot \tilde{O}_T\left(\nu_T\right) + X_t^L \cdot \tilde{O}_T(\nu_T) && \text{$E_2 \subseteq E$} \\
    &\le \tilde{O}_T( T\nu_T) + \tilde{O}_T(T\nu_T)\\
    &=  \tilde{O}_T(T\nu_T).
\end{align*}
The key application of event $E$ in the above result is that $E_2 \subseteq E$ implies that under event $E$, $\max_{s \in [0:s_e]} \epsilon_s = \tilde{O}_T(\nu_T) $.
\end{proof}

\section{Proofs of Lemmas from Appendix \ref{sec:theorem1_prop}}
\subsection{Proof of Lemma \ref{lemma:L_less_than_U}}\label{proof:lemma:L_less_than_U}
Recall the notation that $x'_t$ is the position at time $t$ when using the controller $C^{\mathrm{alg}}$.
We will prove the following. For sufficiently large $T$ and any $t \in [T_0: T]$, if  $C^{\mathrm{alg}}(x'_{t-1})$ is safe for dynamics $\theta^*$, then conditional on $E_1 \cap E_2$, we have that both $u_t^{\mathrm{safeL}} \le u_t^{\mathrm{safeU}}$ and $C^{\mathrm{alg}}(x'_t)$ is safe for dynamics $\theta^*$. Because we assume in this lemma that $u_{T_0 - 1} = C^{\mathrm{alg}}(x'_{T_0-1})$ is safe with respect to dynamics $\theta^*$, this will prove by induction the desired result that conditional on $E_1 \cap E_2$ and for sufficiently large $T$, $u_t^{\mathrm{safeL}} \le u_t^{\mathrm{safeU}}$ for all $t \in [T_0: T]$.

Fix a given $t$, and define $s = \lfloor \log_2(t/\nu_T^2) \rfloor$. Assume $C^{\mathrm{alg}}(x'_{t-1})$ is safe for dynamics $\theta^*$. Then under event $E_1$, we have that $|x'_t| \le \norm{D}_{\infty} + |w_{t-1}| \le B_x$. Let $v = \frac{D_{\mathrm{U}} - a^*x'_{t} - 4\epsilon_sB_x}{b^*}$. We will show that $u_t^{\mathrm{safeU}} \ge v$. Note that $a^*x'_t + b^*v = D_{\mathrm{U}} - 4\epsilon_sB_x$. For sufficiently large $T$, because $D_{\mathrm{U}} - D_{\mathrm{L}} \ge \frac{1}{\log(T)}$ (Assumption \ref{problem_specifications}) and $\epsilon_s = \tilde{O}_T(\nu_T) = o_T(1/\log(T))$ under $E_1 \cap E_2$, this implies that  
\[
    D_{\mathrm{L}} \le a^*x'_t + b^*v \le D_{\mathrm{U}}.
\]
Therefore $v$ is safe for dynamics $\theta^*$, which implies by Lemma \ref{bounded_approx} that under event $E_1$ and for sufficiently large $T$,
\[
    |v| \le B_x.
\]
Under event $E_1 \cap E_2$, $\norm{\theta^* - \hat{\theta}_s}_{\infty} \le \epsilon_s$, therefore by the above results we have that under $E_1 \cap E_2$ and for sufficiently large $T$,
 \begin{align*}
     \max_{\norm{\hat{\theta}_s - \theta}_{\infty} \le \epsilon_s} ax'_t + bv &\le a^*x'_t + b^*v + 2\epsilon_s|x'_t| + 2\epsilon_s |v| \\
     &\le  a^*x'_t + b^*v + 4\epsilon_sB_x && \text{$|v| \le B_x, |x'_t| \le B_x$} \\
    &= D_{\mathrm{U}}. && \text{Def of $v$}
 \end{align*}
 This implies by the definition of $u_t^{\mathrm{safeU}}$ that
 \[
    u_t^{\mathrm{safeU}} \ge v = \frac{D_{\mathrm{U}} - a^*x'_t - 4\epsilon_sB_x}{b^*}.
 \]
 By the same logic, we also have that
 \[
    u_t^{\mathrm{safeL}} \le \frac{D_{\mathrm{L}} - a^*x'_t + 4\epsilon_sB_x}{b^*}.
 \]
For sufficiently large $T$ under event $E_2$, $\frac{8\epsilon_s B_x}{b^*} = \tilde{O}_T(\nu_T) \le \frac{1}{\log(T)}$.  Therefore, using that $D_{\mathrm{U}} \ge D_{\mathrm{L}} + \frac{1}{\log(T)}$ by Assumption \ref{problem_specifications}, we can conclude that under event $E_1 \cap E_2$ for sufficiently large $T$, 
 \begin{align*}
     u_t^{\mathrm{safeL}} &\le \frac{D_{\mathrm{L}} - a^*x'_t + 4\epsilon_sB_x}{b^*} \\
     &\le \frac{D_{\mathrm{U}} - a^*x'_t - 4\epsilon_sB_x}{b^*} \\
     &\le u_t^{\mathrm{safeU}}.
 \end{align*}
 This implies that $u_t^{\mathrm{safeL}} \le C^{\mathrm{alg}}(x'_t) \le u_t^{\mathrm{safeU}}$, which by construction under event $E_1 \cap E_2$ implies that $D_{\mathrm{L}} \le a^*x'_t + b^*C^{\mathrm{alg}}(x'_t) \le D_{\mathrm{U}}$. Finally, this gives that $C^{\mathrm{alg}}(x'_t)$ is safe for dynamics $\theta^*$. Therefore, we have shown the two desired results that $u_t^{\mathrm{safeL}}  \le u_t^{\mathrm{safeU}}$ and $C^{\mathrm{alg}}(x'_t)$ is safe for dynamics $\theta^*$. 

 As mentioned above, this implies by induction the desired result that $u_t^{\mathrm{safeL}} \le u_t^{\mathrm{safeU}}$ for all $t \in [T_0,T]$ conditional on $E_1 \cap E_2$ as long as $C^{\mathrm{alg}}(x'_{T_0-1})$ is safe with respect to $\theta^*$.

\subsection{Proof of Lemma \ref{bounded_approx} (Bounded positions and controls)}\label{proof:bounded_approx}
\begin{proof}
    Define $\gamma_T = \max_{t \in [T]} \norm{\theta^* - \theta_t}_{\infty}$, and we know that $\gamma_T \le \frac{1}{\log(T)}$ by assumption. At time $t$, the control used by controller $C_t$ is safe for dynamics $\theta_t$ by assumption of the lemma, so by Definition \ref{safe_controls_timedep}, for all $t$, if $u_t = C_t(x_t)$ then
    \begin{equation}\label{eq:bounded_approx3a}
        D_{\mathrm{L}} \le a_tx_t + b_tu_t \le D_{\mathrm{U}}.
    \end{equation}
    By definition of $\gamma_T$, this implies that
    \begin{equation}\label{eq:bounded_approxeps}
        D_{\mathrm{L}} - \gamma_T |x_t| - \gamma_T |u_t| \le a^*x_t + b^*u_t \le D_{\mathrm{U}} + \gamma_T |x_t| + \gamma_T |u_t|.
    \end{equation}
    The right inequality in Equation \eqref{eq:bounded_approxeps} implies that 
    \[
        b^*u_t - \gamma_T |u_t| \le D_{\mathrm{U}} + \gamma_T |x_t| -  a^*x_t,
    \]
    which for $u_t \ge 0$ implies that $|u_t| \le \frac{\norm{D}_{\infty} + a^*|x_t| + \gamma_T|x_t|}{b^* - \gamma_T}$. The left inequality in Equation \eqref{eq:bounded_approxeps} implies the same for $u_t \le 0$, and therefore we have that Equation \eqref{eq:bounded_approxeps} implies that
    \begin{equation}\label{eq:bounded_approx3}
        |u_t| \le \frac{\norm{D}_{\infty} + a^*|x_t| + \gamma_T|x_t|}{b^* - \gamma_T}.
    \end{equation}
    First we prove Equations \eqref{eq:bounded_approx2a} and \eqref{eq:bounded_approx2b} by induction.
    
    \textbf{Base Case:} At time $t = 0$, we have by assumption that $|x_0| \le 4\log^2(T)$. Furthermore, Equation \eqref{eq:bounded_approx3} implies that 
    \begin{align*}
        |u_0| &\le \frac{\norm{D}_{\infty} + a^*|x_0| + \gamma_T|x_0|}{b^*-\gamma_T} && \text{Equation \eqref{eq:bounded_approx3}} \\
        &\le \frac{\norm{D}_{\infty} + (a^* + \gamma_T)4\log^2(T)}{b^*-\frac{1}{\log(T)}} && \text{Equation \eqref{eq:bounded_approx1}}\\
        &\le \frac{\log^2(T) + (a^* + \frac{1}{\log(T)})4\log^2(T)}{b^*-\frac{1}{\log(T)}} && \text{Assumption \ref{problem_specifications}}\\
        &\le \frac{\log^2(T) + (a^* + b^*/2)4\log^2(T)}{b^*/2} && \text{Sufficiently large $T$}\\
        &\le \frac{2(1+4a^* + 2b^*)\log^2(T)}{b^*} \numberthis \label{eq:bounded_approx_longv} \\
        &< B_x,
    \end{align*}
    for $T$ sufficiently large such that $2(1 + 4a^*+2b^*)/b^* \le \log(T)$ and $1/\log(T) \le b^*/2$.
    
    \textbf{Induction Hypothesis:} Assume Equations \eqref{eq:bounded_approx2a} and  \eqref{eq:bounded_approx2b} are true for all times less than or equal to $t$. 
    
    \textbf{Induction Step:} Now we will prove that Equations \eqref{eq:bounded_approx2a} and \eqref{eq:bounded_approx2b}  hold at time $t+1$.
    \begin{align*}
        |x_{t+1}| &=  |a^*x_t + b^*u_t + w_t| \\
        &=  |a_tx_t + b_tu_t + w_t + (a^*-a_t)x_t + (b^* -b_t)u_t| \\
        &\le |a_tx_t + b_tu_t| + |w_t| + |(a^*-a_t)x_t| + |(b^* -b_t)u_t| && \text{Triangle Inequality}  \\
        &\stackrel{\text{a.s.}}{\le} \norm{D}_{\infty} + \log^2(T) + \gamma_T |x_t| + \gamma_T |u_t| && \text{Equation \eqref{eq:bounded_approx1}, Equation \eqref{eq:bounded_approx3a}, event $E_1$}\\
        &\le \norm{D}_{\infty} + \frac{1}{\log(T)}\left(|x_t| + |u_t| \right) + \log^2(T) && \text{Equation \eqref{eq:bounded_approx1}}\\
        &\le \norm{D}_{\infty} + \frac{2}{\log(T)}B_x + \log^2(T) && \text{Ind. Hyp.}\\
        &\le \norm{D}_{\infty} + 3\log^2(T)\\
        &\le 4\log^2(T)  && \text{Assumption \ref{problem_specifications}} \\
        &< \log^3(T) \\
        &= B_x.
    \end{align*}

    Above we need $T$ large enough such that $\log(T) > 4$. Since we showed that $|x_{t+1}| \le 4\log^2(T)$,  this also implies by Equations \eqref{eq:bounded_approx3} and \eqref{eq:bounded_approx_longv} that for sufficiently large $T$, $|u_{t+1}| < B_x$. Therefore we have shown Equations \eqref{eq:bounded_approx2a} and \eqref{eq:bounded_approx2b} for time $t+1$, completing the induction proof.

    Now we will prove Equations \eqref{eq:bounded_approx2c} and \eqref{eq:bounded_approx2d} with a similar proof by induction. If the controller $C_t$ is non-random and $x_0$ is not random, this implies that $\E[|x_0|] = |x_0| \le 4\log^2(T)$ and $\E[|u_0|] = |u_0| \le  \frac{2(5+4a^*)\log^2(T)}{b^*}$ by Equation \eqref{eq:bounded_approx_longv}. This proves the base case. For the inductive step, we have that
    \begin{align*}
        &\E[|x_{t+1}|] \\
        &=  \E[|a^*x_t + b^*u_t + w_t|] \\
        &\le \E[|a_tx_t + b_tu_t|] + \E[|w_t|] + \E[|(a^*-a_t)x_t|] + \E[|(b^* -b_t)u_t|] && \text{Triangle Inequality}  \\
        &\le \norm{D}_{\infty} + \log^2(T) + \gamma_T \E[|x_t|] + \gamma_T \E[|u_t|] && \text{Equations \eqref{eq:bounded_approx1}, \eqref{eq:bounded_approx3a}, $w_t$ sub-Gaussian}\\
        &\le \norm{D}_{\infty} + \frac{1}{\log(T)}\left(\E[|x_t|] + \E[|u_t|] \right) + \log^2(T) && \text{Equation \eqref{eq:bounded_approx1}}\\
        &\le \norm{D}_{\infty} + \frac{2}{\log(T)}B_x + \log^2(T) && \text{Ind. Hyp.}\\
        &\le \norm{D}_{\infty} + 3\log^2(T)\\
        &\le 4\log^2(T)  && \text{Assumption \ref{problem_specifications}} \\
        &< \log^3(T) \\
        &= B_x.
    \end{align*}
    We have shown that $\E[|x_{t+1}|] \le 4\log^2(T)$, therefore by Equation \eqref{eq:bounded_approx3} and the same algebraic steps as used in Equation \eqref{eq:bounded_approx_longv}, we have that for sufficiently large $T$,
    \begin{align*}
        \E[|u_{t+1}|] &\le \frac{\norm{D}_{\infty} + a^*\E[|x_{t+1}|] + \gamma_T\E[|x_{t+1}|]}{b^*-\gamma_T} \\
        &\le \frac{\norm{D}_{\infty} + (a^* + \gamma_T)4\log^2(T)}{b^*-\frac{1}{\log(T)}}  \\
        &\le \frac{2(1+4a^*+2b^*)\log^2(T)}{b^*}\\
        &< B_x.
    \end{align*}
    This completes the second proof by induction, proving Equations  \eqref{eq:bounded_approx2c} and \eqref{eq:bounded_approx2d}.
\end{proof}

\subsection{Proof of Lemma \ref{start_invariant_inexpectation} }\label{proof:start_invariant_inexpectation}

\begin{proof}
For this proof, we need the following version of Lemma \ref{offbyepsilon} that applies for expectations rather than with high probability.
\begin{lemma}\label{offbyepsilon_exp}
 Let $x,y$ be two random variables independent of noises $W' = \{w_{i}'\}_{i=0}^{t-1}$ such that for some $L = \tilde{O}_T(1)$, both $\P(|x| \ge L)\E[x^2 \mid |x| \ge L] = o_T\left(\frac{1}{T^{10}}\right)$ and $ \P(|y| \ge L)\E[y^2 \mid |y| \ge L] = o_T\left(\frac{1}{T^{10}}\right)$ and $\P(|x| \le 4\log^2(T)) = 1-o_T(1/T^{11})$ and $\P(|y| \le 4\log^2(T))=1-o_T(1/T^{11})$. Then under Assumptions \ref{assum_init}--\ref{safety_assum} and \ref{parameterization_assum3}, if $\norm{\theta - \theta^*}_{\infty} = \epsilon \le  \epsilon_{\mathrm{A}\ref{parameterization_assum3}}$, then for any $K \in (K_L^{\theta}, K_U^{\theta})$ and $t \le T$,
    \begin{align*}
        &\left|\E\left[t \cdot J(\theta^*,C_{K}^\theta,t,x, W') - t \cdot J(\theta^*,C_{K}^\theta,t,y, W')\right] \right| =  \tilde{O}_T \left(\E[|x-y|]+\epsilon  + \frac{1}{T^2}\right). \numberthis \label{eq:ofbyepsilon2}
    \end{align*}
\end{lemma}
The proof of Lemma \ref{offbyepsilon_exp} can be found in Appendix \ref{proof:offbyepsilon_exp}.
We also need the following generalization of Lemma \ref{bounded_approx}, which bounds the positions for any starting position $x$. 

\begin{lemma}\label{bounded_pos_cont}
    Let $x_0,x_1,...x_T$ be the sequences of positions when starting at position $x_0 = x$ and using controller $C_t$ at time $t$. Suppose that the control $C_t(x_t)$ is safe for dynamics $\theta_t$ and $\norm{\theta_t - \theta^*} \le \frac{1}{\log(T)}$ for all $t < T$. For sufficiently large $T$ under Assumption \ref{problem_specifications},
    \[
        \forall t \le T, \: |x_t| = O_T(|x| + \norm{D}_{\infty} + \max_{i \le t-1} |w_i|).
    \]
    \[
        \forall t < T, \: |C_t(x_t)| = O_T(|x| + \norm{D}_{\infty} + \max_{i \le t-1} |w_i|).
    \]
\end{lemma}
The proof of Lemma \ref{bounded_pos_cont} can be found in Appendix \ref{proof:bounded_pos_cont}.

    Because $C_{K^*}^{\theta^*}, \{C_{K^*_s}^{\theta^*}\}_{s=0}^{s_e}$ are safe for dynamics $\theta^*$, the sequence $x_0^*, x_1^*,...$ starts at $x_0^* = 0$, and $\norm{D}_{\infty} \le \log^2(T)$ by Assumption \ref{problem_specifications}, Lemma \ref{bounded_pos_cont} implies that
    \begin{equation}\label{eq:app_of_bounded_pos}
        |x^*_{   T_s}| = O_T\left(\max_{i \le    T_s-1} |w_i| + \log^2(T)\right).
    \end{equation}
    \begin{lemma}\label{lemma:subgaussian_tail}
        Suppose $w_t$ for $t < T$ are sub-Gaussian and $F$ is an event such that $\P(F) = 1-o_T(1/T^{11})$. Then 
        \[
            \E[\max_{i \le t} w_i^2 \mid \neg F]\P( \neg F) = o_T\left(\frac{1}{T^{10}}\right).
        \]
    \end{lemma}
    The proof of Lemma \ref{lemma:subgaussian_tail} can be found in Appendix \ref{proof:lemma:subgaussian_tail}. 
    Define $F = \{|x^{*}_{   T_s}| < \log^3(T)\}$. Event $E_1$ implies $F$ by Lemma \ref{bounded_approx}, and therefore $\P(F) \ge \P(E_1) = 1-o_T(1/T^{11})$. Therefore, we have by Equation \eqref{eq:app_of_bounded_pos} that
    \begin{align*}
        &\P(\neg F)\E[ |x^{*}_{   T_s}|^2 \mid \neg F] \\
        &= O_T\left(\P(\neg F)\E\left[\max_{i \le    T_s-1} w_i^2 \cond \neg F\right] \right) + \tilde{O}_T\left(\P(\neg F)\right)  && \text{[Eq. \eqref{eq:app_of_bounded_pos} and $(a+b)^2 \le 2a^2 + 2b^2$]}\\
        &= o_T\left(\frac{1}{T^{10}}\right). && \text{Lemma \ref{lemma:subgaussian_tail}, $\P(\neg F) = o_T(1/T^{11})$} \numberthis \label{eq:lemma13app}
    \end{align*}
    Also, note that Lemma \ref{bounded_approx} implies that $\P(x^*_{T_s} \le 4\log^2(T)) \ge \P(E_1) = 1-o_T(1/T^{11}) $.
    We can therefore apply Lemma \ref{offbyepsilon_exp} with $x = x_{   T_s}^*, y= 0, L = \log^3(T), \epsilon=0$. Applying Lemma \ref{offbyepsilon_exp} gives the following desired result.
    \begin{align*}    &\E\left[\left|T_sJ(\theta^*,C^{\theta^*}_{K^*_s}, T_s,  x^*_{   T_s}, W_s) - T_sJ(\theta^*,C^{\theta^*}_{K^*_s}, T_s,  0, W_s)\right| \right] \\
    &= \tilde{O}_T\left( \E\left[|x^*_{   T_s}|\right] + \frac{1}{T^2} \right)  && \text{Lemma \ref{offbyepsilon_exp} } \\
    & =  \tilde{O}_T(1). && \text{Lemma \ref{bounded_approx} for sufficiently large $T$} 
    \end{align*}
    Note that we can apply the expectation form of Lemma \ref{bounded_approx} in the second inequality above because $(C_{K^*}^{\theta^*}, \{C_{K^*_s}^{\theta^*}\}_{s=0}^{s_e})$ are non-random controllers. 
\end{proof}

\subsection{Proof of Lemma \ref{concentration_of_cond_exp} (Concentration of Conditional Expected Cost)}\label{proof:concentration_of_cond_exp}
\begin{proof}
We will use the following form of McDiarmid's Inequality for high probability events.

\begin{lemma}[McDiarmid's Inequality \cite{combes2015extension}]\label{real_mcdiarmids}
Let $f$ be a function such that $f : \mathcal{X}_1 \times \mathcal{X}_2 ... \times \mathcal{X}_n \to \mathbb{R}$ and let $\mathcal{Y} \in \mathcal{X}_1 \times \mathcal{X}_2 ... \times \mathcal{X}_n$ be a subset of the domain such that for some $c$, if $(x_1,...,x_n), (x'_1,...,x'_n) \in \mathcal{Y}$, then
\[
    |f(x_1,...,x_n) - f(x_1',...,x_n')| \le \sum_{i:x_i \ne x_i'} c.
\]
Let $X_1,X_2,...,X_n$ be independent random variables and $X_i \in \mathcal{X}_i$ for all $i$. Define $p = 1 - \P((X_1,...,X_n) \in \mathcal{Y})$ and let $m = \E[f(X_1,...,X_n) \mid (X_1,...,X_n) \in \mathcal{Y}]$. Then for any $\epsilon > 0$,
\[
\P(|f(X_1,...,X_n) - m| \ge \epsilon) \le 2p + 2\exp\left(-\frac{2\max(0, \epsilon - pnc)^2}{nc^2}\right).
\]

\end{lemma}

Define the function $f_{\hat{\theta}_s}(W_m)$ as
    \[
        f_{\hat{\theta}_s}(W_m) =  T_sJ\left(\theta^*,C^{\hat{\theta}_s}_{K_{\mathrm{opt}}(\hat{\theta}_s, T_s)}, T_s, 0, W_m\right).
    \]
 We want to apply McDiarmid's Inequality to $f_{\hat{\theta}_s}$ conditional on $\hat{\theta}_s$ when $E_2^s$ holds, which requires the following bounded difference result.

    \begin{lemma}\label{mcdiarmids_app}
       Under Assumptions \ref{assum_init}--\ref{parameterization_assum3}, given $\hat{\theta}_s$ there exists a fixed $\mathcal{Y}_s \in [-\log^2(T), \log^2(T)]^{T_s}$ such that the event $E^{\mathrm{M}}_s := \{W_s \in \mathcal{Y}_s\}$ satisfies ${\P(E^{\mathrm{M}}_s\mid \hat{\theta}_s) \ge 1-o_T(1/T^8)}$, and conditional on $\hat{\theta}_s$ and $E_2^s$, if $E^{\mathrm{M}}_s$ holds when $W_s = \{w_i\}_{i=   T_s}^{   T_{s+1}-1}$ and when $W'_s =\{w_i'\}_{i=   T_s}^{   T_{s+1} -1}$, then
        \[
            \left|  f_{\hat{\theta}_s}(W_s) - f_{\hat{\theta}_s}(W'_s)\right| \le \sum_{i =    T_s, w_i \ne w'_i}^{   T_{s+1}-1} c
        \]
        for some $c = \tilde{O}_T(1)$.
    \end{lemma}
    The proof of Lemma \ref{mcdiarmids_app} can be found in Appendix \ref{proof:mcdiarmids_app}.
    We will now apply Lemma \ref{real_mcdiarmids} for the function $f_{\hat{\theta}_s}$ conditional on $\hat{\theta}_s$ and $E_2^s$ using Lemma \ref{mcdiarmids_app}. Conditional on $E_2^s$ (where $c$ is from Lemma \ref{mcdiarmids_app}), the following holds for $\epsilon \ge 1/T$ and $T$ sufficiently large.
    \begin{align*}
    &\P\left(\left|f_{\hat{\theta}_s}(W_s) - \E[f_{\hat{\theta}_s}(W_s) \mid E^{\mathrm{M}}_s]\right| \ge \epsilon \cond \hat{\theta}_s \right) \\
    &\le 2\P(\neg E^{\mathrm{M}}_s \mid \hat{\theta}_s) + 2\exp\left(-\frac{2\max\left(0, \epsilon - cT_s \P(\neg E^{\mathrm{M}}_s \mid \hat{\theta}_s )\right)^2}{T_sc^2}\right) \\  
    &=  o_T\left(\frac{1}{T^8}\right) + 2\exp\left(-\frac{\epsilon^2}{2T_sc^2}\right) \quad \quad \quad \quad \quad \quad \text{[$\epsilon \ge 1/T$, $\P( \neg E^{\mathrm{M}}_s| \hat{\theta}_s) = o_T(1/T^8)$, suff. large $T$ ]} \\
    &\le \frac{1}{T^8} + 2\exp\left(-\frac{\epsilon^2}{2T_sc^2}\right).  \quad \quad \quad \quad \quad  \quad \quad \quad \quad  \text{[Suff. large $T$]}
    \end{align*}
\end{proof}

\subsection{Proof of Lemma \ref{uncond_vs_cond_regret} (Unconditional Cost vs Conditional Cost)}\label{proof:uncond_vs_cond_regret}
\begin{proof}

By the Law of Total Expectation,
{\fontsize{10}{10}
  \begin{align*}
            &\E\left[T_sJ(\theta^*,C^{\hat{\theta}_s}_{K_{\mathrm{opt}}(\hat{\theta}_s, T_s)}, T_s, 0, W_s)\cond \hat{\theta}_s \right] \\
            &= \E\left[ T_sJ(\theta^*,C^{\hat{\theta}_s}_{K_{\mathrm{opt}}(\hat{\theta}_s, T_s)}, T_s, 0, W_s) \cond E^{\mathrm{M}}_s, \hat{\theta}_s  \right]\P(E^{\mathrm{M}}_s \mid \hat{\theta}_s)  \\
            &\qquad \qquad + \E\left[ T_sJ(\theta^*,C^{\hat{\theta}_s}_{K_{\mathrm{opt}}(\hat{\theta}_s, T_s)}, T_s, 0, W_s) \cond \neg E^{\mathrm{M}}_s, \hat{\theta}_s \right]\P(\neg E^{\mathrm{M}}_s \mid \hat{\theta}_s) 
            \\
            &\ge \E\left[ T_sJ(\theta^*,C^{\hat{\theta}_s}_{K_{\mathrm{opt}}(\hat{\theta}_s, T_s)}, T_s, 0, W_s) \cond E^{\mathrm{M}}_s, \hat{\theta}_s  \right]\P(E^{\mathrm{M}}_s| \hat{\theta}_s) \quad \quad \quad \quad \quad \quad \quad \quad \quad \quad  \text{Cost is non-negative}\\
            &= \E\left[  T_sJ(\theta^*,C^{\hat{\theta}_s}_{K_{\mathrm{opt}}(\hat{\theta}_s, T_s)}, T_s, 0, W_s) \cond E^{\mathrm{M}}_s, \hat{\theta}_s  \right]  - \E\left[ T_sJ(\theta^*,C^{\hat{\theta}_s}_{K_{\mathrm{opt}}(\hat{\theta}_s, T_s)}, T_s, 0, W_s) \cond E^{\mathrm{M}}_s, \hat{\theta}_s  \right] \P(\neg E^{\mathrm{M}}_s| \hat{\theta}_s)\\
            &= \E\left[  T_sJ(\theta^*,C^{\hat{\theta}_s}_{K_{\mathrm{opt}}(\hat{\theta}_s, T_s)}, T_s, 0, W_s) \cond E^{\mathrm{M}}_s, \hat{\theta}_s  \right]  - o_T\left(\frac{1}{T}\right)\E\left[  T_sJ(\theta^*,C^{\hat{\theta}_s}_{K_{\mathrm{opt}}(\hat{\theta}_s, T_s)}, T_s, 0, W_s) \cond E^{\mathrm{M}}_s, \hat{\theta}_s  \right] \\
            &= \E\left[ T_sJ(\theta^*,C^{\hat{\theta}_s}_{K_{\mathrm{opt}}(\hat{\theta}_s, T_s)}, T_s, 0, W_s) \cond E^{\mathrm{M}}_s, \hat{\theta}_s  \right] -  o_T((q+r)B_x^2) \\
            &= \E\left[ T_sJ(\theta^*,C^{\hat{\theta}_s}_{K_{\mathrm{opt}}(\hat{\theta}_s, T_s)}, T_s, 0, W_s) \cond E^{\mathrm{M}}_s, \hat{\theta}_s  \right] -  \tilde{O}_T(1).
    \end{align*}
    }
    To see the step from the 5th to the 6th line, note that $E^{\mathrm{M}}_s \subseteq \{\forall t \in [T_s : T_{s+1}-1], |w_t| \le \log^2(T)\}$ by assumption and that $E_2^s$ implies that for sufficiently large $T$, $\norm{\theta^* - \hat{\theta}_s} \le \frac{1}{\log(T)}$, therefore by Lemma \ref{bounded_approx} we have that the magnitudes of the positions and controls are all bounded by $B_x$ conditional on events $E_2^s$ and $E^{\mathrm{M}}_s$. Therefore, the cost at each time step conditional on these events is at most $(q+r)B^2_x$, which gives that conditional on event $E_2^s$,
    \begin{align*}
        &\E\left[  T_sJ(\theta^*,C^{\hat{\theta}_s}_{K_{\mathrm{opt}}(\hat{\theta}_s, T_s)}, T_s, 0, W_s) \cond E^{\mathrm{M}}_s, \hat{\theta}_s  \right] \\
        &\le T_s (q+r)B_x^2   \qquad \text{$E_2^s$, $E^{\mathrm{M}}_s \subseteq \{\forall t \in [T_s : T_{s+1}-1], |w_t| \le \log^2(T)\}$,  Lemma \ref{bounded_approx}}\\
        &\le T(q+r)B_x^2.
    \end{align*}
\end{proof}

\subsection{Proof of Lemma \ref{offbyepsilon}}\label{proof:offbyepsilon}
\begin{proof}
    If $|x-y| \le \delta_{\mathrm{A}\ref{parameterization_assum3}}$ then this follows directly from Assumption \ref{parameterization_assum3}. Now for the rest of this proof assume $|x-y| > \delta_{\mathrm{A}\ref{parameterization_assum3}}$ and WLOG assume $x \le y$. Choose $\delta$ to be the largest real number satisfying $\delta \le  \delta_{\mathrm{A}\ref{parameterization_assum3}}$ such that $\frac{|x-y|}{\delta}$ is an integer. Because $\delta_{\mathrm{A}\ref{parameterization_assum3}} < |x-y|$, there must exist an integer in the range $\left[\frac{|x-y|}{\delta_{\mathrm{A}\ref{parameterization_assum3}}}, \frac{2|x-y|}{\delta_{\mathrm{A}\ref{parameterization_assum3}}}\right]$. Therefore, $\delta \ge \delta_{\mathrm{A}\ref{parameterization_assum3}}/2 = \tilde{\Omega}_T(1)$ by definition of $\delta_{\mathrm{A}\ref{parameterization_assum3}}$ . Because  $|x|, |y| < 4\log^2(T)$ and $x \le y$, we know that for all $i \in [0: \frac{|x-y|}{\delta}]$, we have $|x+ i\delta| \le 4\log^2(T)$. For $i \in [0:\frac{|x-y|}{\delta} - 1]$, by Assumption \ref{parameterization_assum3}, under event $E_{\mathrm{A}\ref{parameterization_assum3}}(C_K^\theta, W')$
    \begin{align*}
        &|t \cdot J(\theta^*, C_{K}^\theta, t,  x + i\delta, W') - t \cdot J(\theta^*, C_{K}^\theta, t, x+ (i+1)\delta, W')|  = \tilde{O}_T(\delta+\epsilon).
    \end{align*}
    By the triangle inequality, this implies that conditional on event $E_{\mathrm{A}\ref{parameterization_assum3}}(C_K^\theta, W')$,
    \begin{align*}
        &\left|t \cdot J(\theta^*,C_{K}^\theta,t,x, W') - t \cdot J(\theta^*,C_{K}^\theta,t,y, W') \right| \\
        &\le \sum_{i=0}^{\frac{|x-y|}{\delta}-1} \left|t \cdot J(\theta^*, C_{K}^\theta, t,  x + i\delta, W') - t \cdot J(\theta^*, C_{K}^\theta, t, x+ (i+1)\delta, W')\right| \\
        &= \tilde{O}_T\left(\frac{|x-y|}{\delta} (\delta+\epsilon)\right)  \\
        &= \tilde{O}_T\left(|x-y| + \frac{8\log^2(T)}{\delta}\epsilon\right) && \text{$|x|,|y| < 4\log^2(T)$} \\
        &= \tilde{O}_T\left(|x-y| + \epsilon\right). && \text{$\delta = \tilde{\Omega}_T(1)$}
    \end{align*}
\end{proof}

\subsection{Proof of Lemma \ref{offbyepsiloncontrol_propproof} (Cost of safety controls)}\label{proof:offbyepsiloncontrol_propproof}
\begin{proof}
   
    The first tool for this proof is the following lemma, which informally states that being off by a small amount of control has a small impact on the overall cost.
    \begin{lemma}\label{offbyepsiloncontrol}
            Under Assumptions \ref{assum_init}--\ref{parameterization_assum3}, with conditional probability $1-o_T(1/T)$ given event $E$, for all $s \in [0:s_e]$,
            
        \begin{align*}
            &|T_s \cdot J(\theta^*, C_{K_{\mathrm{opt}}(\hat{\theta}_s, T_s)}^{\hat{\theta}_s}, T_s, x'_{   T_s}, W_s)  - T_s \cdot J(\theta^*, C^{\mathrm{alg}}_s, T_s, x'_{   T_s}, W_s)| \\
            &= \tilde{O}_T\left( \sum_{t=   T_s}^{   T_{s+1}-1} |C_{K_{\mathrm{opt}}(\hat{\theta}_s, T_s)}^{\hat{\theta}_s}(x'_t) - C^{\mathrm{alg}}_s(x'_{   t})|\right) + \tilde{O}_T(T_s\epsilon_s).
        \end{align*}
    \end{lemma}
    The proof of Lemma \ref{offbyepsiloncontrol} can be found in Appendix \ref{proof:offbyepsiloncontrol}.
    
    The control $C_{K_{\mathrm{opt}}(\hat{\theta}_s, T_s)}^{\hat{\theta}_s}(x'_t)$ is safe for dynamics $\hat{\theta}_s$ and conditional on event $E$, $\norm{\hat{\theta}_s - \theta^*}_{\infty} \le \tilde{O}_T(\nu_T) \le 1/\log(T)$ for sufficiently large $T$. The controller $C_s^{\mathrm{alg}}$ is safe for dynamics $\theta^*$ for all $T$ steps conditional on event $E$ by definition of $E$. These together imply by Lemma \ref{bounded_approx} that, conditional on event $E$ and for sufficiently large $T$, for all $t \in [T_s, T_{s+1}-1]$,

    \begin{equation}\label{eq:xs_lem8}
        |x'_{t}|, |\hat{x}_{t}| \le 4\log^2(T) \le B_x.
    \end{equation}
    By Lemma \ref{offbyepsilon} and \ref{offbyepsiloncontrol}, we have that conditional on event $E$, with probability $1-o_T(1/T)$,
    {\small
    \begin{align*}
    &\sum_{s=0}^{s_e}  T_sJ(\theta^*,C^{\mathrm{alg}}_s, T_s,  x'_{   T_s},W_s) - \sum_{s=0}^{s_e} T_sJ(\theta^*,C^{\hat{\theta}_s}_{K_{\mathrm{opt}}(\hat{\theta}_s, T_s)}, T_s,  \hat{x}_{   T_s},W_s) \\
    &=  \tilde{O}_T(1) + \sum_{s=0}^{s_e}  T_sJ(\theta^*,C^{\mathrm{alg}}_s, T_s,  x'_{   T_s},W_s) - \sum_{s=0}^{s_e} T_sJ(\theta^*,C^{\hat{\theta}_s}_{K_{\mathrm{opt}}(\hat{\theta}_s, T_s)}, T_s,  x'_{   T_s},W_s) \quad \quad  \text{Eq. \eqref{eq:xs_lem8}, Lemma \ref{offbyepsilon}}\\
    &=  \tilde{O}_T( 1)+ \sum_{s=0}^{s_e} \left(  T_sJ(\theta^*,C^{\mathrm{alg}}_s, T_s,  x'_{   T_s},W_s) - T_sJ(\theta^*,C^{\hat{\theta}_s}_{K_{\mathrm{opt}}(\hat{\theta}_s, T_s)}, T_s,  x'_{   T_s},W_s)\right) \\
    &= \tilde{O}_T(1) +\tilde{O}_T\left(\sum_{s=0}^{s_e} \left(T_s\epsilon_s + \sum_{t =    T_s}^{   T_{s+1}-1} |C^{\mathrm{alg}}_s(x'_t) - C_{K_{\mathrm{opt}}(\hat{\theta}_s, T_s)}^{\hat{\theta}_s}(x'_t)|\right)\right) \quad \quad \quad \quad\quad \quad \text{Lemma \ref{offbyepsiloncontrol} } \\
    &=  \tilde{O}_T\left(\sum_{s=0}^{s_e} \epsilon_sT_s \right) + \tilde{O}_T\left(\sum_{s=0}^{s_e} \sum_{t =    T_s}^{   T_{s+1}-1} X^U_{t} \cdot \left|u_t^{\mathrm{safeU}} - C_{K_{\mathrm{opt}}(\hat{\theta}_s, T_s)}^{\hat{\theta}_s}(x'_t)\right| + X^L_{t} \cdot \left|u_t^{\mathrm{safeL}} - C_{K_{\mathrm{opt}}(\hat{\theta}_s, T_s)}^{\hat{\theta}_s}(x'_t)\right|\right).
    \end{align*}
    }
    We applied Lemma \ref{offbyepsilon} for every $s\in [0:s_e]$, so $\tilde{O}_T(1)$ times. Since Lemma \ref{offbyepsilon} holds with probability $1-o_T(1/T^{10})$, a union bound gives the first inequality holds with probability $1-o_t(1/T^9)$. Another union bound combining this with the single application of Lemma \ref{offbyepsiloncontrol} gives that the probability of the above result is $1-o_T(1/T)$. The final line simplified using the fact that the two controls are equal if $X_t^L = X_t^U = 0$.
\end{proof}

\subsection{Proof of Lemma \ref{bound_on_cont_diff_propproof} (Difference in Safety Controls)}\label{proof:bound_on_cont_diff_propproof}
\begin{proof}
    By symmetry, it is sufficient to show the first part of the lemma statement for $u_t^{\mathrm{safeU}}$. 
    
    Because $C^{\mathrm{alg}}$ is safe for dynamics $\theta^*$ under event $E$ and $E \subseteq E_1$, we have by Lemma \ref{bounded_approx} that under event $E$,
    \begin{equation}\label{safeL_boundedapprox0}
        |x'_t| \le 4\log^2(T).
    \end{equation}
    Under event $E$ and for sufficiently large $T$, $\norm{\theta^* - \hat{\theta}_s}_{\infty} \le \epsilon_s \le \frac{1}{\log(T)}$. This implies by construction of $u_t^{\mathrm{safeU}}$ that under event $E$ and for sufficiently large $T$, $a^*x'_t + b^*u_t^{\mathrm{safeU}} \le D_{\mathrm{U}}$. By Lemma \ref{lemma:L_less_than_U}, we also have that under event $E$ and for sufficiently large $T$, $u_t^{\mathrm{safeU}} \ge u_t^{\mathrm{safeL}}$. Therefore, by construction of $u_t^{\mathrm{safeL}}$ we have that under event $E$ and for sufficiently large $T$, $a^*x'_t + b^*u_t^{\mathrm{safeU}} \ge a^*x'_t + b^*u_t^{\mathrm{safeL}} \ge  D_{\mathrm{L}}$. Together, this shows that $u_t^{\mathrm{safeU}}$ is safe for dynamics $\theta^*$. By Lemma \ref{bounded_approx} and Equation \eqref{safeL_boundedapprox0}, this gives that under event $E$ and for sufficiently large $T$,
    \begin{equation}\label{safeL_boundedapprox}
        |u_t^{\mathrm{safeU}}| \le B_x.
    \end{equation}
    Because any control used by controller $C^{\hat{\theta}_s}_{K_{\mathrm{opt}}(\hat{\theta}_s, T_s)}$ is safe for dynamics $\hat{\theta}_s$, by Lemma \ref{bounded_approx} we also have that under event $E$ for sufficiently large $T$,
    \begin{equation}\label{safeL_boundedapprox2}
        |C^{\hat{\theta}_s}_{K_{\mathrm{opt}}(\hat{\theta}_s, T_s)}(x'_t)| \le B_x.
    \end{equation}
    Also, note that by Algorithm \ref{alg:cap} Line \ref{line:safeU}, $u_t^{\mathrm{safeU}}$ satisfies, for some $\theta$ such that $\norm{\theta -  \hat{\theta}_s}_{\infty} \le \epsilon_s$,
    \begin{equation}\label{eq:ab1}
        ax'_t + bu_t^{\mathrm{safeU}} = D_{\mathrm{U}}.
    \end{equation}
    Under event $E$, $\norm{\theta^* - \hat{\theta}_s}_{\infty} \le \epsilon_s$, which implies that $\norm{\theta^* - \theta}_{\infty} \le 2\epsilon_s \le \tilde{O}_T(\nu_T) \le 1/\log(T)$ for sufficiently large $T$. Therefore, applying Lemma \ref{bounded_approx} gives that under event $E$ and for sufficiently large $T$, 
    \begin{align*}
     D_{\mathrm{U}} &\ge a^*x'_t + b^*u_t^{\mathrm{safeU}} && \text{$u_t^{\mathrm{safeU}}$ safe for $\theta^*$} \\
     &\ge ax'_t + bu_t^{\mathrm{safeU}} - |u_t^{\mathrm{safeU}}|2\epsilon_s - |x'_t|2\epsilon_s  && \text{$\norm{\theta^* - \theta}_{\infty} \le 2\epsilon_s$}\\
     &\ge D_{\mathrm{U}} - 4B_x\epsilon_s. && \text{ Equations \eqref{safeL_boundedapprox0},\eqref{safeL_boundedapprox}, and \eqref{eq:ab1} } \numberthis \label{eq:enforcingsafety3}
    \end{align*}    
    If $u_t^{\mathrm{safeU}} \le C^{\hat{\theta}_s}_{K_{\mathrm{opt}}(\hat{\theta}_s, T_s)}(x'_t)$, then there must exist some $\theta$ such that $\norm{\hat{\theta}_s - \theta}_{\infty} \le \epsilon_s$ and
    \begin{equation}\label{eq:ab2}
        ax'_t + bC^{\hat{\theta}_s}_{K_{\mathrm{opt}}(\hat{\theta}_s, T_s)}(x'_t) \ge D_{\mathrm{U}}.
    \end{equation}
    Under event $E$, $\norm{\theta^* - \theta}_{\infty} \le 2\epsilon_s \le \tilde{O}_T(\nu_T) \le 1/\log(T)$ for sufficiently large $T$, therefore under event $E$ and for sufficiently large $T$,
    \begin{align*}
        &a^*x'_t + b^* C^{\hat{\theta}_s}_{K_{\mathrm{opt}}(\hat{\theta}_s, T_s)}(x'_t) \\
        &\ge ax'_t + b C^{\hat{\theta}_s}_{K_{\mathrm{opt}}(\hat{\theta}_s, T_s)}(x'_t) -2\epsilon_s|x'_t| - 2\epsilon_s\left|C^{\hat{\theta}_s}_{K_{\mathrm{opt}}(\hat{\theta}_s, T_s)}(x'_t)\right| \\
        &\ge  D_{\mathrm{U}} - 4B_x\epsilon_s.  && \text{Equations \eqref{safeL_boundedapprox0},\eqref{safeL_boundedapprox2}, and \eqref{eq:ab2}}\numberthis \label{eq:enforcingsafety5}
    \end{align*}
    Finally, because $C^{\hat{\theta}_s}_{K_{\mathrm{opt}}(\hat{\theta}_s, T_s)}(x'_t)$ is safe for dynamics $\hat{\theta}_s$, 
    \begin{equation}\label{eq:enforcingsafety6}
        \hat{a}_sx'_t +\hat{b}_s C^{\hat{\theta}_s}_{K_{\mathrm{opt}}(\hat{\theta}_s, T_s)}(x'_t) \le D_{\mathrm{U}}.
    \end{equation}
    Using that under event $E$, $\norm{\theta^* - \hat{\theta}_s}_{\infty} \le \epsilon_s \le \tilde{O}_T(\nu_T) \le 1/\log(T)$ for sufficiently large $T$, Equations \ref{safeL_boundedapprox0}, \ref{safeL_boundedapprox2}, and \eqref{eq:enforcingsafety6} imply that under event $E$ and for sufficiently large $T$,
    \begin{equation}\label{eq:enforcingsafety7}
        a^*x'_t + b^*C^{\hat{\theta}_s}_{K_{\mathrm{opt}}(\hat{\theta}_s, T_s)}(x'_t) \le D_{\mathrm{U}} + 2B_x\epsilon_s.
    \end{equation}
    Combining Equations \eqref{eq:enforcingsafety5} and \eqref{eq:enforcingsafety7}, if $u_t^{\mathrm{safeU}} \le C^{\hat{\theta}_s}_{K_{\mathrm{opt}}(\hat{\theta}_s, T_s)}(x'_t)$ then under event $E$ and for sufficiently large $T$,
    \[
        D_{\mathrm{U}} - 4B_x\epsilon_s \le a^*x'_t + b^*C^{\hat{\theta}_s}_{K_{\mathrm{opt}}(\hat{\theta}_s, T_s)}(x'_t) \le D_{\mathrm{U}} + 2B_x\epsilon_s.
    \]
    Combining this with  Equation \eqref{eq:enforcingsafety3} gives that under event $E$ and for sufficiently large $T$,
    \[
        |(a^*x'_t + b^*u_t^{\mathrm{safeU}}) - (a^*x'_t + b^*C^{\hat{\theta}_s}_{K_{\mathrm{opt}}(\hat{\theta}_s, T_s)}(x'_t))| = 6B_x\epsilon_s.
    \]
    This implies the desired result that under event $E$ and for sufficiently large $T$,
    \[
        |u_t^{\mathrm{safeU}} -  C^{\hat{\theta}_s}_{K_{\mathrm{opt}}(\hat{\theta}_s, T_s)}(x'_t)| = 6B_x\epsilon_s/b^*.
    \]
\end{proof}
    \subsection{Proof of Lemma \ref{mcdiarmids_app} (McDiarmid's Condition)}\label{proof:mcdiarmids_app}
    \begin{proof}

    First, we will construct the event $E_s^{\mathrm{M}}$. Define
    \[
        E_s^{\mathrm{M}} = \{\forall t \in [T_s: T_{s+1}-1], |w_t| \le \log^2(T)\} \cap \bigcap_{i=   T_s}^{   T_{s+1}-1} E_{\mathrm{A}\ref{parameterization_assum3}}\left(C^{\hat{\theta}_s}_{K_{\mathrm{opt}}(\hat{\theta}_s, T_s)}, \{w_t\}_{t=i}^{   T_{s+1}-1}\right).
    \]
    Note because $\P(\{\forall t \in [T_s : T_{s+1}-1], |w_t| \le \log^2(T)\}) \ge \P(E_1) = 1-o_T(1/T^{10})$ and because under event $E_2^s$,  $\P\left(E_{\mathrm{A}\ref{parameterization_assum3}}\left(C^{\hat{\theta}_s}_{K_{\mathrm{opt}}(\hat{\theta}_s, T_s)}, \{w_t\}_{t=i}^{   T_{s+1}-1}\right) \cond \hat{\theta}_s \right) = 1-o_T(1/T^{10})$ we have by a union bound that $\P(E_s^{\mathrm{M}}  \mid \hat{\theta}_s  ) = 1-o_T(1/T^{9})$. Suppose $E_s^{\mathrm{M}}$ holds for $W_s$ and $W'_s$. For $i \in [   T_s,    T_{s+1}]$, define $W^i$ as follows.
   \[
    W^i = \{w_{   T_s},w_{   T_s+1},...,w_{i-1}, w_i',w_{i+1}', w_{i+2}',...w_{   T_{s+1}-1}'\}.
   \]
   In other words, $W^i$ includes noise $w_t$ for $t < i$ and includes $w'_t$ for $t \ge i$. For $i \in [   T_s,    T_{s+1} - 1]$, we will first bound 
    \[
        \left|  f_{\hat{\theta}_s}(W^i) - f_{\hat{\theta}_s}(W^{i+1})\right|.
    \]
    First, note that if $w_i = w_i'$, then $W^i = W^{i+1}$ and therefore $f_{\hat{\theta}_s}(W^i) = f_{\hat{\theta}_s}(W^{i+1})$. Now, assume $w_i \ne w_i'$. Let $x^i_{0},..,x^i_{T_s}$ be the series of positions when the noise random variables are $W^i$, $x^i_0 = 0$, and the controller used is $C_{K_{\mathrm{opt}}(\hat{\theta}_s, T_s)}^{\hat{\theta}_s}$.  Conditional on $E_2^s$, $\norm{\hat{\theta}_s - \theta^*}_{\infty} \le \tilde{O}(\nu_T) \le 1/\log(T)$ for sufficiently large $T$. Because $E^{\mathrm{M}}_s$ holds for $W_s, W_s'$, we have that $E_1$ holds for $W^i$ for all $i$. Therefore by Lemma \ref{bounded_approx} for sufficiently large $T$, $|x^i_{t}| \le 4\log^2(T)$ for all $i,t$. For any $t \le i$, $x^i_t = x^{i+1}_t$. Therefore, the difference in the two trajectories $\{x_t^i\}$ and $\{x_t^{i+1}\}$ only occurs at and after time $i+1$. The first difference occurs at time $i+1$ when $x^i_{i+1} = x^{i+1}_{i+1} - w_{i} + w'_{i} $.   For the next $T_{s+1}-i-1$ steps,  the difference in cost of the two trajectories $\{x^i_t\}$ and $\{x^{i+1}_t\}$ is
    {\fontsize{10}{10}
    \begin{align*}
        &(T_{s+1}-i-1)J(\theta^*, C^{\hat{\theta}_s}_{K_{\mathrm{opt}}(\hat{\theta}_s, T_s)},T_{s+1}-i-1, x^{i+1}_{i+1}, \{w_{t}'\}_{t=i+1}^{   T_{s+1}-1}) \\
        &\quad \quad -  (T_{s+1}-i-1)J(\theta^*, C^{\hat{\theta}_s}_{K_{\mathrm{opt}}(\hat{\theta}_s, T_s)}, T_{s+1}-i-1, x^{i}_{i+1}, \{w_{t}'\}_{t=i+1}^{   T_{s+1}-1} ) \\
        &= \tilde{O}_T\left(|x_{i+1}^{i+1} - x^i_{i+1}| + |\hat{\theta}_s - \theta^*|\right) && \text{Lemma \ref{offbyepsilon}, $|x_t^i| \le 4\log^2(T)$ } \\
        &= \tilde{O}_T\left(|x^{i+1}_{i+1} - x^{i+1}_{i+1} + w_{i} - w'_{i}| + \nu_T\right)&& \text{Event $E_2^s$, $x^i_{i+1} = x^{i+1}_{i+1} - w_{i} + w'_{i} $}  \\
        &= \tilde{O}_T\left(|w_{i} - w_i'| + \nu_T\right)  \\
        &= \tilde{O}_T\left(2\log^2(T) + \nu_T\right)  && \text{$W,W'$ satisfy event $E_s^{\mathrm{M}}$} \\
        &= \tilde{O}_T(1). \numberthis \label{eq:mcdi1}
    \end{align*}
    }
    We have therefore shown that for some $c = \tilde{O}_T(1)$,
        \begin{align*}
        |f_{\hat{\theta}_s}(W^i) - f_{\hat{\theta}_s}(W^{i+1})| \le c.
    \end{align*}
    Because $W_s = W^{   T_{s+1}}$ and $W_s' = W^{   T_{s}}$, we have by the triangle inequality that
    \begin{align*}
        |f_{\hat{\theta}_s}(W_s) - f_{\hat{\theta}_s}(W_s')| &= |f_{\hat{\theta}_s}(W^{   T_{s+1}}) - f_{\hat{\theta}_s}(W^{   T_s})| \\
        &\le \sum_{i=   T_s}^{   T_{s+1}-1} |f_{\hat{\theta}_s}(W^i) - f_{\hat{\theta}_s}(W^{i+1})| \\
        &= \sum_{i=   T_s, w_i \ne w_i'}^{   T_{s+1}-1} |f_{\hat{\theta}_s}(W^i) - f_{\hat{\theta}_s}(W^{i+1})|  \\
        &\le \sum_{i=   T_s, w_i \ne w_i'}^{   T_{s+1}-1} c.
    \end{align*}
\end{proof}

\subsection{Proof of Lemma \ref{offbyepsilon_exp}}\label{proof:offbyepsilon_exp}
\begin{proof}

Define $E^* = \{|x|, |y| \le 4\log^2(T)\} \cap E_{\mathrm{A}\ref{parameterization_assum3}}(C_K^\theta, W')$. By assumption of the lemma, we have that $\P(|x| \le 4\log^2(T)) = 1 - o_T(1/T^{11})$ and $\P(|y| \le 4\log^2(T)) = 1 - o_T(1/T^{11})$. Because $\norm{\theta - \theta^*}_{\infty} \le \epsilon_{\mathrm{A}\ref{parameterization_assum3}}$, $\P(E_{\mathrm{A}\ref{parameterization_assum3}}(C_K^\theta, W')) = 1 - o_T(1/T^{10})$. Therefore, by a union bound we have that $\P(E^*) = 1-o_T(1/T^{10})$. By the Law of Total Expectation,
\begin{align*}
    &\E[\left|t \cdot J(\theta^*, C_K^\theta, t, x, W') - t \cdot J(\theta^*, C_K^\theta, t, y, W')\right|]\\
    &= \E\left[\left|t \cdot J(\theta^*, C_K^\theta, t, x, W') - t \cdot J(\theta^*, C_K^\theta, t, y, W') \right| \cond E^* \right]\P(E^*) \\
    &\quad \quad \quad \quad \quad + \E\left[\left|t \cdot J(\theta^*, C_K^\theta, t, x, W') - t \cdot J(\theta^*, C_K^\theta, t, y, W')\right| \cond \neg E^*\right]\P(\neg E^*) \\
    &= \E\left[\tilde{O}_T\left(|x-y| + \epsilon \right) \cond E^* \right]\P(E^*) \\
    &\quad \quad \quad \quad \quad  + \E\left[\left|t \cdot J(\theta^*, C_K^\theta, t, x, W') - t \cdot J(\theta^*, C_K^\theta, t, y, W')\right| \cond \neg E^*\right]\P(\neg E^*) && \text{Lemma \ref{offbyepsilon}} \\    
    &= \tilde{O}_T\Big(\E\left[|x-y| \cond E^*\right]\P\left(E^*\right)+ \epsilon \Big)  \\
    &\quad \quad \quad \quad \quad  + \E\left[\left|t \cdot J(\theta^*, C_K^\theta, t, x, W') - t \cdot J(\theta^*, C_K^\theta, t, y, W')\right| \cond \neg E^*\right]\P(\neg E^*)  \\
    &= \tilde{O}_T\Big(\E\left[|x-y|\right]+ \epsilon \Big)  
 + \E\left[\left|t \cdot J(\theta^*, C_K^\theta, t, x, W') - t \cdot J(\theta^*, C_K^\theta, t, y, W')\right| \cond \neg E^*\right]\P(\neg E^*)  &&  \text{LoTE}
\end{align*}
Therefore, all we must show is that
\[
\E\left[\left|t \cdot J(\theta^*, C_K^\theta, t,  x, W') - t \cdot J(\theta^*, C_K^\theta, t,  y, W')\right| \cond \neg E^*\right]\P(\neg E^*) = \tilde{O}_T(T^{-2}).
\]
Define $w_m = \max_{w \in W'} |w|$. By Lemma \ref{bounded_pos_cont}, we can bound the position and controls at every time step in terms of $w_m$ to get that
\begin{align*}
&\left|t \cdot J(\theta^*, C_K^\theta, t, x,W') - t \cdot J(\theta^*, C_K^\theta, t,  y, W')\right|  \\
&= T(q+r)O_T\left((w_m + x + \norm{D}_{\infty})^2 + (w_m + y + \norm{D}_{\infty})^2\right)  && \text{Triangle Inequality, Lemma \ref{bounded_pos_cont}} \\
&= O_T\left(T\left((w_m + x + \norm{D}_{\infty})^2+ (w_m + y + \norm{D}_{\infty})^2\right)\right) \\
&= \tilde{O}_T\left(T\left(w_m^2 + w_m|x| + |x|^2 + w_m|y| + |y|^2 + w_m + |x| + |y| + 1\right)\right). && \text{Assum \ref{problem_specifications} ($\norm{D}_{\infty} \le \log^2(T)$)}
\end{align*}
Therefore, we have that
{\fontsize{10}{10}
\begin{align*}
    &\E\left[\left|t \cdot J(\theta^*, C_K^\theta, t,  x, W') - t \cdot J(\theta^*, C_K^\theta, t,  y, W')\right| \mid \neg E^*\right]\P(\neg E^*)\\
    &=  \tilde{O}_T\Big(T\big(\E[w_m^2 \mid \neg E^*]\P(\neg E^*) + \E[w_m \mid \neg E^*]\P(\neg E^*)  + \E[|y|w_m \mid \neg E^*]\P(\neg E^*) +\E[ |y|^2 \mid \neg E^*]\P(\neg E^*) \\
    & \quad  + \E[|x|w_m \mid \neg E^*]\P(\neg E^*) +\E[ |x|^2 \mid \neg E^*]\P(\neg E^*) +  \E[|x| \mid \neg E^*]\P(\neg E^*) 
 +  \E[|y| \mid \neg E^*]\P(\neg E^*) + \P(\neg E^*) \big)\Big).\numberthis \label{eq:expectation}
\end{align*}
}
Therefore, it is sufficient to show that $\E[w_m \mid \neg E^*] \P(\neg E^*)$, $\E[w_m^2 \mid \neg E^*] \P(\neg E^*)$, $\E[|x| \mid \neg E^*] \P(\neg E^*)$, $\E[x^2 \mid \neg E^*]\P(\neg E^*)$ , $\E[|y| \mid \neg E^*] \P(\neg E^*)$, $\E[y^2 \mid \neg E^*] \P(\neg E^*)$,  $\E[|x|w_m \mid \neg E^*] \P(\neg E^*)$, $\E[|y|w_m \mid \neg E^*]\P(\neg E^*)$ are all $\tilde{O}_T(\frac{1}{T^3})$. We will use the following probability result.

\begin{lemma}\label{lemma:condsplit}
    Suppose $X$ is a non-negative random variable. Then for any $L \ge 0$ and any event $E$, we have that
    \[
        \E\left[X \cond E\right]\P(E) \le \P(E)L + \P(X \ge L)\E\left[X \cond X \ge L\right]
    \]
\end{lemma}
\begin{proof}

    For any events $A,B$ such that $A \subseteq B$, we have that
    \begin{align*}
        \E[X \mid B]\P(B) &= \E[X \mid A,B]\P(A \mid B)\P(B) + \E[X \mid \neg A, B]\P(\neg A \mid B)\P(B) \\
        &= \E[X \mid A]\P(A) + \E[X \mid \neg A, B]\P(\neg A \mid B)\P(B)  && \text{$A \subseteq B$} \\
        &\ge \E[X \mid A]\P(A). \numberthis \label{eq:A_to_B}
    \end{align*}
    Therefore, we can conclude that 
    \begin{align*}
        &\E\left[X \cond E\right]\P(E) \\
        &= \E\left[X \cond E, X \le L\right] \P( X \le L \mid E)\P(E) + \E\left[X \cond E, X \ge L\right] \P(X \ge L \mid E)\P(E) \\
        &\le \P(E) L +  \E\left[X \cond E, X \ge L\right] \P(X \ge L \mid E)\P(E)  \\
        &\le \P(E) L +  \E\left[X \cond E, X \ge L\right] \P(E, X \ge L)  \\
        &\le \P(E) L +  \E\left[X \cond X \ge L\right] \P(X \ge L). && \text{Eq \eqref{eq:A_to_B}}
    \end{align*}
\end{proof}

Now, note that by the assumption on $x$ and definition of $E^*$ (where $L$ is from the lemma statement),
\begin{align*}
    \E[x^2 \mid \neg E^*]\P(\neg E^*) &\le \P(\neg E^*)L^2 + \P(|x| \ge L)\E[|x|^2 \mid |x| \ge L] && \text{Lemma \ref{lemma:condsplit}} \\
    &= \tilde{O}_T\left( \frac{1}{T^{10}}\right) + \tilde{O}_T\left(\frac{1}{T^{10}} \right)\\
    &= \tilde{O}_T\left( \frac{1}{T^{10}}\right) .
\end{align*}
This also implies by the Cauchy--Schwarz inequality that 
\begin{align*}
    \E[|x| \mid \neg E^*]\P(\neg E^*) &\le \sqrt{\E[x^2 \mid \neg E^*]}\P(\neg E^*) \\
    &= \sqrt{ \E[x^2 \mid \neg E^*]\P(\neg E^*)}\sqrt{\P(\neg E^*)} \\
    &=  \tilde{O}_T\left( \frac{1}{T^{5}}\right).
\end{align*}
 By Lemma \ref{lemma:subgaussian_tail}, because $\P(E^*) = 1-o_T(1/T^{11})$ we have that
\[
\E[w_m^2 \mid \neg E^*]\P(\neg E^*)  = \tilde{O}_T\left(\frac{1}{T^{10}}\right).
\]
Once again, by the Cauchy-Schwarz inequality this implies that $\E[w_m \mid \neg E^*] = \tilde{O}_T\left( \frac{1}{T^{5}}\right)$. 

By the subgaussian assumption on $\mathcal{D}$ and a union bound, we have that
\begin{align*}
    \P(w_m \ge \log^3(T)) &\le \sum_{w \in W'} \P(|w| \ge \log^3(T)) \\
    &\le t \cdot 2e^{-\Omega_T(\log^6(T))} \\
    &\le o_T(1/T^{11}). \numberthis \label{eq:wm_subgaussian}
\end{align*}
Finally, we have by the independence of $x$ and $w_m$ and the assumption on $x$ that
{\fontsize{10}{10}
\begin{align*}
    &\E[|x|w_m \mid \neg E^*]\P(\neg E^*) \\
    &\le \P(\neg E^*)L\log^3(T) \\
    & \quad \quad \quad + \P(|x| \ge L, w_m \ge \log^3(T))\E\left[|x|w_m  \cond |x| \ge L, w_m \ge \log^3(T)\right]  \\
    & \quad \quad \quad  +\P(|x| \le L, w_m \ge \log^3(T))\E[|x|w_m \mid |x| \le L, w_m \ge \log^3(T)] \\
    & \quad \quad \quad + \P(|x| \ge L, w_m \le \log^3(T))\E[|x|w_m \mid |x| \ge L, w_m \le \log^3(T)] \quad \quad \quad \quad \text{Lemma \ref{lemma:condsplit}}\\
    &\le \P(\neg E^*)L\log^3(T) \\
    & \quad \quad \quad + \P(|x| \ge L)\P(w_m \ge \log^3(T))\E\left[|x| \cond |x| \ge L\right] \E\left[w_m  \cond w_m \ge \log(T)\right] \\
    & \quad \quad \quad  +L\P(w_m \ge \log^3(T))\E[w_m \mid w_m \ge \log^3(T)] \\
    & \quad \quad \quad + \log^3(T)\P(|x| \ge L)\E\left[|x| \cond |x| \ge L\right] \quad \quad \quad \quad \quad \quad \quad \quad \text{[Ind of $x$ and $w_m$]} 
    \\ 
    &= \tilde{O}_T\left( \frac{1}{T^{10}}\right) + \tilde{O}_T\left( \frac{1}{T^{20}}\right) + \tilde{O}_T\left( \frac{1}{T^{10}}\right) + \tilde{O}_T\left(\frac{1}{T^{10}} \right) \quad \quad \quad \quad \text{[Def of $E^*$, Lemma \ref{lemma:subgaussian_tail}, Eq \eqref{eq:wm_subgaussian}, Assum on $x$]}\\
    &= \tilde{O}_T\left( \frac{1}{T^{10}}\right) .
\end{align*}
}
Note that by symmetry, all of the above results also hold for $y$. Returning to Equation \eqref{eq:expectation}, we have that 
\[
\E\left[\left|t \cdot J(\theta^*, C_K^\theta, t,  x, W') - t \cdot J(\theta^*, C_K^\theta, t,  y, W')\right| \mid \neg E^*\right]\P(\neg E^*) = \tilde{O}_T\left(\frac{1}{T^{2}}\right).
\]
This completes the proof that
\[
\E\left[\left|t \cdot J(\theta^*, C_K^\theta, t,  x, W') - t \cdot J(\theta^*, C_K^\theta, t,  y, W')\right| \right] = \tilde{O}_T\left(\E[|x-y|] + \epsilon + \frac{1}{T^{2}}\right).
\]
\end{proof}
\subsection{Proof of Lemma \ref{bounded_pos_cont}}\label{proof:bounded_pos_cont}
\begin{proof}
    Define $\gamma_T = \max_{t \le T-1} \norm{\theta_t - \theta^*}_{\infty} \le \frac{1}{\log(T)}$. Because the control at time $t$ is safe for dynamics $\theta_t$, we have $D_{\mathrm{L}} \le a_tx_t + b_tu_t \le D_{\mathrm{U}}$ for all $t$. By the triangle inequality,
    \[
        |x_{t+1}|  = |a^*x_t + b^*u_t + w_t| \le |w_t| + \norm{D}_{\infty} + \gamma_T(|x_t| + |u_t|).
    \]
    As in Equation \eqref{eq:bounded_approx3}, 
    \[
        |u_t| \le \frac{\norm{D}_{\infty}  + a^*|x_t| + \gamma_T|x_t|}{b^* - \gamma_T} = \frac{\norm{D}_{\infty}  + (a^* + \gamma_T)|x_t|}{b^* - \gamma_T}.
    \]
    For sufficiently large $T$, $\gamma_T \le b^*/2$, and therefore for sufficiently large $T$,
    \[
        |x_{t+1}| \le |w_t| + \norm{D}_{\infty}  + \gamma_T \left(|x_t| + \frac{\norm{D}_{\infty}  + a^*|x_t| + \gamma_T|x_t|}{b^* - \gamma_T}\right) = O_T(|w_t| + \norm{D}_{\infty} + \gamma_T|x_t|).
    \]
    Using $x_0 = x$ as the base-case, this recursive relationship implies that for all $t$,
    \begin{align*}
        |x_t| &\le O_T\left(\sum_{i=0}^{t-1} (|w_i| + \norm{D}_{\infty} + |x|)\gamma_T^{t-1-i} \right) \\
        &\le O_T\left(\left(\max_{i \le t-1} |w_i| + \norm{D}_{\infty} + |x|\right)\sum_{i=0}^{t-1} \gamma_T^i\right) \\
        &\le O_T\left(\left(\max_{i \le t-1} |w_i| + \norm{D}_{\infty} + |x|\right)\sum_{i=0}^{t-1} \left(\frac{1}{\log(T)}\right)^i\right) \\
        &=  O_T\left(\left(\max_{i \le t-1} |w_i| + \norm{D}_{\infty} + |x|\right)\frac{1}{1-\frac{1}{\log(T)}}\right).
    \end{align*}
   This implies that for sufficiently large $T$, $|x_{t}| = O_T(\max_{i \le t-1} |w_i| + \norm{D}_{\infty} + |x|)$
    and
    \[
        |u_{t}| \le \frac{\norm{D}_{\infty} + (a^*+\gamma_T)O_T(\max_{i \le t-1} |w_i| + \norm{D}_{\infty} + |x|)}{b^* - \gamma_T} = O_T(\max_{i \le t-1} |w_i| + \norm{D}_{\infty} + |x|),
    \]
    which are exactly the desired bounds.
    \end{proof}

\subsection{Proof of Lemma \ref{lemma:subgaussian_tail}}\label{proof:lemma:subgaussian_tail}
\begin{proof}
    Define $w_m = \max_{i \le t} |w_i|$. Because $w_t$ is sub-Gaussian, there exists $\alpha > 0$ such that for any $w \ge 0$, $\P(|w_t| \ge w) \le 2e^{-w^2/(2\alpha)}$. Therefore, we have for any $w \ge 0$,
\begin{align*}
    \P(w_m^2 \ge w) &= 1 - \P\left(\forall i \le t, |w_i| \le \sqrt{w}\right) \\
    &\le 1 - \left(1 - 2e^{-w/(2\alpha)}\right)^t \\
    &= O_T(te^{-w/(2\alpha)}). 
\end{align*}
This implies by the Law of Total Expectation that
\begin{align*}
    \E[w_m^2 \mid \neg F]\P(\neg F) &\le \P(\neg F)\log^6(T) + \P(w_m \ge \log^3(T))\E[w_m^2 \mid w_m \ge \log^3(T)] && \text{Lemma \ref{lemma:condsplit}}\\
    &=  o_T\left( \frac{1}{T^{10}}\right)+ \int_{\log^6(T)}^{\infty} \P(w_m^2 \ge w) dw \\
    &=  o_T\left( \frac{1}{T^{10}}\right)+ O_T\left(\int_{\log^6(T)}^{\infty} te^{-w/(2\alpha)}dw \right)\\
    &=  o_T\left( \frac{1}{T^{10}}\right)+ O_T\left(2t\alpha e^{-\log^6(T)/(2\alpha)}\right)\\
    &= o_T\left(\frac{1}{T^{10}}\right).
\end{align*}
\end{proof}

\subsection{Proof of Lemma \ref{offbyepsiloncontrol}}\label{proof:offbyepsiloncontrol}

\begin{proof}

    Fix a value of $s$. For $i \in [0: T_s]$, define the controller $C_t^i$ as the controller that at time $t < i$ uses controller $C^{\mathrm{alg}}_s$ and at time $t \ge i$ uses controller $C_{K_{\mathrm{opt}}(\hat{\theta}_s, T_s)}^{\hat{\theta}_s}$. We will compare the cost of controller $C_t^i$ versus controller $C_t^{i+1}$ over $T_s$ steps starting at position $x_{T_s}'$. Note that the cost of the first $i$ steps is the same, as $C_t^i = C_t^{i+1} = C^{\mathrm{alg}}_s$ for $t < i$. Therefore
    \[
        |i \cdot J(\theta^*, C^{i+1}_t, i, x_{   T_s}', \{w_{j}\}_{j=   T_s}^{   T_s+i-1})  - i \cdot J(\theta^*, C^{i}_t, i, x_{   T_s}',  \{w_{j}\}_{j=   T_s}^{   T_s+i-1})|  = 0.
    \]
    The position at time $i$ when using either $C_t^i$ or $C_t^{i+1}$ is $x'_{T_s+i}$. Conditional on event $E$ and for sufficiently large $T$, by Lemma \ref{bounded_approx} we have that $|C_s^{\mathrm{alg}}(x_{   T_s + i}')|, |C_{K_{\mathrm{opt}}(\hat{\theta}_s, T_s)}^{\hat{\theta}_s}(x_{   T_s + i}')| \le B_x$. Therefore conditional on event $E$ and for sufficiently large $T$,
    \begin{align*}
        &r\left(C_s^{\mathrm{alg}}(x_{   T_s + i}')^2 - C_{K_{\mathrm{opt}}(\hat{\theta}_s, T_s)}^{\hat{\theta}_s}(x_{   T_s + i}')^2\right) \\
        &\le 2r|C_s^{\mathrm{alg}}(x_{   T_s + i}') - C_{K_{\mathrm{opt}}(\hat{\theta}_s, T_s)}^{\hat{\theta}_s}(x_{   T_s + i}')| + r\left(C_s^{\mathrm{alg}}(x_{   T_s + i}') - C_{K_{\mathrm{opt}}(\hat{\theta}_s, T_s)}^{\hat{\theta}_s}(x_{   T_s + i}')\right)^2 \\
        &\le r(2 + 2B_x)|C_s^{\mathrm{alg}}(x_{   T_s + i}') - C_{K_{\mathrm{opt}}(\hat{\theta}_s, T_s)}^{\hat{\theta}_s}(x_{   T_s + i}')|. \numberthis \label{eq:offbyepsilonone1}   
    \end{align*}
    The difference in the next position when at position $x'_{   T_s+i}$ and using control $C_s^{\mathrm{alg}}(x_{   T_s + i}')$ versus $C_{K_{\mathrm{opt}}(\hat{\theta}_s, T_s)}^{\hat{\theta}_s}(x_{   T_s + i}')$ is
    \begin{align*}
        &\left| a^*x_{   T_s + i}' + b^*C_s^{\mathrm{alg}}(x_{   T_s + i}') + w_{   T_s + i} - (a^*x_{   T_s + i}' + b^*C_{K_{\mathrm{opt}}(\hat{\theta}_s, T_s)}^{\hat{\theta}_s}(x_{   T_s + i}') + w_{   T_s + i})\right| \\
        &= b^*|C_s^{\mathrm{alg}}(x_{   T_s + i}') - C_{K_{\mathrm{opt}}(\hat{\theta}_s, T_s)}^{\hat{\theta}_s}(x_{   T_s + i}')|. \numberthis \label{eq:bstar}
    \end{align*}
Under event $E$, the controls $C_{K_{\mathrm{opt}}(\hat{\theta}_s, T_s)}^{\hat{\theta}_s}(x_{   T_s + i}')$ and $C_s^{\mathrm{alg}}(x_{   T_s + i}')$ are safe for dynamics $\hat{\theta}_s$ and $\theta^*$, respectively and $\norm{\theta^* - \hat{\theta}_s}_{\infty} \le \epsilon_s \le \tilde{O}_T(\nu_T)$. Therefore, by Lemma \ref{bounded_approx}, conditional on event $E$ and for sufficiently large $T$, we have that $|x'_{T_s + i}| \le 4\log^2(T)$ and that
    \[
        |a^*x_{   T_s + i}' + b^*C_s^{\mathrm{alg}}(x_{   T_s + i}') + w_{   T_s + i}|, |a^*x_{   T_s + i}' + b^*C_{K_{\mathrm{opt}}(\hat{\theta}_s, T_s)}^{\hat{\theta}_s}(x_{   T_s + i}') + w_{   T_s + i}| \le 4\log^2(T).
    \]
    Conditional on event $E$ and for sufficiently large $T$, we therefore have by Lemma \ref{offbyepsilon} that conditional on $E_{\mathrm{A}\ref{parameterization_assum3}}(C_{K_{\mathrm{opt}}({\hat{\theta}_s}, T_s)}^{\hat{\theta}_s}, \{w_{j}\}_{j=   T_s+i+1}^{   T_{s+1}-1})$, we can bound the difference in future cost of the next $T_s - i- 1$ steps starting at time $   T_s + i+1$ when using control $C_s^{\mathrm{alg}}(x_{   T_s + i}')$ versus $C_{K_{\mathrm{opt}}(\hat{\theta}_s, T_s)}^{\hat{\theta}_s}(x_{   T_s + i}')$ as follows.
    \begin{align*}
        &(T_s-i-1)\Big|J(\theta^*, C_{K_{\mathrm{opt}}(\hat{\theta}_s, T_s)}^{\hat{\theta}_s}, T_s-i-1, a^*x_{   T_s + i}' + b^*C_s^{\mathrm{alg}}(x_{   T_s + i}') + w_{   T_s + i},  \{w_{j}\}_{j=   T_s+i+1}^{   T_{s+1}-1})  \\
        &\quad \quad - J(\theta^*, C_{K_{\mathrm{opt}}(\hat{\theta}_s, T_s)}^{\hat{\theta}_s}, T_s-i-1, a^*x_{   T_s + i}' + b^*C_{K_{\mathrm{opt}}(\hat{\theta}_s, T_s)}^{\hat{\theta}_s}(x_{   T_s + i}') + w_{   T_s + i},  \{w_{j}\}_{j=   T_s+i+1}^{   T_{s+1}-1})\Big| \\
        &= \tilde{O}_T\left(b^*|C_s^{\mathrm{alg}}(x_{   T_s + i}') - C_{K_{\mathrm{opt}}(\hat{\theta}_s, T_s)}^{\hat{\theta}_s}(x_{   T_s + i}')| + \epsilon_s\right). \quad \quad \quad \quad \quad \quad \text{[Eq \eqref{eq:bstar}, Lemma \ref{offbyepsilon}]} \numberthis \label{eq:offbyepsilonone3}
    \end{align*}
    Therefore, the difference in total cost between $C_t^i$ and $C_t^{i+1}$ conditional on event $E$ with probability $\P(E_{\mathrm{A}\ref{parameterization_assum3}}(C_{K_{\mathrm{opt}}({\hat{\theta}_s} T_s)}^{\hat{\theta}_s}, \{w_{j}\}_{j=   T_s+i+1}^{   T_{s+1}-1})) = 1-o_T(1/T^{10})$ is
    {\fontsize{10}{10}
    \begin{align*}
        &|T_s \cdot J(\theta^*, C_t^{i+1}, T_s, x_{   T_s}', W_s)  - T_s \cdot J(\theta^*, C^i_t, T_s, x_{   T_s}', W_s)  \\
        &\le \Big|i \cdot J(\theta^*, C_t^{i+1}, i, x_{   T_s}', \{w_{j}\}_{j=   T_s}^{   T_s+i})  - i \cdot J(\theta^*, C^{i}_t, i, x_{   T_s}', \{w_{j}\}_{j=   T_s}^{   T_s+i})\Big| + r\left(C_s^{\mathrm{alg}}(x_{   T_s + i}')^2 - C_{K_{\mathrm{opt}}(\hat{\theta}_s, T_s)}^{\hat{\theta}_s}(x_{   T_s + i}')^2\right) \\
        &\quad\quad  + (T_s-i-1)\Big|J(\theta^*, C_{K_{\mathrm{opt}}(\hat{\theta}_s, T_s)}^{\hat{\theta}_s}, T_s-i-1, a^*x_{   T_s + i}' + b^*C_s^{\mathrm{alg}}(x_{   T_s + i}') + w_{   T_s + i},  \{w_{j}\}_{j=   T_s+i+1}^{   T_{s+1}-1}) \\
        &\quad \quad - J(\theta^*, C_{K_{\mathrm{opt}}(\hat{\theta}_s, T_s)}^{\hat{\theta}_s}, T_s-i-1, a^*x_{   T_s + i}' + b^*C_{K_{\mathrm{opt}}(\hat{\theta}_s, T_s)}^{\hat{\theta}_s}(x_{   T_s + i}') + w_{   T_s + i},  \{w_{j}\}_{j=   T_s+i+1}^{   T_{s+1}-1})\Big| \\
        &\le 0 + r(2+2B_x)|C_s^{\mathrm{alg}}(x_{   T_s + i}') - C_{K_{\mathrm{opt}}(\hat{\theta}_s, T_s)}^{\hat{\theta}_s}(x_{   T_s + i}')|  \\
        &\quad \quad + \tilde{O}_T\left(b^*|C_s^{\mathrm{alg}}(x_{   T_s + i}') - C_{K_{\mathrm{opt}}(\hat{\theta}_s, T_s)}^{\hat{\theta}_s}(x_{   T_s + i}')| + \epsilon_s \right)\quad \text{Eq.  \eqref{eq:offbyepsilonone1}, \eqref{eq:offbyepsilonone3}} \\
        &= \tilde{O}_T\left(|C_s^{\mathrm{alg}}(x_{   T_s + i}') - C_{K_{\mathrm{opt}}(\hat{\theta}_s, T_s)}^{\hat{\theta}_s}(x_{   T_s + i}')| + \epsilon_s \right). \numberthis \label{eq:offbyepsilon_i}
    \end{align*}
    }
    We can use Equation \eqref{eq:offbyepsilon_i} for all $i \in [0:T_s-1]$, the triangle inequality, and a union bound to get that conditional on event $E$, with probability $1-o_T(1/T^9)$
    \begin{align*}
         &|T_s \cdot J(\theta^*, C_{K_{\mathrm{opt}}(\hat{\theta}_s, T_s)}^{\hat{\theta}_s}, T_s, x_{   T_s}', W_s)  - T_s \cdot J(\theta^*, C^{\mathrm{alg}}_s, T_s, x_{   T_s}', W_s)| \\
         &\le \sum_{i=0}^{T_s-1}  |T_s \cdot J(\theta^*, C^{i+1}_t, T_s, x_{   T_s}', W_s)  - T_s \cdot J(\theta^*, C^i_t, T_s, x_{   T_s}', W_s)| \\
         &= \tilde{O}_T\left(\sum_{i=0}^{T_s-1} |C_s^{\mathrm{alg}}(x_{   T_s + i}') - C_{K_{\mathrm{opt}}(\hat{\theta}_s, T_s)}^{\hat{\theta}_s}(x_{   T_s + i}')| + T_s\epsilon_s\right). \numberthis \label{eq:lemma_16_eq}
    \end{align*}

    The above was for a fixed value of $s$. Taking a union bound over all $s \in [0:s_e]$, we have that with conditional probability $1-o_T(1/T)$ given event $E$, the desired result holds for all $s \in [0: s_e]$.
\end{proof}

\section{Proofs of Sufficiently Large Noise Case}\label{app:suff_large_noise_case}

\subsection{Proof of Theorem \ref{sufficiently_large_error}}\label{proof:sufficiently_large_error}
First, we present the algorithm which is used to prove Theorem \ref{sufficiently_large_error}.
\begin{proof}
\begin{algorithm}[H]
\caption{Safe LQR for Large Noise}\label{alg:cap_large}
\begin{algorithmic}[1]
\Require $D, \mathcal{D}, \Theta, C^{\mathrm{init}}, \{\mathcal{C}^{\theta}\}_{\theta \in \Theta}, T, \lambda$
\State{$\nu_T \gets T^{-1/4}$} 
\For{$t \gets 0$ to $\frac{1}{\nu_T^2}-1$} \Comment{Safe warm-up exploration phase}
        \State{$\phi_t \sim \mathrm{Rademacher}(0.5)$}
        \State{Use control $u_t = C^{\mathrm{init}}(x_t) + \frac{\phi_t}{\log(T)}$}
\EndFor
\For{$s \gets 0$ to $\log_2(T\nu_T^2) - 1$} \Comment{Safe certainty equivalence phase}
    \State{$T_s \gets \frac{2^s}{\nu_T^2}$}
    \State{$\epsilon_s \gets B_{T_s} \sqrt{\frac{\max\left(V_{T_s}^{22}, V_{T_s}^{11}\right)}{V_{T_s}^{11}V_{T_s}^{22} - (V_{T_s}^{12})^2}}$}  
        \State{$\hat{\theta}^{\text{pre}}_s \gets (Z_{{T_s}}^{\top}Z_{{T_s}}+\lambda I)^{-1}Z_{{T_s}}^{\top}X_{{T_s}}$}
    \State{$\hat{\theta}_s \gets  \argmin_{\norm{\theta' - \hat{\theta}_s^{\text{pre}}}_{\infty} \le \epsilon_s} \min_{\norm{\theta -\hat{\theta}^{\text{pre}}_s}_{\infty} \le \epsilon_s} P(\theta, K_{\mathrm{opt}}(\theta', T_s), D_{\mathrm{U}})$}
    \State{$C_s^{\mathrm{alg}} \gets C^{\hat{\theta}_s}_{K_{\mathrm{opt}}(\hat{\theta}_s, T_s)}$}
    \For{$t \gets T_s$ to $2T_s - 1$}
        \State{$u_t^{\mathrm{safeU}} \gets \max \left\{ u :  \displaystyle\max_{\norm{\theta - \hat{\theta}_s^{\mathrm{pre}}}_{\infty} \le \epsilon_s} ax_t + bu \le D_{\mathrm{U}}\right\}$}
        \State{$u_t^{\mathrm{safeL}} \gets \min \left\{ u : \displaystyle\min_{\norm{\theta - \hat{\theta}_s^{\mathrm{pre}}}_{\infty} \le \epsilon_s} ax_t + bu \ge D_{\mathrm{L}}\right\}$}
        \State{Use control $u_t = \max\left(\min\left(C_s^{\mathrm{alg}}(x_t), u^{\mathrm{safeU}}_t\right),u^{\mathrm{safeL}}_t\right)$}
    \EndFor 
\EndFor
\end{algorithmic}
\end{algorithm}

Importantly, we note that Algorithm \ref{alg:cap_large} fundamentally only differs from Algorithm \ref{alg:cap} in two ways. The first is that $\nu_T$ changes from $T^{-1/3}$ (in Algorithm \ref{alg:cap}) to $T^{-1/4}$ (in Algorithm \ref{alg:cap_large}), which changes $T_s$ as well. The second is that the definition of $\hat{\theta}_s$ changes between the two algorithms.  Note that the definition of $\hat{\theta}_s^{\mathrm{pre}}$  in Algorithm \ref{alg:cap_large} is equivalent to the definition of $\hat{\theta}_s$ in Algorithm \ref{alg:cap}.  This means that the definitions of $u_t^{\mathrm{safeU}}$ and $u_t^{\mathrm{safeL}}$ are the same in both Algorithm \ref{alg:cap} and Algorithm \ref{alg:cap_large}.  Of course, the changes in $\nu_T$ and the definition of $\hat{\theta}_s$ change the entire trajectory of the algorithm, which affects all of the other variables as well. However, all other differences in the algorithm trajectory can be derived from these two changes.

For the rest of Appendix \ref{app:suff_large_noise_case}, let $C^{\mathrm{alg}}$ be the controller of Algorithm \ref{alg:cap_large}. Because Algorithm \ref{alg:cap_large} and Algorithm \ref{alg:cap} differ, we will now redefine the important events and lemmas from Appendix \ref{app:proof_of_performance} with respect to Algorithm \ref{alg:cap_large} (and the corresponding $\hat{\theta}_s$), and use this notation for the rest of Appendix \ref{app:suff_large_noise_case}. Define $s_e = \log_2(T\nu_T^2) - 1$, and let
\begin{equation}\label{eq:E0_large}
    E_0 := \left\{\forall s \le s_e : \norm{\theta^* - \hat{\theta}_s^{\mathrm{pre}}}_{\infty} \le \epsilon_s \right\}.
\end{equation} 
By Lemma  \ref{v_to_use} we have that with probability $1-o_T(1/T^2)$, $\norm{\theta^* - \hat{\theta}_s^{\mathrm{pre}}}_{\infty} \le \epsilon_s$. Therefore, 
\[
    \P(E_0) = 1-o_T(1/T^2).
\]
Also note that because $\norm{\hat{\theta}_s - \hat{\theta}_s^{\mathrm{pre}}}_{\infty} \le \epsilon_s$ by construction, we have by the triangle inequality that under event $E_0$, $\norm{\theta^* - \hat{\theta}_s}_{\infty} \le 2\epsilon_s$. 

For the rest of this section, define
\begin{equation}\label{eq:E2_large}
    E_2 := E_0 \bigcap \left\{ \max_{s \in [0:s_e]} \epsilon_s = \tilde{O}_T(\nu_T)\right\}.
\end{equation}
We also have the following equivalent result to Lemma \ref{initial_uncertainty}, but with respect to the $\epsilon_s$ in Algorithm \ref{alg:cap_large}.
\begin{lemma}\label{initial_uncertainty_large}
    Under Assumptions \ref{assum_init}--\ref{parameterization_assum3}, with probability $1-o_T(1/T^2)$
    \[
     \max_{s \in [0:s_e]} \epsilon_s = \tilde{O}_T(\nu_T).
    \]
\end{lemma}
The proof of Lemma \ref{initial_uncertainty} relies only on the first $\nu_T$ steps and is written agnostic to the choice of $\nu_T$, and therefore the result of Lemma \ref{initial_uncertainty_large} follows directly from that proof.  Lemma \ref{initial_uncertainty_large} implies that we have
\[
    \P(E_2) = 1-o_T(1/T^2).
\]
For this section, $E_1$ will still refer to the same event as in Equation \eqref{eq:E1}. We also define the event $E_{\mathrm{safe}}$ the same way as in Equation \eqref{eq:safeE} except with respect to the positions and controls of Algorithm \ref{alg:cap_large}, and finally we define the event $E = E_1 \cap E_2 \cap E_{\mathrm{safe}}$ (the same as in Appendix \ref{sec:regret}). Therefore by a union bound we still have that $\P(E) = 1-o_T(1/T^2)$. Using this new notation and Lemma \ref{initial_uncertainty_large}, we can proceed to the main proof. 

The safety of $C^{\mathrm{alg}}$ follows from an equivalent version of Lemma \ref{safety_append}, except stated for Algorithm \ref{alg:cap_large} instead of Algorithm \ref{alg:cap}. The proof follows as in the proof of Lemma \ref{safety_append} except using Lemma \ref{initial_uncertainty_large} instead of Lemma \ref{initial_uncertainty}, and using the above definitions of $E_0,E_1$ and $E_2$ with respect to Algorithm \ref{alg:cap_large}. An equivalent statement of Lemma \ref{lemma:L_less_than_U} holds except for the $u_t^{\mathrm{safeU}}$ and $u_t^{\mathrm{safeL}}$ coming from Algorithm \ref{alg:cap_large}. Note that the only place that the proof of Lemma \ref{lemma:L_less_than_U} relies on $\nu_T$ is that it requires that $\epsilon_s = \tilde{O}_T(\nu_T)$ and that $\tilde{O}_T(\nu_T) = o_T(1/\log(T))$ at multiple points in the proof, which still holds under the new definitions of $E_2$ and $\nu_T$. Finally, as noted above, the $u_t^{\mathrm{safeU}}$ and $u_t^{\mathrm{safeL}}$ are constructed in the same way for both algorithms, and therefore the rest of the proof of Lemma \ref{safety_append} follows directly.

The rest of this section will focus on bounding the regret of Algorithm \ref{alg:cap_large} to be $\tilde{O}_T(\sqrt{T})$ with probability $1-o_T(1/T)$. Informally, the key idea behind the regret bound of Algorithm \ref{alg:cap_large} is that with high probability, the uncertainty upper bound $\epsilon_s$ will decrease at a rate proportional to $1/\sqrt{   T_s}$. This is formalized in Lemma \ref{bounded_st}.

    \begin{lemma}\label{bounded_st}
        Under Assumptions \ref{assum_init}--\ref{assum_sufficiently_large_error}, given event $E$ with conditional probability $1-o_T(1/T)$,
        \[
            \max_{s \in [0: s_e]} \epsilon_s\sqrt{T_s}  = \tilde{O}_T(1).
        \]
    \end{lemma}
    The proof of Lemma \ref{bounded_st} can be found in Appendix \ref{proof:bounded_st}.
   
   Define event $E_3$ as the event
    \[
        E_3 = \left\{\max_{s \in [0: s_e]} \epsilon_s\sqrt{T_s}  = \tilde{O}_T(1)\right\}.
    \]
    By Lemma \ref{bounded_st}, $\P(E_3) = 1-o_T(1/T)$. 
    We can decompose the regret of Algorithm \ref{alg:cap_large} into the same components of regret as in Appendix \ref{sec:regret}. The first two propositions stated below are exactly equivalent to their counterparts in Appendix \ref{sec:regret}.
    
\begin{proposition}[Regret from Warm-up Period]\label{warm_up_regret_large}
    Define $x'_0, x'_1,...$ as the sequence of random variables that are the positions of the controller $C^{\mathrm{alg}}$ defined in Algorithm \ref{alg:cap_large}. Define $R_0$ as the cost of the first $1/\nu_T^2$ steps, i.e.
    \begin{equation}
        R_0 = T \cdot J(\theta^*, C^{\mathrm{alg}}, T, 0, W) - \sum_{s=0}^{s_{e}}  T_s \cdot J(\theta^*,C^{\mathrm{alg}}_s, T_s,  x'_{   T_s}, W_s).
    \end{equation}
    Then under Assumptions \ref{assum_init}--\ref{parameterization_assum3} and conditional on event $E$,
    \[
        R_0 \stackrel{\text{a.s.}}{\le} \tilde{O}_T\left(\frac{1}{\nu_T^2} \right).
    \]
\end{proposition}
The proof of Proposition \ref{warm_up_regret_large} can be found in Appendix \ref{proof:warm_up_regret_large}.

\begin{proposition}[Regret from Randomness]\label{r1b_bound_large}
    Define $\hat{x}_{   T_0},\hat{x}_{   T_0+1},...$ as the sequence of random variables representing the sequence of positions if the control at each time $t \ge    T_0$ is $C^{\hat{\theta}_s}_{K_{\mathrm{opt}}(\hat{\theta}_s, T_s)}(x_t)$ for $s = \lfloor\log_2\left(t\nu_T^2\right)\rfloor$ and starting at $\hat{x}_{   T_0} = x'_{   T_0}$. Define $R_2$ as 
    \[
        R_2 := \sum_{s=0}^{s_e}  T_sJ(\theta^*,C^{\hat{\theta}_s}_{K_{\mathrm{opt}}(\hat{\theta}_s, T_s)}, T_s,  \hat{x}_{   T_s}, W_s) -  \sum_{s=0}^{s_e}  \E\left[T_sJ(\theta^*,C^{\hat{\theta}_s}_{K_{\mathrm{opt}}(\hat{\theta}_s, T_s)}, T_s,  0, W_s) \; \cond \; \hat{\theta}_s \right].
    \]
    Then with conditional probability $1-o_T(1/T)$ given event $E$,
    \begin{equation}\label{eq:r1b_bound_suff}
         R_2 \le \tilde{O}_T(\sqrt{T}).
    \end{equation}
\end{proposition}
The proof of Proposition \ref{r1b_bound_large} can be found in Appendix \ref{proof:r1b_bound_large}. The next two propositions have different regret bounds than their counterparts in Appendix \ref{sec:regret}.

\begin{proposition}[Regret from Non-optimal Controller with Sufficiently Large Noise]\label{non_optimal_controller_suff}
 Define $R_1$ as 
 \[
    R_1 :=  \sum_{s=0}^{s_e}  \E\left[T_sJ(\theta^*,C^{\hat{\theta}_s}_{K_{\mathrm{opt}}(\hat{\theta}_s, T_s)}, T_s, 0, W_s) \cond \hat{\theta}_s \right] -  \E\left[\sum_{s=0}^{s_e} T_sJ(\theta^*,C^{\theta^*}_{K^*_s}, T_s,  x^*_{   T_s}, W_s)  \right].
 \]
 Note that $W_s$ is independent of $\hat{\theta}_s$ by construction. Then under Assumptions \ref{assum_init}--\ref{assum_sufficiently_large_error} and conditional on event $E_2 \cap E_3$,
    \begin{equation}\label{eq:non_optimal_controller_sufflarge}
       R_1 \stackrel{\text{a.s.}}{\le} \tilde{O}_T\left(\sqrt{T}\right).
    \end{equation}
\end{proposition}
The proof of Proposition \ref{non_optimal_controller_suff} can be found in Appendix \ref{proof:non_optimal_controller_suff}.

\begin{proposition}[Regret from Enforcing Safety with Sufficiently Large Noise]\label{enforcing_safety_suff}
    Define $\hat{x}_{   T_0},\hat{x}_{   T_0+1},...$ as the sequence of random variables representing the sequence of positions if the control at each time $t \ge    T_0$ is $C^{\hat{\theta}_s}_{K_{\mathrm{opt}}(\hat{\theta}_s, T_s)}(\hat{x}_t)$ for $s = \lfloor\log_2\left(t\nu_T^2\right)\rfloor$ and starting at $\hat{x}_{   T_0} = x'_{   T_0}$. Define $R_3$ as (the random variable)
    \[
        R_3 := \sum_{s=0}^{s_e}  T_sJ(\theta^*,C^{\mathrm{alg}}_s, T_s,  x'_{   T_s}, W_s) - \sum_{s=0}^{s_e} T_sJ(\theta^*,C^{\hat{\theta}_s}_{K_{\mathrm{opt}}(\hat{\theta}_s, T_s)}, T_s,  \hat{x}_{   T_s}, W_s).
    \]
    Then under Assumptions \ref{assum_init}--\ref{assum_sufficiently_large_error}, with conditional probability $1-o_T(1/T)$ given event $E \cap E_3$,
    \[
        R_3 \le \tilde{O}_T(\sqrt{T}).
    \]
\end{proposition}
The proof of Proposition \ref{enforcing_safety_suff} can be found in Appendix \ref{proof:enforcing_safety_suff}.

     Using Equation \eqref{eq:all_decomposition} combined with Propositions \ref{warm_up_regret_large}, \ref{r1b_bound_large}, \ref{non_optimal_controller_suff} and \ref{enforcing_safety_suff}, the total regret is upper bounded by the following conditioned on event $E \cap E_3$, with conditional probability $1-o_T(1/T)$
     \[
       T\cdot J(\theta^*, C^{\mathrm{alg}},T) - T \cdot \bar{J}(\theta^*, C^{\theta^*}_{K_{\mathrm{opt}}(\theta^*, T)}, T) \le R_0 + R_1 + R_{2} + R_3 = \tilde{O}_T\left(\frac{1}{\nu_T^2} + \sqrt{T}\right).
     \]
     Because $\nu_T = T^{-1/4}$ in Algorithm \ref{alg:cap_large}, this gives total regret of $\tilde{O}_T(\sqrt{T})$ conditional on $E_3 \cap E$. Since $\P(E_3) = 1-o_T(1/T)$ and $\P(E) = 1-o_T(1/T)$, a union bound gives that the regret of Algorithm \ref{alg:cap_large} is $\tilde{O}_T(\sqrt{T})$ with unconditional probability $1-o_T(1/T)$.
\end{proof}

\subsection{Proof of Lemma \ref{bounded_st}(Uncertainty bounds using boundary times)}\label{proof:bounded_st}
\begin{proof}

To prove this lemma, we will show that the controller $C^{\mathrm{alg}}$ uses the control $u_i^{\mathrm{safeU}}$ ``sufficiently frequently". Let $S_t$ be the set of times $i < t$ when the control used by Algorithm \ref{alg:cap_large} is $u_i^{\mathrm{safeU}}$. Formally, if $u'_0,u'_1,...u'_{T-1}$ are the sequence of controls used by $C^{\mathrm{alg}}$, then
\begin{equation}
    S_t = \{i < t : u'_i = u_i^{\mathrm{safeU}}\}.
\end{equation}
\begin{lemma}\label{s_t_lowerbound}
    Under Assumptions \ref{assum_init}--\ref{assum_sufficiently_large_error} and conditional on event $E$ with conditional probability $1-o_T(1/T)$,
    \[
        \min_{s \in [1:s_e]} \frac{|S_{T_s}|}{T_s} = \Omega_T(1).
    \]
\end{lemma}
The proof of Lemma \ref{s_t_lowerbound} can be found in Appendix \ref{proof:s_t_lowerbound}.
Equipped with the fact that $|S_t|$ scales linearly with $t$ from Lemma \ref{s_t_lowerbound}, we need the following result that  will lower the upper bound for $\epsilon_s$. 
\begin{lemma}\label{boundary_uncertainty}
    Under Assumptions \ref{assum_init}--\ref{assum_sufficiently_large_error} and conditional on event $E$ with conditional probability $1-o_T(1/T)$,
    \[
        \max_{s \in [0:s_e]} \epsilon_s\sqrt{|S_{T_s}|} = \tilde{O}_T(1).
    \]
\end{lemma}
The proof of Lemma \ref{boundary_uncertainty} can be found in Appendix \ref{sec:proof_of_boundary_uncertainty}.
    To see that $\epsilon_0\sqrt{T_0} = \tilde{O}_T(1)$, note that $\sqrt{T_0} = 1/\nu_T$ and Lemma \ref{initial_uncertainty_large} imply that conditional on event $E$, $\epsilon_0 = \tilde{O}_T(\nu_T)$. For $s > 0$, a union bound combining Lemma \ref{s_t_lowerbound} with Lemma \ref{boundary_uncertainty} gives the desired result that conditioned on event $E$ with conditional probability $1-o_T(1/T)$,
    \[
        \max_{s \in [0: s_e]} \epsilon_s\sqrt{T_s} = \tilde{O}_T(1).
    \]
\end{proof}

\subsection{Proof of Proposition \ref{warm_up_regret_large}}\label{proof:warm_up_regret_large}

\begin{proof}
    The proof of Proposition \ref{warm_up_regret_large} follows the same as the proof of Proposition \ref{warm_up_regret}. The proof of Proposition \ref{warm_up_regret} relies on the fact that the controller is safe for dynamics $\theta^*$ conditional on event $E$. This is still true by construction of event $E$, and therefore the result follows directly.
\end{proof}

\subsection{Proof of Proposition \ref{r1b_bound_large}}\label{proof:r1b_bound_large}

\begin{proof}
    Note that this statement is exactly the same as the statement of Proposition \ref{r1b_bound} except for Algorithm \ref{alg:cap_large}. The proof of Proposition \ref{r1b_bound} relies on Lemmas \ref{concentration_of_cond_exp} and \ref{uncond_vs_cond_regret}. Define the event $E_2^s$ as
    \begin{equation}\label{eq:e2s_sec}
        E_2^s = \left\{ \norm{\hat{\theta}_s - \theta^*}_{\infty} \le 2 \cdot \epsilon_s \le 2 c_T \cdot \nu_T \right\},
    \end{equation}
    where the $c_T = \tilde{O}_T(1)$ from Lemma \ref{initial_uncertainty_large}. Note that we still have $\P(E_2^s) \ge \P(E_2) \ge 1-o_T(1/T^2)$. An analogous version of Lemma \ref{concentration_of_cond_exp} holds with this new definition of $E_2^s$ for Algorithm \ref{alg:cap_large}. Examining Lemma \ref{concentration_of_cond_exp}, the proof relies on $\hat{\theta}_s$ and $\nu_T$ through Lemma \ref{mcdiarmids_app}. A version of Lemma \ref{mcdiarmids_app} holds with the exact same statement with the new definition of $E_2^s$. Examining the proof of Lemma \ref{mcdiarmids_app}, we must have that under event $E_2^s$, $\norm{\hat{\theta}_s - \theta^*}_{\infty} \le \tilde{O}_T(\nu_T) \le \min(\epsilon_{\mathrm{A}\ref{parameterization_assum3}}, \frac{1}{\log(T)})$ in order to apply Lemmas \ref{offbyepsilon} and \ref{bounded_approx}, and this holds for $\nu_T = T^{-1/4}$. Therefore, we have shown the equivalent version of Lemma \ref{concentration_of_cond_exp} for Algorithm \ref{alg:cap_large}.

    Similarly, an analogous version of Lemma \ref{uncond_vs_cond_regret} holds for Algorithm \ref{alg:cap_large}. Lemma \ref{uncond_vs_cond_regret} depends on $\hat{\theta}_s$ and $\nu_T$ only in that it uses $\norm{\theta^* - \hat{\theta}_s}_{\infty} \le 1/\log(T)$ conditional on event $E_2^s$, which still holds by construction of $E_2^s$ for $\nu_T = T^{-1/4}$ and sufficiently large $T$.

    Now that we have shown that equivalent versions of Lemmas \ref{concentration_of_cond_exp} and \ref{uncond_vs_cond_regret} still hold, we can return to the proof of Proposition \ref{r1b_bound}. Outside of the two lemmas discussed above, the only places in the proof that depend on the choice of $\nu_T$ and $\hat{\theta}_s$ is that $s_e = \tilde{O}_T(1)$ is still true in Equation \eqref{eq:sumofsqrtT} and that conditional on event $E$, $\norm{\hat{\theta}_s- \theta^*}_{\infty} \le \tilde{O}_T(\nu_T) \le  \min(\epsilon_{\mathrm{A}\ref{parameterization_assum3}}, \frac{1}{\log(T)})$ in order to apply Lemmas \ref{bounded_approx} and \ref{offbyepsilon}. As both of these still hold for the new definition of $E$ and for $\nu_T = T^{-1/4}$, we are done.
\end{proof}

    \subsection{Proof of Proposition \ref{non_optimal_controller_suff}}\label{proof:non_optimal_controller_suff}
    \begin{proof}
        The proof of Proposition \ref{non_optimal_controller_suff} will mostly follow as in the proof of Proposition \ref{non_optimal_controller}. The proof of Proposition \ref{non_optimal_controller} relies on Lemma \ref{start_invariant_inexpectation}. An equivalent version of Lemma \ref{start_invariant_inexpectation} holds for Algorithm \ref{alg:cap_large}, where the only difference is that the $T_s$ are now defined differently. To see this, note that the proof of Lemma \ref{start_invariant_inexpectation} works for any $T_s \le T$, and therefore the proof follows exactly the same. 

        Returning to the proof of Proposition \ref{non_optimal_controller}, we can still apply Assumption \ref{parameterization_assum2} under the event $E_2^s$ as defined in Equation \eqref{eq:e2s_sec}. Looking at the last block of equations in Proposition \ref{non_optimal_controller},  we can follow the logic exactly and pick up from the second to last line. Applying Lemma \ref{bounded_st}, conditional on $E \cap E_3$,  
        \begin{align*}
             &\sum_{s=0}^{s_e} \E\left[T_sJ(\theta^*,C^{\hat{\theta}_s}_{K_{\mathrm{opt}}(\hat{\theta}_s, T_s)}, T_s, 0, W_s)\cond \hat{\theta}_s  \right] -\E\left[ \sum_{s=0}^{s_e}  T_sJ(\theta^*,C^{\theta^*}_{K^*_s}, T_s,  x^*_{   T_s}, W_s) \right]  \\
             &\le \tilde{O}_T(1) + \tilde{O}_T\left(\sum_{s=0}^{s_e} T_s\epsilon_s + \frac{T_s}{T^2}\right) \\
             &= \tilde{O}_T(1) + \tilde{O}_T\left(\sum_{s=0}^{s_e} T_s\tilde{O}_T\left(\frac{1}{\sqrt{T_s}}\right) + \frac{T_s}{T^2}\right) && \text{Lemma \ref{bounded_st}} \\
            &= \tilde{O}_T(\sqrt{T}).
        \end{align*}
    \end{proof}
    
    \subsection{Proof of Proposition \ref{enforcing_safety_suff}}\label{proof:enforcing_safety_suff}
    
    \begin{proof}
        The proof of Proposition \ref{enforcing_safety_suff} will mostly follow as in the proof of Proposition \ref{enforcing_safety}. The proof of Proposition \ref{enforcing_safety} relies on Lemmas \ref{offbyepsiloncontrol_propproof} and \ref{bound_on_cont_diff_propproof}. We will show that equivalent versions of these lemmas hold for Algorithm \ref{alg:cap_large}. 
        
        Starting with Lemma \ref{offbyepsiloncontrol_propproof}, an equivalent version holds for the $u_t^{\mathrm{safeU}}$ and $u_t^{\mathrm{safeL}}$ defined in Algorithm \ref{alg:cap_large} and $C^{\mathrm{alg}}$ as the controller of Algorithm \ref{alg:cap_large}. Looking at the proof of Lemma \ref{offbyepsiloncontrol_propproof}, the main tool is Lemma \ref{offbyepsiloncontrol}. An equivalent version of Lemma \ref{offbyepsiloncontrol} holds for Algorithm \ref{alg:cap_large}. Looking at the proof of Lemma \ref{offbyepsiloncontrol}, the dependency on $\hat{\theta}_s$ and $\nu_T$ is that we must have that conditional on event $E$, $\norm{\hat{\theta}_s - \theta^*}_{\infty} \le \tilde{O}_T(\nu_T) \le  \min(\epsilon_{\mathrm{A}\ref{parameterization_assum3}}, \frac{1}{\log(T)})$ in order to apply Lemmas \ref{bounded_approx} and \ref{offbyepsilon}. The union bound at the end of the proof also relies on $s_e = \tilde{O}_T(1)$, which also does still hold. Returning to the proof of the equivalent of Lemma \ref{offbyepsiloncontrol_propproof} for Algorithm \ref{alg:cap_large}, we again need that conditional on event $E$, $\norm{\hat{\theta}_s - \theta^*}_{\infty} \le \tilde{O}_T(\nu_T) \le  \min(\epsilon_{\mathrm{A}\ref{parameterization_assum3}}, \frac{1}{\log(T)})$ in order to apply Lemmas \ref{bounded_approx} and \ref{offbyepsilon}.  Once again using that $s_e = \tilde{O}_T(1)$, the rest of the proof of Lemma \ref{offbyepsiloncontrol_propproof} can be directly applied.

        An equivalent version of Lemma \ref{bound_on_cont_diff_propproof} also holds when $C^{\mathrm{alg}}$ is the controller of Algorithm \ref{alg:cap_large} with $\nu_T = T^{-1/4}$. We defer the proof of this to Appendix \ref{sec:equivalent}.

        Now we can return to the proof of Proposition \ref{enforcing_safety} and show that a slight modification gives the desired result. Looking at the last set of equations, we can pick up from the third line and apply Lemma \ref{bounded_st} to get that, conditional on event $E \cap E_3$,
        \begin{align*}
            R_3 &\le  \tilde{O}_T\left(\sum_{s=0}^{s_e} T_s \epsilon_s\right) + \sum_{s=0}^{s_e} \sum_{t=T_s}^{T_{s+1}-1}X_{t}^U \cdot\tilde{O}_T\left(\epsilon_{s }\right) + X_t^L \cdot\tilde{O}_T\left(\epsilon_{s }\right)\\
            &\le \tilde{O}_T\left(\sum_{s=0}^{s_e} T_s \epsilon_s\right) \\
            &\le \sum_{s=0}^{s_e} T_s \cdot \tilde{O}_T\left(\frac{1}{\sqrt{T_s}}\right) && \text{Lemma \ref{bounded_st}}  \\
            &=  \tilde{O}_T(\sqrt{T}).
        \end{align*}
        The last line follows from the fact that for all $s$, $T_s \le T$ and that $s_e = \tilde{O}_T(1)$.
    \end{proof}

\subsection{Proof of Equivalent Version of Lemma \ref{bound_on_cont_diff_propproof} for Algorithm \ref{alg:cap_large}}\label{sec:equivalent}

Examining the proof of Lemma \ref{bound_on_cont_diff_propproof}, the main change when using Algorithm \ref{alg:cap_large} is that we now have that under event $E$ and for sufficiently large $T$, $\norm{\theta^* - \hat{\theta}_s}_{\infty} \le 2\epsilon_s$ (while for Algorithm \ref{alg:cap} there was no factor of $2$). Because $\nu_T = T^{-1/4}$, this still allows us to apply Lemma \ref{bounded_approx}. Picking up the proof of Lemma \ref{bound_on_cont_diff_propproof} directly before Equation \eqref{eq:ab1}, the extra factor of $2$ mentioned above will result in the following changes.

    By the construction of Algorithm \ref{alg:cap_large}, $u_t^{\mathrm{safeU}}$ satisfies, for some $\theta$ such that $\norm{\theta -  \hat{\theta}_s^{\mathrm{pre}}}_{\infty} \le \epsilon_s$,
    \begin{equation}\label{eq:ab1_large}
        ax'_t + bu_t^{\mathrm{safeU}} = D_{\mathrm{U}}.
    \end{equation}
    Under event $E$, $\norm{\theta^* - \hat{\theta}_s}_{\infty} \le 2\epsilon_s$ and $\norm{\hat{\theta}_s - \hat{\theta}_s^{\mathrm{pre}}}_{\infty} \le \epsilon_s$, which implies that $\norm{\theta^* - \theta}_{\infty} \le 4\epsilon_s \le \tilde{O}_T(\nu_T) \le 1/\log(T)$ for sufficiently large $T$. Therefore, applying Lemma \ref{bounded_approx} gives that under event $E$ and for sufficiently large $T$, 
    \begin{align*}
     D_{\mathrm{U}} &\ge a^*x'_t + b^*u_t^{\mathrm{safeU}} && \text{$u_t^{\mathrm{safeU}}$ safe for $\theta^*$} \\
     &\ge ax'_t + bu_t^{\mathrm{safeU}} - |u_t^{\mathrm{safeU}}|4\epsilon_s - |x'_t|4\epsilon_s  && \text{$\norm{\theta^* - \theta}_{\infty} \le 4\epsilon_s$}\\
     &\ge D_{\mathrm{U}} - 8B_x\epsilon_s. && \text{ Equations \eqref{safeL_boundedapprox0},\eqref{safeL_boundedapprox}, and \eqref{eq:ab1_large} } \numberthis \label{eq:enforcingsafety3_large}
    \end{align*}    
    If $u_t^{\mathrm{safeU}} \le C^{\hat{\theta}_s}_{K_{\mathrm{opt}}(\hat{\theta}_s, T_s)}(x'_t)$, then there must exist some $\theta$ such that $\norm{\hat{\theta}_s^{\mathrm{pre}} - \theta}_{\infty} \le \epsilon_s$ and
    \begin{equation}\label{eq:ab2_large}
        ax'_t + bC^{\hat{\theta}_s}_{K_{\mathrm{opt}}(\hat{\theta}_s, T_s)}(x'_t) \ge D_{\mathrm{U}}.
    \end{equation}
    By the same logic as above, under event $E$, $\norm{\theta^* - \theta}_{\infty} \le 4\epsilon_s \le \tilde{O}_T(\nu_T) \le 1/\log(T)$ for sufficiently large $T$, therefore under event $E$ and for sufficiently large $T$,
    \begin{align*}
        &a^*x'_t + b^* C^{\hat{\theta}_s}_{K_{\mathrm{opt}}(\hat{\theta}_s, T_s)}(x'_t) \\
        &\ge ax'_t + b C^{\hat{\theta}_s}_{K_{\mathrm{opt}}(\hat{\theta}_s, T_s)}(x'_t) -4\epsilon_s|x'_t| - 4\epsilon_s\left|C^{\hat{\theta}_s}_{K_{\mathrm{opt}}(\hat{\theta}_s, T_s)}(x'_t)\right| \\
        &\ge  D_{\mathrm{U}} - 8B_x\epsilon_s.  && \text{Eqs \eqref{safeL_boundedapprox0},\eqref{safeL_boundedapprox2}, and \eqref{eq:ab2_large}}\numberthis \label{eq:enforcingsafety5_large}
    \end{align*}
    Finally, because $C^{\hat{\theta}_s}_{K_{\mathrm{opt}}(\hat{\theta}_s, T_s)}(x'_t)$ is safe for dynamics $\hat{\theta}_s$, 
    \begin{equation}\label{eq:enforcingsafety6_large}
        \hat{a}_sx'_t +\hat{b}_s C^{\hat{\theta}_s}_{K_{\mathrm{opt}}(\hat{\theta}_s, T_s)}(x'_t) \le D_{\mathrm{U}}.
    \end{equation}
    Using that under event $E$, $\norm{\theta^* - \hat{\theta}_s}_{\infty} \le 2\epsilon_s \le \tilde{O}_T(\nu_T) \le 1/\log(T)$ for sufficiently large $T$, Equations \eqref{safeL_boundedapprox0}, \eqref{safeL_boundedapprox2}, and \eqref{eq:enforcingsafety6_large} imply that under event $E$ and for sufficiently large $T$,
    \begin{equation}\label{eq:enforcingsafety7_large}
        a^*x'_t + b^*C^{\hat{\theta}_s}_{K_{\mathrm{opt}}(\hat{\theta}_s, T_s)}(x'_t) \le D_{\mathrm{U}} + 4B_x\epsilon_s.
    \end{equation}
    Combining Equations \eqref{eq:enforcingsafety5_large} and \eqref{eq:enforcingsafety7_large}, if $u_t^{\mathrm{safeU}} \le C^{\hat{\theta}_s}_{K_{\mathrm{opt}}(\hat{\theta}_s, T_s)}(x'_t)$ then under event $E$ and for sufficiently large $T$,
    \[
        D_{\mathrm{U}} - 8B_x\epsilon_s \le a^*x'_t + b^*C^{\hat{\theta}_s}_{K_{\mathrm{opt}}(\hat{\theta}_s, T_s)}(x'_t) \le D_{\mathrm{U}} + 4B_x\epsilon_s.
    \]
    Combining this with  Equation \eqref{eq:enforcingsafety3_large} gives that under event $E$ and for sufficiently large $T$,
    \[
        |(a^*x'_t + b^*u_t^{\mathrm{safeU}}) - (a^*x'_t + b^*C^{\hat{\theta}_s}_{K_{\mathrm{opt}}(\hat{\theta}_s, T_s)}(x'_t))| \le 12B_x\epsilon_s.
    \]
    This implies the desired result that under event $E$ and for sufficiently large $T$,
    \[
        |u_t^{\mathrm{safeU}} -  C^{\hat{\theta}_s}_{K_{\mathrm{opt}}(\hat{\theta}_s, T_s)}(x'_t)| \le 12B_x\epsilon_s/b^*.
    \]

    \subsection{Proof of Lemma \ref{s_t_lowerbound}}\label{proof:s_t_lowerbound}
    \begin{proof}
        In this proof, we will use the following result about the times the algorithm chooses control $u_t^{\mathrm{safeU}}$.

        \begin{lemma}\label{safe_condition_P}
        Let $x'_t$, $u'_t$ be the positions and controls of controller $C^{\mathrm{alg}}$ at time $t$. For $t \ge T_0$, let $s_t = \lfloor \log_2(t\nu_T) \rfloor$. Then under Assumptions \ref{assum_init}--\ref{assum_sufficiently_large_error}, there exists a $P_{s_t}(\hat{\theta}_{s_t}, \epsilon_{s_t})$ such that
        \[
            \{x'_t \ge P_{s_t}(\hat{\theta}_{s_t}, \epsilon_{s_t})\} \subseteq \{u'_t = u_t^{\mathrm{safeU}}\},
        \]
        and such that conditional on event $E$, we have $P_{s_t}(\hat{\theta}_{s_t}, \epsilon_{s_t}) \le P(\theta^*, K_{\mathrm{opt}}(\theta^*,T_{s_t}), D_{\mathrm{U}})$.
\end{lemma}
The proof of Lemma \ref{safe_condition_P} can be found in Appendix \ref{proof:safe_condition_P}.
    Recall that $\{i \in S_{T_s}\} = \{u'_i = u_i^{\mathrm{safeU}}\}$. Therefore, for $i \in [T_s:T_{s+1} - 1]$, Lemma \ref{safe_condition_P} implies that
     \begin{align*}
     \{x'_i \ge P(\theta^*,K_{\mathrm{opt}}(\theta^*, T_s), D_{\mathrm{U}})\}  \cap E  &\subseteq  \{x'_i \ge P_{s_t}(\hat{\theta}_{s_t}, \epsilon_{s_t})\}  \cap E \\
     &\subseteq \{u'_i = u_i^{\mathrm{safeU}}\} \cap E \\
     &=  \{i \in S_{T_s}\} \cap E . \numberthis \label{eq:subsets_of_Sts}
     \end{align*}
        By Assumption \ref{assum_sufficiently_large_error}, for any $x'_{i-1}, u'_{i-1}$ satisfying $a^*x'_{i-1} + b^*u'_{i-1} \in [D_{\mathrm{L}}, D_{\mathrm{U}}]$, we have that
        \begin{align*}
         &\P\left(x'_{i} \ge P(\theta^*,K_{\mathrm{opt}}(\theta^*, T_{s_t}), D_{\mathrm{U}}) \cond  x'_{i-1}, u'_{i-1} \right) \\
         &\ge \P\left(w_{i} \ge P(\theta^*,K_{\mathrm{opt}}(\theta^*, T_{s_t}), D_{\mathrm{U}}) - D_{\mathrm{L}} \right)  \\
         &\ge \epsilon_{\mathrm{A}\ref{assum_sufficiently_large_error}}. && \text{Assumption \ref{assum_sufficiently_large_error}} \numberthis \label{eq:safeimplieseps}
        \end{align*}
        Because $\P(E) \ge 1-o_T(1/T)$, this implies for sufficiently large $T$ and for any $x'_{i-1}, u'_{i-1}$ satisfying $a^*x'_{i-1} + b^*u'_{i-1} \in [D_{\mathrm{L}}, D_{\mathrm{U}}]$,
        \begin{align*}
            \P\left(x'_{i} \ge P(\theta^*,K_{\mathrm{opt}}(\theta^*, T_{s_t}), D_{\mathrm{U}}) \cond x'_{i-1}, u'_{i-1}, E \right) &\ge \epsilon_{\mathrm{A}\ref{assum_sufficiently_large_error}} - \P( \neg E) \\
            &\ge \epsilon_{\mathrm{A}\ref{assum_sufficiently_large_error}} - o_T(1/T) \\
            &\ge \frac{\epsilon_{\mathrm{A}\ref{assum_sufficiently_large_error}}}{2}. \numberthis \label{eq:prob_with_E}
        \end{align*}
        Also, recall that conditional on event $E$, $C^{\mathrm{alg}}$ is safe for dynamics $\theta^*$ for all $T$ steps, therefore conditional on event $E$, for all $i \ge 1$, $D_{\mathrm{L}} \le a^*x'_{i-1} + b^*u'_{i-1} \le D_{\mathrm{U}}$. 
        Therefore, for $T_1 \le i < T_{s}$ and sufficiently large $T$,
        \begin{align*}
            &\P\left(i \in S_{T_s} \middle| x'_{0},x'_{1},...,x'_{i-1}, u'_{0},u'_{1},...,u'_{i-1}, E \right) \\
            &\ge \P\left(x'_i \ge P(\theta^*,K_{\mathrm{opt}}(\theta^*, T_s), D_{\mathrm{U}}) \middle| x'_{0},x'_{1},...,x'_{i-1}, u'_{0},u'_{1},...,u'_{i-1}, E \right) && \text{Equation \eqref{eq:subsets_of_Sts} }\\
            &\ge  \P\left(x'_i \ge P(\theta^*,K_{\mathrm{opt}}(\theta^*, T_s), D_{\mathrm{U}}) \middle| x'_{i-1}, u'_{i-1}, E \right) \\
            &\ge \frac{\epsilon_{\mathrm{A}\ref{assum_sufficiently_large_error}}}{2}. && \text{Equation \eqref{eq:prob_with_E}}  \numberthis \label{eq:prob_with_E2}
        \end{align*}
        Defining $X_i = 1_{i \in S_{T_s}}$, the above equation is equivalent to
        \[
            \E\left[ X_i \cond  x'_{0},x'_{1},...,x'_{i-1}, u'_{0},u'_{1},...,u'_{i-1}, E \right] \ge \frac{\epsilon_{\mathrm{A}\ref{assum_sufficiently_large_error}}}{2}.
        \]
        Therefore, we can conclude that conditional on event $E$, $\sum_{i=T_0}^{T_s-1} X_i$ is stochastically dominated by $\sum_{i=T_0}^{T_s-1} Y_i$, where $Y_i$ are i.i.d. Bernoulli random variables that are equal to $1$ with probability $\epsilon_{\mathrm{A}\ref{assum_sufficiently_large_error}}/2$. By this coupling argument and Hoeffding's inequality, for $s \ge 1$, conditional on event $E$ with conditional probability $1-o_T(1/T^2)$,
        \begin{equation}\label{eq:st_lower}
            |S_{T_s}| = \sum_{i=T_0}^{T_s - 1} X_i \ge \sum_{i=T_0}^{T_s-1} Y_i \ge \frac{\epsilon_{\mathrm{A}\ref{assum_sufficiently_large_error}}}{2} (T_s - T_0) - \log(T)\sqrt{T_s - T_0} \ge \frac{\epsilon_{\mathrm{A}\ref{assum_sufficiently_large_error}}}{4} \cdot (T_s-T_0) \ge \frac{\epsilon_{\mathrm{A}\ref{assum_sufficiently_large_error}}}{8} \cdot T_s,
        \end{equation}
        where the second to last inequality comes from for sufficiently large $T$ and $s \ge 1$, $T_s-T_0 \ge \sqrt{T}$ and therefore $\sqrt{T_s - T_0} \ge \frac{4\log(T)}{\epsilon_{\mathrm{A}\ref{assum_sufficiently_large_error}}}$. The last inequality comes from the fact that $T_s - T_0 \ge \frac{T_s}{2}$ by the definition of $T_s$ for $s \ge 1$. A union bound over all $s \in [1:s_e]$ gives that conditional on event $E$ with conditional probability $1-o_T(1/T)$,
        \[
            \min_{s \in [1:s_e]} \frac{|S_{T_s}|}{T_s} \ge \frac{\epsilon_{\mathrm{A}\ref{assum_sufficiently_large_error}}}{8}.
        \]
    \end{proof}

\subsection{Proof of Lemma \ref{safe_condition_P}}\label{proof:safe_condition_P}
\begin{proof}

    Defining $P_s(\hat{\theta}_s, \epsilon_s)$ as 
    \[
        P_s(\hat{\theta}_s, \epsilon_s) =  \min_{\norm{\theta - \hat{\theta}^{\mathrm{pre}}_s}_{\infty} \le \epsilon_s } P(\theta, K_{\mathrm{opt}}(\hat{\theta}_s, T_s), D_{\mathrm{U}}),
    \]
    we have by definition of $u_t^{\mathrm{safeU}}$ in Algorithm \eqref{alg:cap_large} that    \begin{equation}\label{eq:sufficiently_large_errord}
     \{x'_t \ge P_s(\hat{\theta}_s, \epsilon_s) \} = \left\{x_t' \ge \min_{\norm{\theta - \hat{\theta}^{\mathrm{pre}}_s}_{\infty} \le \epsilon_s } P(\theta, K_{\mathrm{opt}}(\hat{\theta}_s, T_s), D_{\mathrm{U}})\right\} \subseteq \{u_t' = u_t^{\mathrm{safeU}}\}.
    \end{equation}
    Under event $E$, $\norm{\theta^* - \hat{\theta}^{\mathrm{pre}}_s}_{\infty} \le \epsilon_s $, therefore
    \begin{equation}\label{eq:sufficiently_large_errorc}
     \min_{\norm{\theta - \hat{\theta}^{\mathrm{pre}}_s}_{\infty} \le \epsilon_s } P(\theta, K_{\mathrm{opt}}(\theta^*, T_s), D_{\mathrm{U}}) \le P(\theta^*, K_{\mathrm{opt}}(\theta^*,T_s), D_{\mathrm{U}}).
    \end{equation}
    Therefore, we can conclude that conditional on event $E$,
    \begin{align*}
         P_s(\hat{\theta}_s, \epsilon_s) &= \min_{\norm{\theta - \hat{\theta}^{\mathrm{pre}}_s}_{\infty} \le \epsilon_s } P(\theta, K_{\mathrm{opt}}(\hat{\theta}_s, T_s), D_{\mathrm{U}}) \\
         &\le \min_{\norm{\theta - \hat{\theta}^{\mathrm{pre}}_s}_{\infty} \le \epsilon_s } P(\theta, K_{\mathrm{opt}}(\theta^*, T_s), D_{\mathrm{U}}) && \text{Choice of $\hat{\theta}_s$}\\
         &\le P(\theta^*, K_{\mathrm{opt}}(\theta^*,T_s), D_{\mathrm{U}}). && \text{Equation \eqref{eq:sufficiently_large_errorc}}
    \end{align*}
\end{proof}

\section{Uncertainty Bounds}

\subsection{Tools for Uncertainty Bounds}

The proofs of uncertainty bounds will rely on the following result from \cite{abbasi2011regret}.
\begin{lemma}[Derived from Theorem 1 in \cite{abbasi2011regret}]\label{v_to_use}
    Suppose $x_t$ and $u_t$ are respectively the position and control at time $t$ when using an arbitrary controller $C$ starting at position $x_0 = 0$. Define $z_t = (x_t,u_t)$ and let $\lambda > 0$. Let $Z_t \in \mathbb{R}^{t \times 2}$ where the $i$th row is $z_{i-1}$, let $X_t \in \mathbb{R}^{t \times 1}$ where the $i$th element is $x_{i}$, and let $I \in \mathbb{R}^{2 \times 2}$ be the identity matrix.  Then under Assumptions \ref{assum_init}--\ref{problem_specifications}, with probability $1-o_T\left(\frac{1}{T^2}\right)$ the following holds for all $1 \le t \le T-1$ and for any $S \subseteq [0:t-1]$:
    \begin{equation}\label{eq:v_to_use}
        \norm{\theta^* - (Z_{t}^{\top}Z_{t}+\lambda I)^{-1}Z_{t}^{\top}X_{t}}_{\infty} \le \sqrt{\frac{\max((V_t^S)_{11}, (V_t^S)_{22})}{\det(V_t^S)}}B_t,
    \end{equation}
    where $V_t^S = \lambda I + \sum_{s=0}^{t-1}z_sz_s^\top 1_{s \in S}$, $B_t = \alpha\sqrt{\log\left(\det\left(
    V_t^{[0:t-1]}
    \right)\right) + \log(\lambda^2) + 2\log(T^2)} + \sqrt{\lambda}(\bar{a}^2 + \bar{b}^2)$, and $\alpha$ is from the subgaussian assumption on the noise distribution $\mathcal{D}$, which implies that there exists an $\alpha$ such that $\E_{w \sim \mathcal{D}}[\exp(\gamma w)] \le \exp(\gamma^2\alpha^2/2)$ for any $\gamma \in \mathbb{R}$.
\end{lemma}

Lemma \ref{v_to_use} can be directly derived from Theorem 1 in \cite{abbasi2011regret} as shown in Appendix \ref{sec:proof_of_v_to_use}. The other tool that will be shared by the proofs in the following sections is the following anti-concentration inequality of the sum of non-negative random variables.
\begin{lemma}\label{sum_of_depend_rand}
    For $p \in (0,1]$ and $1 \le n \le T$, suppose $X_0,...,X_{n-1}$ are non-negative random variables such that $(X_0,...,X_{i-1})$ is a deterministic function of the random variable set $F_i$ for all $i \in [1:n]$ and $F_i \subseteq F_{i+1}$. Let the set $S_n \subseteq [0:n-1]$ be a random variable such that the event $\{i \in S_n\}$ is a deterministic function of $F_{i+1}$. For $i \in [0:n-1]$, define $S_i = \{k < i : k \in S_n\}$, therefore $S_i$ is a deterministic function of $F_i$. Let $E^*$ be an event such that for all $i \in [0:n-1]$,
    \begin{equation}\label{eq:condxdef_rand}
        \E\left[X_i \cond F_i, E^*, i \in S_n\right]  \ge c \cdot |S_{i}|,
    \end{equation}
    where $c > 0$ is non-random. Furthermore, assume that conditional on $E^*$:
\begin{equation}\label{eq:boundxdef_rand}
        0  \stackrel{\text{a.s.}}{\le} X_i  \stackrel{\text{a.s.}}{\le} \frac{c|S_i|}{2p}.
    \end{equation}
    Then conditional on event $E^*$, with conditional probability $1-o_T(1/T^2)$, 
    \[
        \sum_{i=0}^{n-1} X_i \ge \frac{c}{4}\left(\max(\lfloor p|S_n| - \log(T)\sqrt{|S_n|} \rfloor, 1)\right)\left(\max(\lfloor p|S_n| - \log(T)\sqrt{|S_n|} \rfloor, 1)-1\right).
    \]
\end{lemma}

The proof of Lemma \ref{sum_of_depend_rand} can be found in Appendix \ref{sec:proof_of_sum_of_depend_rand}.

\subsection{Proof of Lemma \ref{initial_uncertainty}}\label{proof:initial_uncertainty}
\begin{proof}
For the rest of this proof, $x_t$ and $u_t$ are respectively the position and control at time $t$ of controller $C^{\mathrm{alg}}$ that corresponds to Algorithm \ref{alg:cap} starting at $x_0 = 0$. Recall that Lemma \ref{initial_uncertainty} was stated and used in Appendix \ref{app:proof_of_performance} with respect to Algorithm \ref{alg:cap}, therefore all events and variables in this subsection refer to those defined with respect to Algorithm \ref{alg:cap}. To prove Lemma \ref{initial_uncertainty}, we will use Lemma \ref{v_to_use} applied to $S = [0:\frac{1}{\nu_T^2}-1]$. The goal will be to bound the right side of Equation \eqref{eq:v_to_use} for this choice of $S$. Consider any fixed arbitrary $s \in [0:s_e]$ and the corresponding matrix $V_{T_s}^S$. Define $N$ as the event that for all $i < 1/\nu_T^2$, the control $u_i$ is safe for dynamics $\theta^*$. Note that we showed in Lemma \ref{safety_append} that $\P(N) \ge \P(E_{\mathrm{safe}}) = 1-o_T(\frac{1}{T^2})$. Under event $N \cap E_1$, we can apply Lemma \ref{bounded_approx} to get the following equations for sufficiently large $T$:
\begin{equation}\label{eq:initial_uncertainty_V1}
    (V_{T_s}^S)_{11} = \lambda + \sum_{i=0}^{\frac{1}{\nu_T^2}-1} x_i^2 \le \lambda + \frac{1}{\nu_T^2}B_x^2 
\end{equation}
\begin{align*}
    (V_{T_s}^{S})_{22} &= \lambda + \sum_{i=0}^{\frac{1}{\nu_T^2}-1} u_i^2
    \le \lambda + \frac{1}{\nu_T^2}B_x^2 \numberthis \label{eq:initial_uncertainty_V2}
\end{align*} 
\begin{align*}
    (V_{T_s}^{S})_{12}^2 &= \left(\sum_{i=0}^{\frac{1}{\nu_T^2}-1}  u_ix_i\right)^2.  \numberthis \label{eq:initial_uncertainty_V3}
\end{align*}
We can now compute $(V_{T_s}^{S})_{22}(V_{T_s}^S)_{11} - (V_{T_s}^{S})_{12}^2$. Recall that for the first $1/\nu_T^2$ steps of Algorithm \ref{alg:cap}, the control is $u_i = C^{\mathrm{init}}(x_i) + \frac{\phi_i}{\log(T)}$ where $\phi_i$ is i.i.d. from the Rademacher distribution and independent from the noise random variables.
\begin{align*}
    &(V_{T_s}^{S})_{22}(V_{T_s}^S)_{11} - (V_{T_s}^{S})_{12}^2 \\
    &= \left( \lambda + \sum_{i=0}^{\frac{1}{\nu_T^2}-1}  u_i^2 \right)\left( \lambda + \sum_{i=0}^{\frac{1}{\nu_T^2}-1}  x_i^2\right) - \left(\sum_{i=0}^{\frac{1}{\nu_T^2}-1}  u_ix_i\right)^2 && \text{Equations \eqref{eq:initial_uncertainty_V1} \eqref{eq:initial_uncertainty_V2} \eqref{eq:initial_uncertainty_V3}}\\
    &\ge  \left(  \sum_{i=0}^{\frac{1}{\nu_T^2}-1}  u_i^2\right)\left( \sum_{i=0}^{\frac{1}{\nu_T^2}-1}  x_i^2\right) - \left(\sum_{i=0}^{\frac{1}{\nu_T^2}-1}  u_ix_i\right)^2  \\
    &= \sum_{i < j}^{\frac{1}{\nu_T^2}-1} (u_ix_j - u_jx_i)^2 \\
    &= \sum_{i < j}^{\frac{1}{\nu_T^2}-1} \left(u_ix_j -  C^{\mathrm{init}}(x_j)x_i +\frac{\phi_j}{\log(T)}x_i\right)^2. \numberthis \label{eq:diff_of_v1}
\end{align*}
Conditional on $N \cap E_1$, for all $i < 1/\nu_T^2$, we have $|u_i|, |x_i| \le B_x$ by Lemma \ref{bounded_approx}. Define the random variable $X_j$ as
\begin{align*}
    X_j &= \sum_{i =0}^{j-1} (u_ix_j - u_jx_i)^2 \\
    &= \sum_{i=0}^{j-1} \left(u_ix_j - C^{\mathrm{init}}(x_j)x_i + \left(\frac{\phi_j}{\log(T)}\right) x_i\right)^2 \\
    &\le 4jB_x^4. && \text{Conditional on $N \cap E_1$ by Lemma \ref{bounded_approx}} \numberthis \label{eq:Xj_Bounds}
\end{align*}
We will use the following lemma to lower bound the conditional expectation of $X_j$.
\begin{lemma}\label{logt_xvals}
    Under Assumptions \ref{assum_init}--\ref{problem_specifications}, let $x_0,x_1,...,x_T$ be the positions of the controller $C^{\mathrm{alg}}$ starting at $x_0 = 0$. Then there exists an event $E_{L\ref{logt_xvals}}$ such that $\P(E_{L\ref{logt_xvals}}) = 1-o_T(1/T^2)$ and for sufficiently large $T$ conditional on $E_{L\ref{logt_xvals}}$, for all $j \ge \log^8(T)$, 
    \begin{equation}\label{eq:logt_xvals}
        \sum_{i=0}^{j-1} x_i^2 \ge \frac{j}{2\log^2(T)}.
    \end{equation}
\end{lemma}
The proof of Lemma \ref{logt_xvals} can be found in Appendix \ref{sec:proof_of_logt_xvals}. Now define $E^* = N \cap E_1 \cap E_{L\ref{logt_xvals}}$. Note that $\P(E^*) = 1-o_T(1/T^2)$ by a union bound. Because $\phi_j$ is a Rademacher random variable, we therefore have that $\P(\phi_j = 1 \mid E^*) = 1/2 - o_T(1/T^2)$ and $\P(\phi_j = -1 \mid E^*) = 1/2 - o_T(1/T^2)$. This implies that $|\E[\phi_j \mid E^*]| = o_T(1/T^2)$, and therefore for sufficiently large $T$, we have $\Var\left[\phi_j\cond E^*\right] \ge 1/2$. Then we can bound the conditional expectation of $X_j$ under event $E^*$ as follows for all $j \ge \log^8(T)$ and for sufficiently large $T$. Define $F_j = \{x_0,u_0,...,x_{j-1},u_{j-1}, x_j\}$. Then we have
\begin{align*}
    \E[X_j \mid F_j, E^*] &= \sum_{i=0}^{j-1} \E\left[\left(u_ix_j - C^{\mathrm{init}}(x_j)x_i + \left(\frac{\phi_j}{\log(T)}\right) x_i\right)^2 \cond F_j, E^*\right] \\
    &\ge \sum_{i=0}^{j-1} \Var\left[u_ix_j - C^{\mathrm{init}}(x_j)x_i + \left(\frac{\phi_j}{\log(T)}\right) x_i \cond F_j, E^*\right] \\
    &= \sum_{i=0}^{j-1} x_i^2\Var\left[\frac{\phi_j}{\log(T)}\cond F_j, E^*\right] \\\
    &= \sum_{i=0}^{j-1} x_i^2\Var\left[\frac{\phi_j}{\log(T)} \cond E^*\right]&& \text{$\phi_j$ is ind. of $F_j$}\\
    &= \frac{\Var\left[\phi_j\cond E^*\right]}{\log^2(T)}\sum_{i=0}^{j-1} x_i^2 \\
    &\ge \frac{1}{2\log^2(T)}\sum_{i=0}^{j-1} x_i^2  \\
    &\ge \frac{j}{4\log^4(T)}. && \text{$E^* \subseteq E_{\mathrm{L}\ref{logt_xvals}}$}\numberthis \label{eq:conditional_bound_on_X}
\end{align*}
Therefore, we can apply Lemma \ref{sum_of_depend_rand} to $X_{\log^8(T)},X_{\log^8(T)+1},...,X_{1/\nu_T^2 - 1}$ with $n = 1/\nu_T^2 - \log^8(T)$, $p = \frac{1}{32B_x^4\log^4(T)}$, $F_i = \{x_0,u_0,...,u_{i-1}, x_i\}$, $S_n = [0: n-1]$, and $c = \frac{1}{4\log^4(T)}$. Note that this choice of $p$ is less than $1$ for sufficiently large $T$.

We will also use that for sufficiently large $T$, $n = 1/\nu_T^2 - \log^8(T) = T^{2/3} - \log^8(T) \ge 4\log^2(T)/p^2$. This implies that for sufficiently large $T$,
\begin{equation}\label{eq:pn_conversion}
   p n - \log(T)\sqrt{n} \ge pn/2 = \frac{1/\nu_T^2 - \log^8(T)}{64B_x^4\log^4(T)}\ge 1.
\end{equation}

Recall by Equation \eqref{eq:Xj_Bounds} that under event $E^*$, the $X_j$ are bounded by $0 \le X_j \le 4jB_x^4 = \frac{c}{2p} \cdot j$.  Lemma \ref{sum_of_depend_rand} gives that for sufficiently large $T$ and conditional on event $E^*$ with conditional probability $1-o_T(1/T^2)$, 
\begin{align*}
   &(V_{T_s}^{S})_{22}(V_{T_s}^S)_{11} - (V_{T_s}^{S})_{12}^2 \\
   &\ge  \sum_{j=0}^{\frac{1}{\nu_T^2}-1} X_j && \text{Equation \eqref{eq:diff_of_v1}} \\
   &\ge  \sum_{j=\log^8(T)}^{\frac{1}{\nu_T^2}-1} X_j && \text{$X_i \ge 0$}\\
   &\ge \frac{1}{16\log^4(T)} \left(\max(\lfloor p n - \log(T)\sqrt{n}\rfloor, 1) - 1\right)\left(\max(\lfloor p n - \log(T)\sqrt{n}\rfloor, 1) \right) && \text{Lemma \ref{sum_of_depend_rand}} \\
   &\ge \frac{1}{16\log^4(T)} \left(\left\lfloor \frac{1/\nu_T^2 - \log^8(T)}{64B_x^4\log^4(T)} \right\rfloor -  1\right) \left(\left\lfloor \frac{1/\nu_T^2 - \log^8(T)}{64B_x^4\log^4(T)} \right\rfloor\right)&& \text{Equation \eqref{eq:pn_conversion}} \\
    &= \Omega_T\left(\frac{\frac{1}{\nu_T^4}}{B_x^8\log^{12}(T)}\right). \numberthis \label{eq:vden1}
\end{align*}
Finally, we need to bound the quantity $B_{T_s}$ from Lemma \ref{v_to_use}. The only non-constant term in $B_{T_s}$ is $\sqrt{\log(\det(\lambda I + \sum_{i=0}^{T_s-1} z_iz_i^\top)) + 2\log(T^2)}$. Define $V_{T_s} = \lambda I + \sum_{i=0}^{T_s-1} z_iz_i^\top$. Conditional on event $N \cap E_1$, we have by Lemma \ref{bounded_approx} that $(V_{T_s})_{22} \le \lambda + TB_x^2$ and $(V_{T_s})_{11} \le \lambda + TB_x^2$. Therefore, conditional on event $N \cap E_1$, 
\begin{align*}
   \sqrt{\log\left(\det\left(\lambda I + \sum_{i=0}^{T_s-1} z_iz_i^\top\right)\right) + 2\log(T^2)} &\le \sqrt{\log((V_{T_s})_{11}(V_{T_s})_{22})+ 2\log(T^2)} \\
    &\le \sqrt{\log\left(\left(\lambda + TB_x^2\right)^2\right)+ 2\log(T^2)} \\
    &= \tilde{O}_T(1). \numberthis \label{eq:bounded_bt}
\end{align*}
Now, combining Lemma \ref{v_to_use} and Equations \eqref{eq:initial_uncertainty_V1}, \eqref{eq:initial_uncertainty_V2}, \eqref{eq:vden1}, and \eqref{eq:bounded_bt} gives that conditional on event $E^*$ with conditional probability $1-o_T(1/T^2)$, for all $s \in [0:s_e]$,
\begin{align*}
    \epsilon_s^2 \le \frac{\max((V_{T_s}^S)_{11}, (V_{T_s}^S)_{22})}{\det(V_{T_s}^S)}B^2_{T_s} = \frac{(\lambda + \left(\frac{1}{\nu_T^2}\right)B_x^2)\tilde{O}_T(1) }{ \Omega_T\left(\frac{\left(\frac{1}{\nu_T^2}\right)^2}{B_x^8\log^{12}(T)}\right)} = \tilde{O}_T\left(\nu_T^2\right).
\end{align*}
Because $\P(E^*) = 1-o_T(1/T^2)$, this gives the desired result with unconditional probability $1-o_T(1/T^2)$.
\end{proof}

\subsection{Proof of Lemma \ref{boundary_uncertainty}}\label{sec:proof_of_boundary_uncertainty}
\begin{proof}
    Recall that Lemma \ref{boundary_uncertainty} was stated and used in Appendix \ref{app:suff_large_noise_case} with respect to Algorithm \ref{alg:cap_large}, therefore all events and variables in this subsection refer to those defined with respect to Algorithm \ref{alg:cap_large}. 
We will prove a more general result in Lemma \ref{boundary_uncertainty_cont}.

\begin{lemma}\label{boundary_uncertainty_cont}
    Let $x_t, u_t$ respectively be the position and control of $C^{\mathrm{alg}}$ (the controller of Algorithm \ref{alg:cap_large}) at time $t$ starting at $x_0 = 0$. Define $G_i = (x_0,u_0,...,x_{i-1}, u_{i-1})$. For constant $\gamma > 0$, define $S_t'$ as 
    \begin{equation}
        S_t' = \Big\{i< t : u_i = u_i^{\mathrm{safeU}} \text{ $\mathrm{and}$ } \P(u_i = u_i^{\mathrm{safeU}} \mid G_i, E) \ge \gamma \Big\},
    \end{equation}
    where $E$ is the event defined in Appendix \ref{app:suff_large_noise_case}.
    Then under Assumptions \ref{assum_init}--\ref{parameterization_assum3} and for sufficiently large $T$, with probability $1-o_T(1/T)$, 
    \[
        \max_{s \in [0:s_e]} \epsilon_s \sqrt{|S'_{T_s}|} = \tilde{O}_T\left(1\right),
    \]
    where $\epsilon_s$ is from Algorithm \ref{alg:cap_large}.
\end{lemma}
The proof of Lemma \ref{boundary_uncertainty_cont} can be found in Appendix \ref{sec:proof_of_boundary_uncertainty_cont}. We will now prove that Lemma \ref{boundary_uncertainty} is a direct consequence of Lemma \ref{boundary_uncertainty_cont}. By Equation \eqref{eq:prob_with_E2}, we have that for all $i$,
\[
    \P\left(u_i = u_i^{\mathrm{safeU}} \cond G_i, E\right) \ge \frac{\epsilon_{\mathrm{A}\ref{assum_sufficiently_large_error}}}{2}.
\]

Therefore, we have that
\begin{align*}
    \Big\{i< t : u_i = u_i^{\mathrm{safeU}} \text{ $\mathrm{and}$ } \P(u_i = u_i^{\mathrm{safeU}} \mid G_i, E) \ge \frac{\epsilon_{\mathrm{A}\ref{assum_sufficiently_large_error}}}{2} \Big\}  &= \Big\{i< t : u_i = u_i^{\mathrm{safeU}}\Big\}. \numberthis \label{eq:safe_sprime}
\end{align*}
Lemma \ref{boundary_uncertainty_cont} for $\gamma = \frac{\epsilon_{\mathrm{A}\ref{assum_sufficiently_large_error}}}{2}$ gives that with probability $1-o_T(1/T)$,
\begin{align*}
     \max_{s \in [0:s_e]}  \epsilon_s \cdot \sqrt{\left|\left\{i< t : u_i = u_i^{\mathrm{safeU}} \text{ $\mathrm{and}$ } \P(u_i = u_i^{\mathrm{safeU}} \mid G_i, E) \ge \frac{\epsilon_{\mathrm{A}\ref{assum_sufficiently_large_error}}}{2} \right\}\right|} = \tilde{O}_T\left(1\right). \numberthis \label{eq:boun_sprime}
\end{align*}
Combining Equation \eqref{eq:safe_sprime} and Equation \eqref{eq:boun_sprime} gives that with probability $1-o_T(1/T)$,
\[
 \max_{s \in [0:s_e]} \epsilon_s \sqrt{\left|\left\{i< t : u_i = u_i^{\mathrm{safeU}} \right\}\right|} = \tilde{O}_T\left(1\right),
\]
which is the desired result of Lemma \ref{boundary_uncertainty}.
\end{proof}

\subsection{Proof of Lemma \ref{boundary_uncertainty_cont}}\label{sec:proof_of_boundary_uncertainty_cont}

Lemma \ref{boundary_uncertainty_cont} is stated above to be used in Appendix \ref{app:suff_large_noise_case} with respect to Algorithm \ref{alg:cap_large}, therefore all events and variables in this subsection refer to those defined for Algorithm \ref{alg:cap_large}. 
\begin{proof}

The first step of the proof will be to prove that conditional on event $E$ for all $i \ge T_0$,
\begin{equation}\label{rewrite_usafe}
    u_i^{\mathrm{safeU}} = -\frac{a^*}{b^*}x_i + \frac{D_{\mathrm{U}}+ e_i}{b^*},
\end{equation}
where $|e_i| = \tilde{O}_T(\nu_T)$.  Let $s_i = \lfloor\log_2(i\nu_T^2) \rfloor$. Recall that $u_i^{\mathrm{safeU}}$ is the largest $u$ such that 
    \[
        \max_{\norm{\theta - \hat{\theta}_{s_i}}_{\infty} \le \epsilon_{s_i}} ax_i + bu \le D_{\mathrm{U}}.
    \]
    For sufficiently large $T$ and conditional on event $E$,
    \[
        \epsilon_{s_i} = \tilde{O}_T(\nu_T) \le \min(a^*,b^*) - \epsilon_{s_i} \le \min(\hat{a}_{s_i}, \hat{b}_{s_i}).
    \]
    This implies that $\hat{a}_{s_i} - \epsilon_{s_i} \ge 0$, giving the following equations of casework for $u_i^{\mathrm{safeU}}$:
    \begin{equation}
       u_i^{\mathrm{safeU}} = 
           \begin{cases}
                 \frac{D_{\mathrm{U}}- (\hat{a}_{s_i} + \epsilon_{s_i})x_i}{\hat{b}_{s_i} - \epsilon_{s_i}}, & \text{if  $x_i \ge 0$ and $(\hat{a}_{s_i} + \epsilon_{s_i})x_i \ge D_{\mathrm{U}}$ }\\
                 \frac{D_{\mathrm{U}}- (\hat{a}_{s_i} + \epsilon_{s_i})x_i}{\hat{b}_{s_i} + \epsilon_{s_i}}, & \text{if  $x_i \ge 0$ and $(\hat{a}_{s_i} + \epsilon_{s_i})x_i \le D_{\mathrm{U}}$ }\\
                 \frac{D_{\mathrm{U}}- (\hat{a}_{s_i} - \epsilon_{s_i})x_i}{\hat{b}_{s_i} + \epsilon_{s_i}}, & \text{if  $x_i \le 0$ }\\
    \end{cases}
    \end{equation}
    which implies
    \begin{equation}
       u_i^{\mathrm{safeU}} = 
           \begin{cases}
                \frac{D_{\mathrm{U}}-a^*x_i}{b^*}\frac{b^*}{\hat{b}_{s_i} - \epsilon_{s_i}} + \frac{a^*x_i- (\hat{a}_{s_i} + \epsilon_{s_i})x_i}{\hat{b}_{s_i} -\epsilon_{s_i}}, & \text{if  $x_i \ge 0$ and $(\hat{a}_{s_i} + \epsilon_{s_i})x_i \ge D_{\mathrm{U}}$ }\\
                 \frac{D_{\mathrm{U}}-a^*x_i}{b^*}\frac{b^*}{\hat{b}_{s_i} + \epsilon_{s_i}} + \frac{a^*x_i- (\hat{a}_{s_i} + \epsilon_{s_i})x_i}{\hat{b}_{s_i} +\epsilon_{s_i}}, & \text{if  $x_i \ge 0$ and $(\hat{a}_{s_i} + \epsilon_{s_i})x_i \le D_{\mathrm{U}}$ }\\
                  \frac{D_{\mathrm{U}}-a^*x_i}{b^*}\frac{b^*}{\hat{b}_{s_i} + \epsilon_{s_i}} + \frac{a^*x_i- (\hat{a}_{s_i} - \epsilon_{s_i})x_i}{\hat{b}_{s_i} +\epsilon_{s_i}}, & \text{if  $x_i \le 0$ }\\
    \end{cases}
    \end{equation}
    which implies
    \begin{equation}
       u_i^{\mathrm{safeU}} = 
           \begin{cases}
                \frac{D_{\mathrm{U}}-a^*x_i}{b^*} + \frac{b^* - \hat{b}_{s_i} + \epsilon_{s_i}}{\hat{b}_{s_i} - \epsilon_{s_i}} \frac{D_{\mathrm{U}}-a^*x_i}{b^*} + \frac{(a^*- \hat{a}_{s_i} - \epsilon_{s_i})x_i}{\hat{b}_{s_i} -\epsilon_{s_i}}, & \text{if  $x_i \ge 0$ and $(\hat{a}_{s_i} + \epsilon_{s_i})x_i \ge D_{\mathrm{U}}$ }\\
                 \frac{D_{\mathrm{U}}-a^*x_i}{b^*}  +\frac{b^* - \hat{b}_{s_i} - \epsilon_{s_i}}{\hat{b}_{s_i} + \epsilon_{s_i}}  \frac{D_{\mathrm{U}}-a^*x_i}{b^*} + \frac{(a^*- \hat{a}_{s_i} - \epsilon_{s_i})x_i}{\hat{b}_{s_i} +\epsilon_{s_i}}, & \text{if  $x_i \ge 0$ and $(\hat{a}_{s_i} + \epsilon_{s_i})x_i \le D_{\mathrm{U}}$ }\\
                  \frac{D_{\mathrm{U}}-a^*x_i}{b^*} +\frac{b^* - \hat{b}_{s_i} - \epsilon_{s_i}}{\hat{b}_{s_i} + \epsilon_{s_i}}  \frac{D_{\mathrm{U}}-a^*x_i}{b^*} + \frac{(a^*- \hat{a}_{s_i} + \epsilon_{s_i})x_i}{\hat{b}_{s_i} +\epsilon_{s_i}}. & \text{if  $x_i \le 0$. }\\
    \end{cases}
    \end{equation}
    Under event $E$, $|a^* - \hat{a}_{s_i}| \le \epsilon_{s_i}$, $|b^* - \hat{b}_{s_i}| \le \epsilon_{s_i}$, and $|x_i| = \tilde{O}_T(1)$, therefore in all three cases we have that $u_i^{\mathrm{safeU}} = -\frac{a^*}{b^*}x_i + \frac{D_{\mathrm{U}}+ e_i}{b^*}$,
    for some $e_i$ satisfying

    \begin{equation}\label{eq:ei}
        |e_i| = \tilde{O}_T(\epsilon_{s_i}) = \tilde{O}_T(\nu_T).
    \end{equation}
            We now  define
            \[
                S''_t = \Big\{i< t : u_i = u_i^{\mathrm{safeU}} \text{ $\mathrm{and}$ } \P(u_i = u_i^{\mathrm{safeU}} \mid G_i, E) \ge \gamma \text{ and } \P(E \mid G_i) \ge \frac{1}{2} \Big\}.
            \]
            \begin{lemma}\label{lemma:converting_to_s_prime}
                Using the same notation and assumptions as in the proof of Lemma \ref{boundary_uncertainty_cont}, for any constant $c < 1$,
                \[
                    \P \Big(\forall i \in [0:t-1], \P(E \mid G_i) \ge c \Big) = 1- o_T(1/T).
                \]
            \end{lemma}
            \begin{proof}
                Consider any fixed $i \in [0:t-1]$. We will show that $\P\Big(\P(E \mid G_i)\ge c\Big) = 1-o_T(1/T^2)$. Suppose this is not true, i.e. suppose that $\P\Big(\P(E \mid G_i) \ge c\Big) = 1 - \Omega_T(1/T^2)$, or equivalently that $\P\Big(\P(\neg E \mid G_i) \ge 1-c\Big) =  \Omega_T(1/T^2)$. Note that by the law of total expectation,
                \begin{align*}
                    \P\Big(\neg E \mid \P(\neg E \mid G_i) \ge 1-c\Big) & = \E \Big[\P\Big(\neg E \mid G_i, \P(\neg E \mid G_i) \ge 1-c\Big) \mid \P(\neg E \mid G_i) \ge 1-c\Big] \\
                    &\ge \E\left[1-c\right] \\
                    &= 1-c.
                \end{align*}
                    
                This implies that 
                \[
                    \P(\neg E) = \P\Big(\neg E \mid \P(\neg E \mid G_i) \ge 1-c\Big)\P\Big(\P(\neg E \mid G_i) \ge 1-c\Big) = (1-c)\Omega_T(1/T^2).
                \]
                This would then imply that $\P(E) = 1- \P(\neg E) = 1 - \Omega_T(1/T^2)$, which is a contradiction with the fact that $\P(E) = 1-o_T(1/T^2)$. Therefore, we must have that for all fixed $i$, 
                \[
                    \P\Big(\P(E \mid G_i)\ge c\Big) = 1-o_T(1/T^2).
                \]
                Taking a union bound gives that
                \begin{align*}
                    \P\Big( \forall i \in [0:t-1],  \P(E \mid G_i)\ge c\Big)  &\ge 1- \sum_{i=0}^{t-1} \left(1 - \P(\P(E \mid G_i) \ge c) \right)
                    = 1 - o_T(1/T),
                \end{align*}
                which is exactly what we want to show.
            \end{proof}

            If $\forall i \in [0:t-1]$, $\P(E \mid G_i) \ge 1/2$, then $|S'_t| = |S''_t|$. Using Lemma \ref{lemma:converting_to_s_prime} with $c = 1/2$,
            \begin{equation}\label{eq:S'_equal_S}
               \P \left( \left| S'_t \right| = \left| \Sprime_t \right| \right) \ge \P(\forall i \in [0:t-1], \P(E \mid G_i) \ge 1/2) = 1- o_T(1/T).
            \end{equation}
            Therefore, if we can show that with probability $1-o_T(1/T)$,
            \begin{equation}\label{eq:Sprime_to_show}
                    \max_{s \in [0:s_e]} \epsilon_s \sqrt{|S''_{T_s}|} = \tilde{O}_T\left(1\right),
            \end{equation}
            then a union bound combining Equation \eqref{eq:Sprime_to_show} with Equation \eqref{eq:S'_equal_S} gives that with probability $1-o_T(1/T)$,
            \[
        \max_{s \in [0:s_e]} \epsilon_s \sqrt{|S'_{T_s}|} = \tilde{O}_T\left(1\right),
        \]
        which is our desired result. Therefore, the rest of this proof will focus on proving Equation \eqref{eq:Sprime_to_show}. 

Fix any $s \in [0: s_e]$. We will use Lemma \ref{v_to_use} with $S = \Sprime_{T_s}$. Under event $E$, we have by  Lemma \ref{bounded_approx} the following three equations:
\begin{equation}\label{eq:num1}
    (V_{T_s}^{\Sprime_{T_s}})_{11} = \lambda + \sum_{i=0}^{T_s - 1} x_i^21_{i \in \Sprime_{T_s}} \le \lambda + |\Sprime_{T_s}|B_x^2
\end{equation}
\begin{equation}\label{eq:num2}
    (V_{T_s}^{\Sprime_{T_s}})_{22} = \lambda + \sum_{i=0}^{T_s - 1} u_i^2 1_{i \in \Sprime_{T_s}} \le \lambda + |\Sprime_{T_s}|B_x^2
\end{equation}
\begin{equation}
    (V_{T_s}^{\Sprime_{T_s}})_{12}^2 = \left(\sum_{i=0}^{T_s - 1} u_ix_i1_{i \in \Sprime_{T_s}}\right)^2.
\end{equation}
We can now lower bound $(V_{T_s}^{\Sprime_{T_s}})_{22}(V_{T_s}^{\Sprime_{T_s}})_{11} - (V_{T_s}^{\Sprime_{T_s}})_{12}^2$ for sufficiently large $T$ conditional on event $E$.
\begin{align*}
    &(V_{T_s}^{\Sprime_{T_s}})_{22}(V_{T_s}^{\Sprime_{T_s}})_{11} - (V_{T_s}^{\Sprime_{T_s}})_{12}^2 \\
    &= \left( \lambda + \sum_{i=0}^{T_s - 1} u_i^2 1_{i \in \Sprime_{T_s}}\right)\left( \lambda + \sum_{i=0}^{T_s - 1} x_i^21_{i \in \Sprime_{T_s}}\right) - \left( \left(\sum_{i=0}^{T_s - 1} u_ix_i1_{i \in \Sprime_{T_s}}\right)^2 \right) \\
    &\ge  \left(  \sum_{i=0}^{T_s - 1} u_i^2 1_{i \in \Sprime_{T_s}}\right)\left( \sum_{i=0}^{T_s - 1} x_i^21_{i \in \Sprime_{T_s}}\right) - \left( \left(\sum_{i=0}^{T_s - 1} u_ix_i1_{i \in \Sprime_{T_s}}\right)^2 \right) \\
    &= \sum_{i < j}^{T_s - 1} (u_ix_j - u_jx_i)^21_{i,j \in \Sprime_{T_s}} \\
    &= \sum_{i < j}^{T_s - 1} \left(\left(-\frac{a^*}{b^*}x_i + \frac{D_{\mathrm{U}}+ e_i}{b^*}\right)x_j - \left(-\frac{a^*}{b^*}x_j + \frac{D_{\mathrm{U}}+ e_j}{b^*}\right) x_i\right)^21_{i,j \in \Sprime_{T_s}} && \text{Equation \eqref{rewrite_usafe}} \\
    &= \frac{1}{(b^*)^2}\sum_{i < j}^{T_s - 1}  \left(D_{\mathrm{U}}x_j + e_ix_j - D_{\mathrm{U}}x_i - e_jx_i\right)^21_{i,j \in \Sprime_{T_s}} \\
    &= \frac{1}{(b^*)^2}\sum_{i < j}^{T_s - 1}  \left(x_j(D_{\mathrm{U}}+ e_i) - (D_{\mathrm{U}}+ e_j)x_i\right)^21_{i,j \in \Sprime_{T_s}} \\
    &= \frac{1}{(b^*)^2}\sum_{j=0}^{T_s - 1} X_{j}1_{j \in \Sprime_{T_s}}. \numberthis \label{eq:det2}
\end{align*}
Above we defined the random variable $X_j$ as
\begin{align*}
X_j &= \sum_{i=0}^{j-1} ((D_{\mathrm{U}} + e_j)x_i1_{i \in \Sprime_{T_s}} - (D_{\mathrm{U}}+ e_i)x_j1_{i \in \Sprime_{T_s}})^2 \\
&\le |\Sprime_j|4(D_{\mathrm{U}}+1)^2B_x^2,  && \text{Equation \eqref{eq:ei}} \numberthis \label{eq:bound_on_uncondXj}
\end{align*}
where the last inequality holds by Lemma \ref{bounded_approx} and because $e_j \le 1$ under event $E$ for sufficiently large $T$ by Equation \eqref{eq:ei}. We need one last lemma to help lower bound the conditional expectation of $X_j$.
\begin{lemma}\label{cond_var_bounded}
    Using the same notation and assumptions as in the proof of Lemma \ref{boundary_uncertainty_cont} (and recall that $B_P$ is the upper bound on the density of the noise random variables), if $\P(u_j = u_j^{\mathrm{safeU}} \mid G_j, E) \ge \gamma$ and $\P(E \mid G_j) \ge 1/2$, then $\Var\left(w_{j-1} \cond G_j, E, u_j = u_j^{\mathrm{safeU}}\right) \ge \frac{\gamma^2}{64B_P}$.
\end{lemma}
The proof of Lemma \ref{cond_var_bounded} can be found in Appendix \ref{sec:proof_of_cond_var_bounded}. By definition, $j \in \Sprime_{T_s}$ implies three events: $\{u_j = u_j^{\mathrm{safeU}}\}$, $\{\P(u_j = u_j^{\mathrm{safeU}} \mid G_j, E) \ge \gamma\}$, and $\{\P(E \mid G_j) \ge 1/2\}$. Note that the second and third events are deterministic functions of $G_j$. Therefore in the algebra below, the information in $\{j \in \Sprime_{T_s}\}$ that tells us that $\P(u_j = u_j^{\mathrm{safeU}} \mid G_j, E) \ge \gamma$ and $\P(E \mid G_j) \ge 1/2$ will be absorbed into the conditioning on $G_j$ in the first equality, i.e., starting in the second line below, the $G_j$ being conditioned on should be understood to be one for which $\P(u_j = u_j^{\mathrm{safeU}} \mid G_j, E) \ge \gamma$ and $\P(E \mid G_j) \ge 1/2$. For sufficiently large $T$, 
{\fontsize{10}{10}
\begin{align*}
&\E[X_j \mid G_j, E, j \in \Sprime_{T_s} ] \\
&= \E[X_j \mid G_j,E, u_j = u_j^{\mathrm{safeU}} ] \\
&= \E\left[\sum_{i=0}^{j-1} ((D_{\mathrm{U}} + e_j)x_i1_{i \in \Sprime_{T_s}} - (D_{\mathrm{U}}+ e_i)x_j1_{i \in \Sprime_{T_s}})^2 \mid G_j,E, u_j = u_j^{\mathrm{safeU}} \right] \\
&= \sum_{i=0}^{j-1} \E[ ((D_{\mathrm{U}}+ e_j)x_i1_{i \in \Sprime_{T_s}} - (D_{\mathrm{U}}+ e_i)(ax_{j-1} + bu_{j-1})1_{i \in \Sprime_{T_s}} - (D_{\mathrm{U}}+ e_i)w_{j-1}1_{i \in \Sprime_{T_s}})^2  \mid G_j,E, u_j = u_j^{\mathrm{safeU}}] \\
&\ge \sum_{i=0}^{j-1} \Var\left( (D_{\mathrm{U}}+ e_j)x_i1_{i \in \Sprime_{T_s}} - (D_{\mathrm{U}}+ e_i)(ax_{j-1} + bu_{j-1})1_{i \in \Sprime_{T_s}} - (D_{\mathrm{U}}+ e_i)w_{j-1}1_{i \in \Sprime_{T_s}} \mid G_j,E, u_j = u_j^{\mathrm{safeU}}\right) \\
&= \sum_{i=0}^{j-1} \Var\left((D_{\mathrm{U}}+ e_i)w_{j-1} - e_jx_i| G_j,E, u_j = u_j^{\mathrm{safeU}}\right)1_{{i} \in \Sprime_{T_s}} \\
&= \sum_{i=0}^{j-1} (D_{\mathrm{U}}+ e_i)^2\Var\left(w_{j-1} - \frac{e_jx_i}{D_{\mathrm{U}}+ e_i}\cond G_j,E, u_j = u_j^{\mathrm{safeU}}\right)1_{{i} \in \Sprime_{T_s}} \\
&\ge \sum_{i=0}^{j-1} (D_{\mathrm{U}}+ e_i)^2\left(\Var\left(w_{j-1}\cond G_j,E, u_j = u_j^{\mathrm{safeU}}\right) - \left|2\mathrm{Cov}\left(\frac{e_jx_i}{D_{\mathrm{U}}+ e_i}, w_{j-1}\cond G_j,E, u_j = u_j^{\mathrm{safeU}}\right)\right| \right)1_{{i} \in \Sprime_{T_s}} \\
&\ge \sum_{i=0}^{j-1} (D_{\mathrm{U}}+ e_i)^2\left(\Var\left(w_{j-1}\cond G_j,E, u_j = u_j^{\mathrm{safeU}}\right) - \frac{1}{\log(T)} \right)1_{{i} \in \Sprime_{T_s}}  \quad \quad \quad \text{Suff large $T$ (see below)}\\
&\ge  (D_{\mathrm{U}}- \tilde{O}_T(\nu_T))^2\cdot |\Sprime_j| \cdot \left( \Var\left(w_{j-1} \cond G_j, E, u_j = u_j^{\mathrm{safeU}}\right) - \frac{1}{\log(T)}\right)\quad \quad \quad \text{Equation \eqref{eq:ei}} \\
&= (D_{\mathrm{U}}- \tilde{O}_T(\nu_T))^2 \cdot |\Sprime_j| \cdot \left( \frac{\gamma^2}{64B_P} - \frac{1}{\log(T)}\right) \quad \quad \quad\quad \quad \quad\quad \quad   \text{ $\P(u_j = u_j^{\mathrm{safeU}} \mid G_j, E) \ge \gamma$,  Lemma \ref{cond_var_bounded}}\\
&\ge \frac{D_{\mathrm{U}}^2|\Sprime_j|\gamma^2}{128B_P}.  \quad \quad \quad\quad \quad \quad\quad \quad \quad\quad \quad  \text{Suff large $T$} \numberthis \label{eq:bound_on_condXj}
\end{align*}
}
Note that we are able to divide by $D_{\mathrm{U}} + e_i$ for sufficiently large $T$ by Equation \eqref{eq:ei}. The for-sufficiently-large-$T$ bound on the covariance comes from the fact that under event $E$, we have $|w_{j-1}| = \tilde{O}_T(1)$ and $\frac{e_jx_i}{D_{\mathrm{U}} + e_i} = \tilde{O}_T(\nu_T)$, and therefore for sufficiently large $T$ the covariance has magnitude less than $\frac{1}{2\log(T)}$. 

We can now apply Lemma \ref{sum_of_depend_rand} to $\{X_i\}_{i=0}^{T_s-1}$ with $n = T_s$, $S_n = \Sprime_{T_s}$, $p = \frac{D_{\mathrm{U}}^2\gamma^2}{1024B_P(D_{\mathrm{U}}+1)^2B_x^2}$, $F_i = G_i$, $E^* = E$, and $c =  \frac{D_{\mathrm{U}}^2\gamma^2}{128B_P}$ (where Equations \eqref{eq:bound_on_uncondXj} and \eqref{eq:bound_on_condXj} imply Equations \eqref{eq:condxdef_rand} and \eqref{eq:boundxdef_rand}). Because $D_{\mathrm{U}} \le \log^2(T)$, $B_x = \log^3(T)$, and $\gamma$ is a  constant, this choice of $p$ is less than $1$ for sufficiently large $T$.

Applying Lemma \ref{sum_of_depend_rand} gives that for sufficiently large $T$, conditional on event $E$ with conditional probability $1-o_T(1/T^2)$, 
\begin{align*}
    &(V_{T_s}^{\Sprime_{T_s}})_{22}(V_{T_s}^{\Sprime_{T_s}})_{11} - (V_{T_s}^{\Sprime_{T_s}})_{12}^2 \\
    &\ge \frac{1}{(b^*)^2}\sum_{j = 0}^{T_s - 1} X_{j} && \text{Equation \eqref{eq:det2}} \\
    &\ge\frac{1}{(b^*)^2} \frac{D_{\mathrm{U}}^2 \gamma^2}{512B_P} \left(\max(\left\lfloor p|\Sprime_{T_s}| - \sqrt{|\Sprime_{T_s}|}\log(T) \right\rfloor, 1) - 1\right) \\
    &\quad \quad \times \left(\max(\left\lfloor p|\Sprime_{T_s}| - \sqrt{|\Sprime_{T_s}|}\log(T) \right\rfloor, 1)\right)  && \text{Lemma \ref{sum_of_depend_rand}} \numberthis \label{eq:event_E'}
\end{align*}
Define $E'_s$ as the event that Equation \eqref{eq:event_E'} holds (therefore $\P(E'_s \mid E) = 1 - o_T(1/T^2)$). If $|\Sprime_{T_s}| \ge 4\log^2(T)/p^2$, then $\frac{p|\Sprime_{T_s}|}{2} \ge \log(T)\sqrt{|\Sprime_{T_s}|}$, and therefore
\begin{equation}\label{eq:p_Sprime}
p|\Sprime_{T_s}| - \log(T)\sqrt{|\Sprime_{T_s}|} \ge \frac{p|\Sprime_{T_s}|}{2} \ge 1.
\end{equation}
Therefore, conditional on $E \cap E'_s \cap \{|\Sprime_{T_s}| \ge 4\log^2(T)/p^2\}$,
\begin{align*}
    &(V_{T_s}^{\Sprime_{T_s}})_{22}(V_{T_s}^{\Sprime_{T_s}})_{11} - (V_{T_s}^{\Sprime_{T_s}})_{12}^2\\
    &\ge \frac{1}{(b^*)^2} \frac{D_{\mathrm{U}}^2 \gamma^2}{512B_P} \left(\max(\left\lfloor p|\Sprime_{T_s}| - \sqrt{|\Sprime_{T_s}|}\log(T) \right\rfloor, 1) - 1\right) \\
    &\quad \quad \times \left(\max(\left\lfloor p|\Sprime_{T_s}| - \sqrt{|\Sprime_{T_s}|}\log(T) \right\rfloor, 1)\right) && \text{Equation \eqref{eq:event_E'}} \\
    &\ge\frac{1}{(b^*)^2} \frac{D_{\mathrm{U}}^2 \gamma^2}{512B_P} \left(\left\lfloor p|\Sprime_{T_s}|/2 \right\rfloor - 1\right) \left(\left\lfloor p|\Sprime_{T_s}|/2 \right\rfloor\right) && \text{Equation \eqref{eq:p_Sprime}}\\
    &= \tilde{\Omega}_T \left( |\Sprime_{T_s}|^2\right). \numberthis \label{eq:det2_full}
\end{align*}
Because $E \subseteq E_1$, we have by Equation \eqref{eq:bounded_bt} that $B_{T_s} = \tilde{O}_T(1)$ conditional on event $E \cap N$. Therefore, by Lemma \ref{v_to_use} and Equations \eqref{eq:num1}, \eqref{eq:num2}, \eqref{eq:det2_full}, and \eqref{eq:bounded_bt}, we have conditional on event $E \cap E'_s \cap N \cap \{|\Sprime_{T_s}| \ge 4\log^2(T)/p^2\}$ and for sufficiently large $T$,
\begin{equation}\label{eq:epsilon_omega}
    \epsilon_s^2 \le \frac{\lambda + |\Sprime_{T_s}|B_x^2\tilde{O}_T(1)}{\tilde{\Omega}(|\Sprime_{T_s}|^2)} \le \tilde{O}_T\left(\frac{1}{|\Sprime_{T_s}|}\right).
\end{equation}
Taking $E' = \cap_{s \in [0:s_e]} E'_s$, Equation \eqref{eq:epsilon_omega} implies that conditional on $E \cap E' \cap N \cap \{|\Sprime_{T_s}| \ge 4\log^2(T)/p^2\}$,
\begin{equation}\label{eq:first_Sprime}
 \epsilon_s^2|\Sprime_{T_s}| \le \tilde{O}_T(1).
\end{equation}
Under event $E$, because $E_2 \subseteq E$, $\epsilon_s = \tilde{O}_T(\nu_T)$. Therefore, conditional on $E \cap E' \cap N \cap \{|\Sprime_{T_s}| < 4\log^2(T)/p^2\}$,
\begin{equation}\label{eq:small_Sts}
    \epsilon_s^2|\Sprime_{T_s}| \le \epsilon_s^2 \cdot 4\log^2(T)/p^2 = \tilde{O}_T(\nu_T^2)  = \tilde{O}_T(1).
\end{equation}
Because Equations \eqref{eq:first_Sprime} and \eqref{eq:small_Sts} hold for all $s \in [0:s_e]$, the right hand sides do not depend on $s$, and the equations hold almost surely, these two equations together imply that conditional on $E \cap E' \cap N$,
\[
    \max_{s \in [0:s_e]} \epsilon_s^2|\Sprime_{T_s}| = \tilde{O}_T(1).
\]
Because $\P(E) \ge 1-o_T(1/T)$, $\P(N) \ge 1-o_T(1/T)$, and $\P(E'  \mid E) \ge 1 - \sum_{s=0}^{s_e} \P(E'_s \mid E) = 1-o_T(1/T)$, by a union bound we can conclude that with probability $1-o_T(1/T)$,
\[
    \max_{s \in [0:s_e]} \epsilon_s^2|\Sprime_{T_s}| = \tilde{O}_T(1).
\]
This completes the proof of Equation \eqref{eq:Sprime_to_show}, and therefore completes the proof of this lemma.
\end{proof}

\subsection{Proof of Lemma \ref{v_to_use}}\label{sec:proof_of_v_to_use}
Recall that Lemma \ref{v_to_use} applies to all algorithms as defined in the lemma statement, and therefore this lemma is not specific to a previous appendix section.
\begin{proof}
    First, we restate the theorem from \cite{abbasi2011regret} in the notation and setup of this paper.
\begin{lemma}[Restatement of Theorem 1 in \cite{abbasi2011regret}]\label{uncertainty_ref}
    Let $\theta^* \in \mathbb{R}^2$ and $C$ be a controller. For $t \in [0:T-1]$, define $z_t = (x_t, C(x_t))$ and $x_{t+1} = \theta^* \cdot z_t + w_t$ where $w_t \sim_{i.i.d.} \mathcal{D}$ and $\mathcal{D}$ a subgaussian distribution with mean $0$ and variance $1$, and $\norm{\theta^*}_{2} \le \bar{a}^2 + \bar{b}^2$. Define $V_t = \lambda I + \sum_{s=0}^{t-1} z_sz_s^\top$, $Z_t$ as the matrix where row $i \in [1:t]$ is $z_{i-1}^\top$, and $X_t$ as the matrix where row $i \in [1:t]$ is $x_{i}$. Finally, let $\hat{\theta}_t = (Z_{t}^{\top}Z_{t}+\lambda I)^{-1}Z_{t}^{\top}X_{t}$ and $\Delta_t = \hat{\theta} - \theta^*$. Then with probability $1-o_T(1/T^2)$, for all $1 \le t \le T$.  
    \begin{equation}\label{eq:lemma18}
        \mathrm{Tr}(\Delta_t^\top V_t\Delta_t) \le B_t^2,
    \end{equation}
    where $B_t = \alpha\sqrt{\log(\det(V_t)) + \log(\lambda^2) + 2\log(T^2)} + \sqrt{\lambda}(\bar{a}^2 + \bar{b}^2)$ and $\alpha$ satisfies $\E_{w \sim \mathcal{D}}[\exp(\gamma w)] \le \exp(\gamma^2\alpha^2/2)$ for any $\gamma \in \mathbb{R}$.
\end{lemma}

    Now define $V_t^{S^{\mathsf{c}}} = \lambda I + \sum_{s=0}^{t-1} z_sz_s^\top1_{s \not\in S}$. Then by Lemma \ref{uncertainty_ref},
    \begin{align*}
         B_t \ge \mathrm{Tr}(\Delta_t^\top V_t\Delta_t) = \mathrm{Tr}(\Delta_t^\top (V_t^S + V_t^{S^{\mathsf{c}}})\Delta_t) = \mathrm{Tr}(\Delta_t^\top V_t^S\Delta_t) + \mathrm{Tr}(\Delta_t^\top V_t^{S^{\mathsf{c}}}\Delta_t).
    \end{align*}
    Because both traces are non-negative, this implies that $
        \mathrm{Tr}(\Delta_t^\top V_t^{S}\Delta_t) \le B_t^2$. Suppose $\Delta_t = (\Delta_{ta}, \Delta_{tb})$. Then expanding the trace  gives that
    \[
        (V_t^{S})_{11}\Delta_{ta}^2 + (V_t^{S})_{22}\Delta_{tb}^2 + 2\Delta_{ta}\Delta_{tb}(V_t^{S})_{12} \le B_t^2.
    \]
    The left side of the above equation is a quadratic in $\Delta_{tb}$, with minimum occurring at $\Delta_{tb} = \frac{-\Delta_{ta} (V_t^{S})_{12}}{(V_t^{S})_{22}}$. Therefore, plugging this in gives the following inequality.
    \[
        (V_t^{S})_{11}\Delta_{ta}^2 - \frac{\Delta_{ta}^2(V_t^{S})_{12}^2}{(V_t^{S})_{22}} \le B_t^2.
    \]
    Simplifying, we have the desired result that $\Delta_{ta}^2 \le \frac{(V_t^{S})_{22}}{(V_t^{S})_{11}(V_t^{S})_{22} - (V_t^{S})_{12}^2}B_t^2$. The proof follows symmetrically for $\Delta_{tb}$. 
    \end{proof}

\subsection{Proof of Lemma \ref{sum_of_depend_rand}}\label{sec:proof_of_sum_of_depend_rand}
\begin{proof}
    By the law of total expectation, for all $k \in [0,n-1]$,
    \begin{align*}
       c |S_k| &\le \E\left[X_k \cond F_k, E^*, k \in S_n \right] && \text{Eq \eqref{eq:condxdef_rand}} \\
       &= \E\left[X_k \cond F_k, E^*, k \in S_n, X_k \le \frac{c |S_k|}{2}\right]\P\left(X_k \le \frac{c |S_k| }{2} \cond  F_k, E^*, k \in S_n\right) \\
       &\quad \quad  +  \E\left[X_k \cond X_k > \frac{c |S_k|}{2},  F_k, E^*, k \in S_n\right]\P\left(X_k > \frac{c |S_k|}{2}\cond F_k, E^*, k \in S_n\right)  \\
       &\le \frac{c |S_k|}{2} + \frac{c|S_k|}{2p}\P\left(X > \frac{c |S_k|}{2}\cond  F_{k}, E^*, k \in S_n\right). && \text{Eq \eqref{eq:boundxdef_rand}} 
    \end{align*}
 
    For $i \in [0:|S_n| - 1]$, define $\kappa_i$ as the $(i+1)$th smallest index in the set $S_n$. This implies that $|S_{\kappa_i}| = i$ and $\kappa_i \in S_n$. By Equation \eqref{eq:boundxdef_rand}, for all $k$, 
    \begin{equation}\label{eq:stoch_dom_rand}
          \P\left(X_k \ge \frac{c|S_k|}{2} \cond F_k, E^*, k \in S_n, \kappa_{|S_{k}|} = k\right)  
 = \P\left(X_k \ge \frac{c|S_k|}{2} \cond F_k, E^*, k \in S_n\right)  \ge \frac{c|S_k|/2}{c|S_k|/2p}  = p.
    \end{equation}
    Note that the first equality comes from the fact that by definition, $\kappa_{|S_k|} = k$ if $k \in S_n$, and $|S_k|$ is a deterministic function of $F_k$.
    
    Let $A_0,A_1,...,A_{n-1}$ be a sequence of i.i.d. Bernoulli random variables with probability $p$ of being $1$ that are independent of all other random variables in this lemma, including $E^*, S_n, X_i, F_i$ for all $i$. For $i \in [0:n-1]$, define the random variable $A_i'$ as
    \[
        A_i' =  \begin{cases}
                  1_{X_{\kappa_i} \ge \frac{c\cdot i}{2}} & \text{if  $i \le |S_n|-1$}\\
                  A_i  &\text{otherwise. }\\
    \end{cases}
    \]

    Define $F_i^A := F_{\kappa_{\min(i, |S_n|-1)}} \cup \{A_0,...,A_{i-1}\}$. By Equation \eqref{eq:stoch_dom_rand}, we have that for all $i$,
    {\fontsize{10}{10}
    \begin{align*}
        &\P\left(A'_i = 1 \cond F_i^A, E^*, i \le |S_n| - 1\right) \\
        &= \sum_{k=0}^{n-1} \P\left(A'_i = 1 \cond F_i^A, E^*, i \le |S_n| - 1, \kappa_i = k\right) \P\left( \kappa_i = k  \cond F_i^A, E^*, i \le |S_n| - 1\right) 
 && \text{LoTE}\\
        &= \sum_{k=0}^{n-1}\P\left(X_{\kappa_i} \ge \frac{c \cdot i}{2} \cond F_i^A, E^*, i \le |S_n| - 1, \kappa_i = k\right) \P\left( \kappa_i = k  \cond F_i^A, E^*, i \le |S_n| - 1\right) \\
        &= \sum_{k=0}^{n-1}\P\left(X_{\kappa_i} \ge \frac{c \cdot |S_{\kappa_i}|}{2} \cond F_i^A, E^*, i \le |S_n| - 1, \kappa_i = k\right) \P\left( \kappa_i = k  \cond F_i^A, E^*, i \le |S_n| - 1\right) \\
        &= \sum_{k=0}^{n-1}\P\left(X_{\kappa_i} \ge \frac{c \cdot |S_{\kappa_i}|}{2} \cond F_{\kappa_i}, E^*, \kappa_i \in S_n, \kappa_i = k\right) \P\left( \kappa_i = k  \cond F_i^A, E^*, i \le |S_n| - 1\right) \\
        &= \sum_{k=0}^{n-1}\P\left(X_{k} \ge \frac{c \cdot |S_{k}|}{2} \cond F_{k}, E^*, k \in S_n, \kappa_i = k\right) \P\left( \kappa_i = k  \cond F_i^A, E^*, i \le |S_n| - 1\right) \\
        &= \sum_{k=0}^{n-1}\P\left(X_{k} \ge \frac{c \cdot |S_{k}|}{2} \cond F_{k}, E^*, k \in S_n, \kappa_{|S_k|} = k\right) \P\left( \kappa_i = k  \cond F_i^A, E^*, i \le |S_n| - 1\right) && \text{$i = |S_{\kappa_i}| = |S_k|$}\\
        &\ge \sum_{k=0}^{n-1}p \cdot \P\left( \kappa_i = k  \cond F_i^A, E^*, i \le |S_n| - 1\right)  && \text{Eq \eqref{eq:stoch_dom_rand}} \\
        &= p, \numberthis \label{eq:coupling_cond_A}
    \end{align*}
    }
    and 
    \begin{align*}
        &\P\left(A'_i = 1 \cond F_i^A, E^*, i > |S_n| - 1\right) \\
        &= \P\left(A_i = 1 \cond F_i^A, E^*, i > |S_n| - 1 \right) && \text{Independence of $A_i$}\\
        &= p. \numberthis \label{eq:coupling_cond_Ab}
    \end{align*}
    Putting together Equations \eqref{eq:coupling_cond_A} and \eqref{eq:coupling_cond_Ab} and the Law of Total Probability,
        \begin{align*}
        &\P\left(A'_i = 1 \cond F_i^A, E^* \right) \\
        &= \P\left(A'_i = 1 \cond F_i^A, E^*, i \le |S_n| - 1\right) \P\left(i \le |S_n| - 1 \cond  F_i^A, E^*\right)\\ 
        & \quad \quad+ \P\left(A'_i = 1 \cond F_i^A, E^*, i > |S_n| - 1\right) \P\left(i > |S_n| - 1 \cond  F_i^A, E^*\right) \\
        &\ge p. && \text{Eqs \eqref{eq:coupling_cond_A} and \eqref{eq:coupling_cond_Ab}}. \numberthis \label{eq:coupling_cond_A_combined}
    \end{align*}
   Because  $A'_i$ is a deterministic function of $F_{i+1}^A$ and $F_i^A \subseteq F_{i+1}^A$, Equation \eqref{eq:coupling_cond_A_combined} implies that $M_k = \sum_{i=0}^{k-1} (A_i' - p)$ is a submartingale conditional on $E^*$ with increments bounded in magnitude by $1$.  For any non-random $m \in [1:n]$, the Azuma--Hoeffding Inequality therefore gives that
    \[
        \P\left( \sum_{i=0}^{m-1} (A_i' - p) \ge - \log(T)\sqrt{m} \cond E^* \right) \ge  1 - e^{-\log^2(T)m/(2m)} = 1- o_T(1/T^3).
    \]    
    Taking a union bound over all $m \in [1:n]$ (because $n \le T$), we have that 
    \[
        \P\left( \forall m \in [1:n], \sum_{i=0}^{m-1} A_i' \ge pm - \log(T)\sqrt{m} \cond E^* \right) \ge 1- o_T(1/T^2).
    \]
    Define $E'$ as the event that for all $m \in [1:n]$, $\sum_{i=0}^{m-1} A_i' \ge pm - \log(T)\sqrt{m}$. Because $|S_n| \in [0,n]$, we must have that conditional on event $E'$,
    \begin{equation}\label{eq:As_cond_on_E}
        \sum_{i=0}^{|S_n| - 1} A_i' \ge p|S_n| - \log(T)\sqrt{|S_n|}.
    \end{equation}
    Therefore, conditional on event $E'$, we have
    \begin{align*}
        \sum_{j=0}^{n-1} X_j &\ge \sum_{j=0, j \in S_n}^{n-1} X_j && \text{$X_j \ge 0$} \\
        &\ge \sum_{i=0}^{|S_n|-1} \frac{c \cdot i}{2} \cdot A'_i && \text{Def of $A'_i$}  \\
        &\ge \frac{c}{2}\sum_{k=0}^{\max(\lfloor p|S_n| - \log(T)\sqrt{|S_n|} \rfloor, 1)-1 } k. && \text{Eq \eqref{eq:As_cond_on_E}} \\
        &= \frac{c}{4}\left(\max(\lfloor p|S_n| - \log(T)\sqrt{|S_n|} \rfloor, 1)\right)\left(\max(\lfloor p|S_n| - \log(T)\sqrt{|S_n|} \rfloor, 1)-1\right)
    \end{align*}
    Because we already showed that $\P(E' \mid E^*) \ge 1-o_T(1/T^2)$, this is the desired result.
\end{proof}

\subsection{Proof of Lemma \ref{logt_xvals}}\label{sec:proof_of_logt_xvals}

Recall that Lemma \ref{logt_xvals} was stated to be used in Appendix \ref{app:proof_of_performance} with respect to Algorithm \ref{alg:cap}, therefore all events and variables in this subsection refer to those defined with respect to Algorithm \ref{alg:cap}. 
\begin{proof}
     Define $A_i  = 1_{|x_i| \le \frac{1}{\log(T)}}$. Recall that $x_i = a^*x_{i-1} + b^*u_{i-1} + w_{i-1}$, where $x_{i-1}$ and $u_{i-1}$ are respectively the position and control at time $t = i-1$. The probability that $A_i$ is equal to $1$ is the probability that $w_{i-1} \in [-(a^*x_{i-1} + b^*u_{i-1})- \frac{1}{\log(T)}, -(a^*x_{i-1} + b^*u_{i-1}) + \frac{1}{\log(T)}]$. Because $\mathcal{D}$ has a bounded density function (bounded by $B_P$) as assumed in Assumption \ref{problem_specifications}, the conditional probability given $G_i$ is at most $\frac{2B_P}{\log(T)}$. Therefore, we have that
    \[
        \P(A_i = 1 \mid G_i) \le \frac{2B_P}{\log(T)}.
    \]
    Therefore, $M_j = \sum_{i=0}^{j-1} (A_i - \frac{2B_P}{\log(T)})$ is a submartingale with differences bounded in magnitude by $\max(1,\frac{2B_P}{\log(T)}) \le 1$ for sufficiently large $T$. By Azuma--Hoeffding's inequality, with probability $1-o_T(1/T^3)$,
    \[
        M_j  \le  \log(T)\sqrt{j}.
    \]
    Define $E_{\mathrm{L}\ref{logt_xvals}}^j$ as the event that this bound on $M_j$ holds. By construction of $M_j$, under event $E_{\mathrm{L}\ref{logt_xvals}}^j$,
    \[
      \left| \left\{i < j : |x_i| \le \frac{1}{\log(T)}\right\}\right| = \sum_{i=0}^{j-1} A_i  \le  \frac{2jB_P}{\log(T)} + \log(T)\sqrt{j} \le \frac{4jB_P}{\log(T)}
    \]
    for $j \ge \log^8(T)$ assuming $T$ is large enough that $\log^2(T) \ge \frac{1}{2B_P}$. As long as $\log(T) \ge 8B_P$, this implies that under event $E_{\mathrm{L}\ref{logt_xvals}}^j$,
    \[
        \left| \left\{i < j : |x_i|^2 \ge \frac{1}{\log^2(T)}\right\}\right| \ge j - \frac{4jB_P}{\log(T)} \ge \frac{j}{2}.
    \]
    Finally, we can conclude that under event $E_{\mathrm{L}\ref{logt_xvals}}^j$,
    \[
    \sum_{i=0}^{j-1} x_i^2 \ge \frac{j}{2\log^2(T)}.
    \]
    We have shown that Equation \eqref{eq:logt_xvals} holds for any fixed $j$ under event $E_{\mathrm{L}\ref{logt_xvals}}^j$  for sufficiently large $T$. Therefore, the same result holds for all $j \ge \log^8(T)$ under event  $E_{\mathrm{L}\ref{logt_xvals}}  = \cap_{j \ge \log^8(T)} E_{\mathrm{L}\ref{logt_xvals}}^j$. By a union bound and because $\P(E^j_{\mathrm{L}\ref{logt_xvals}}) = 1-o_T(1/T^3)$ for all $j$, we have that $\P(E_{\mathrm{L}\ref{logt_xvals}}) = 1-o_T(1/T^2)$.
\end{proof}

\subsection{Proof of Lemma \ref{cond_var_bounded}}\label{sec:proof_of_cond_var_bounded}

Recall that Lemma \ref{cond_var_bounded} is defined to be used in Appendix \ref{app:suff_large_noise_case} with respect to Algorithm \ref{alg:cap_large}, therefore all events and variables in this subsection refer to those defined with respect to Algorithm \ref{alg:cap_large}. 
\begin{proof}

By assumption of this lemma, 
\begin{align*}
    \P(u_j = u_j^{\mathrm{safeU}}, E \mid G_j) &=  \P(u_j = u_j^{\mathrm{safeU}} \mid G_j, E)\P(E \mid G_j)  \ge \frac{\gamma}{2}. \numberthis \label{eq:uj_and_E}
\end{align*}
We also note the following result:
\begin{lemma}\label{lemma:var_bound}
   For any event $E^*$ such that $\P(E^*) > 0$,
    \[
        \var_{w \sim \mathcal{D}}(w \mid E^*) \ge \frac{\P(E^*)^2}{16B_P^2}
    \]
\end{lemma}
\begin{proof}    
    First, we will show that any continuous distribution $\mathcal{D}'$ with density function bounded by $B$ must have variance at least $\frac{1}{16B^2}$. Let $f_{\mathcal{D}'}$ be the probability density function of $\mathcal{D}'$. First, we can assume WLOG that $\mathcal{D}'$ has mean $0$ (this is without loss of generality because variance is invariant to shifts in mean). If $\mathcal{D}'$ has mean $0$, then by the law of total expectation
    \[
        \E_{x \sim \mathcal{D}'} [x \mid x \ge 0]\P_{x \sim \mathcal{D}'}(x \ge 0) = -\E_{x \sim \mathcal{D}'} [x \mid x \le 0]\P_{x \sim \mathcal{D}'}(x \le 0) .
    \]
    Note that we can have non-strict inequalities because $\mathcal{D}'$ is continuous. Furthermore, either $\P_{x \sim \mathcal{D}'}(x \le 0) \ge 1/2$ or $\P_{x \sim \mathcal{D}'}(x \ge 0) \ge 1/2$. Because variance is invariant to multiplying by $-1$, we can assume WLOG that $\P_{x \sim \mathcal{D}'}(x \ge 0) \ge 1/2$. If $\P_{x \sim \mathcal{D}'}(x \ge 0) \ge 1/2$ then $\int_0^{\infty} f_{\mathcal{D}'}(x)dx \ge 1/2$. Define $f^*(x) = \frac{1}{2B}$ for $x \in [0,B]$ and $f^*(x) = 0$ otherwise. Note that $f=f^*$ achieves the minimum possible value of $\int_{0}^{\infty} x \cdot f(x)dx$ subject to the constraints $\int_0^{\infty} f(x)dx \ge 1/2$ and $0 \le f(x) \le B$ for all $x$. This is because $f^*$ puts as much weight as possible close to $0$ without violating the bounded by $B$ constraint. Furthermore, any $f$ such that $\int_{1/2B}^{\infty} f(x)dx > 0$ puts non-$0$ weight on values of $x$ greater than $B$ and therefore has a larger value of  $\int_{0}^{\infty} x \cdot f(x)dx$ than $f^*$. Using this, we have that
    \[
        \E_{x \sim \mathcal{D}'} [x \mid x \ge 0]\P_{x \sim \mathcal{D}'}(x \ge 0) = \int_{0}^{\infty} x \cdot f_{\mathcal{D}'}(x)dx \ge \int_{0}^{1/2B} x \cdot B dx =  \frac{1}{8B}.
    \]
    Therefore, we must have (again by the law of total expectation) that
    \[
        \E_{x \sim \mathcal{D}'}[|x|] = \E_{x \sim \mathcal{D}'} [x \mid x \ge 0]\P_{x \sim \mathcal{D}'}(x \ge 0) -\E_{x \sim \mathcal{D}'} [x \mid x \le 0]\P_{x \sim \mathcal{D}'}(x \le 0) \ge \frac{1}{4B}.
    \]
    By Jensen's inequality, 
    \[
       \mathrm{Var}_{x \sim \mathcal{D}'}(x) =  \E_{x \sim \mathcal{D}'}[x^2] =   \E_{x \sim \mathcal{D}'}[|x|^2] \ge \E_{x \sim \mathcal{D}'}[|x|]^2 \ge \frac{1}{16B^2}.
    \]
    We have therefore  shown that any continuous distribution $\mathcal{D}'$ with probability density function $f$ such that $f(x) \le B$ for all $x$ must have variance at least $\frac{1}{16B^2}$.

    We know that the conditional distribution of $w$ given $E^*$ has a probability density function that is bounded by $\frac{B_P}{\P(E^*)}$. Therefore, we must have that $\var(w \mid E^*) \ge \frac{\P(E^*)^2}{16B_P^2}$.
\end{proof}
Recall that $w_{j-1}$ is independent of $G_j$. Therefore, $\Var\left(w_{j-1} \cond G_j, E, u_j=u_j^{\mathrm{safeU}} \right)$ is simply the variance of $w_{j-1}$ conditional on an event that has probability $\P(E, u_j=u_j^{\mathrm{safeU}} \mid G_j)$. Therefore, we can apply Lemma \ref{lemma:var_bound} and Equation \eqref{eq:uj_and_E} to get that for some event $E'$ such that $\P(E') \ge \gamma/2$,
\begin{align*}
    \Var\left(w_{j-1} \cond G_j, E, u_j=u_j^{\mathrm{safeU}} \right)  &= \var\left(w_{j-1} \cond E'\right)  \ge \frac{\gamma^2}{64B_P^2}.
\end{align*}
\end{proof}

\section{Feasibility and Boundary Proofs}\label{app:feasibility}

\subsection{Relaxation of Assumption \ref{assum_init}}\label{app:relax_assum_init}

The assumption that $\underline{a}, \underline{b} > 0$ in Assumption \ref{assum_init} can actually be dropped under Assumptions \ref{assum:initial} and \ref{problem_specifications}. Informally, this is because the controller $C^{\mathrm{init}}$ can be used for $\log^{10}(T)$ steps to, with high probability, obtain an estimate $\hat{\theta}$ such that $\norm{\hat{\theta} - \theta^*}_{\infty} \le \frac{1}{\log(T)}$ (by the same logic as in Lemma \ref{initial_uncertainty}). Therefore, we could include an initial phase in every algorithm that does $\log^{10}(T)$ steps of initial exploration and then replaces $\Theta$ with $\Theta' = \{\theta : \norm{\theta - \hat{\theta}}_{\infty} \le \frac{1}{\log(T)}\}$, and this $\Theta'$ will satisfy $\underline{a}', \underline{b}' > 0$ for sufficiently large $T$ because $a^* > 0$. However, to simplify the algorithms and proofs we will assume that the initial uncertainty set $\Theta$ is small enough that this is unnecessary. Note that this assumption of sufficiently small bounded initial uncertainty appears in other safe LQR literature such as \cite{li2021safe}.

\subsection{Discussion on Assumption \ref{assum:initial}}\label{app:assum_initial}

To better understand Assumption \ref{assum:initial}, consider the case of bounded noise and constant boundaries as in \cite{li2021safe, dean2019safely}. In this case, to satisfy Assumption \ref{assum:initial}, it is sufficient to replace the $\forall x \in \left[ D_{\mathrm{L}}^{\E[x]} + F_{\mathcal{D}}^{-1}(\frac{1}{T^4}), D_{\mathrm{U}}^{\E[x]} + F_{\mathcal{D}}^{-1}(1-\frac{1}{T^4})\right] $ with $ \forall x \in [D_{\mathrm{L}}^{\E[x]} - \bar{w}, D_{\mathrm{U}}^{\E[x]} + \bar{w} ]$. \cite{li2021safe} makes a similar assumption that there is an initial linear controller that satisfies this property. For the bounded noise case, Assumption \ref{assum:initial} can be shown to be equivalent to an assumption on the size of the initial uncertainty set. Let $C^{\mathrm{init}}(x_t) = -\frac{a}{b}x_t$ for some arbitrary $\theta \in \Theta$. When using this controller, the position and control at time $t$ (denoted $x_t$ and $u_t$ respectively) satisfy
\[
|a^*x_{t} + b^*u_t| \le |x_t|\left|a^* - \frac{ab^*}{b} \right| \le |x_t|\left|a^*-a - \frac{(b^*-b)a}{b}\right| \le \left(1 + \frac{a}{b}\right)|x_{t}|\mathrm{size}(\Theta) \le \left(1 + \frac{\bar{a}}{\underline{b}}\right)|x_{t}|\mathrm{size}(\Theta).
\] 
This controller $C^{\mathrm{init}}$ satisfies Assumption \ref{assum:initial} under bounded noise if 
\[
\mathrm{size}(\Theta) \le  \frac{\min(D_{\mathrm{U}}^{\E[x]}, |D_{\mathrm{L}}^{\E[x]}|) - \frac{\bar{b}}{\log(T)}}{\left|1 + \frac{\bar{a}}{\underline{b}}\right|\left(\norm{D^{\E[x]}}_{\infty} + \bar{w}\right)}.
\]
Therefore, instead of assuming Assumption \ref{assum:initial}, it is sufficient to assume that $\mathrm{size}(\Theta) \le  \frac{\min(D_{\mathrm{U}}^{\E[x]}, |D_{\mathrm{L}}^{\E[x]}|) - \frac{\bar{b}}{\log(T)}}{\left|1 + \frac{\bar{a}}{\underline{b}}\right|\left(\norm{D^{\E[x]}}_{\infty} + \bar{w}\right)}$, as the controller  $C^{\mathrm{init}}(x_t) = -\frac{a}{b}x_t$ satisfies Assumption \ref{assum:initial}. The bound on $\mathrm{size}(\Theta)$ does still depend on the end points of $\Theta$. As a sanity check, suppose $\norm{D^{\E[x]}}_{\infty} = O_T(1)$ and $\bar{a}, \bar{b}, \frac{1}{\underline{b}} \le c$ for some constant $c$. Then there exists a constant such that if $\mathrm{size}(\Theta)$ is less than that constant, then Assumption \ref{assum:initial} is satisfied for sufficiently large $T$. 

\subsection{Assumptions Relationship to Infeasibility}\label{app:infeasibility}

In this section we briefly relate the assumptions we make to a notion of infeasibility. We begin with two formal definitions. The first is a formal definition of feasibility for our problem. The second is a property of a controller that is slightly stronger than regular safety.

\begin{definition}[Feasibility]\label{def:feas}
    An initial uncertainty set of system dynamics $\Theta$ is \emph{feasible} for boundary $D^{\E[x]}$ and trajectory length $T$ with probability $1-\delta$ if there exists a controller $C$ that satisfies the following. For any $\theta^* \in \Theta$, if the true dynamics are $\theta^*$, then 
    \[
        \P\left(\forall t < T : D_{\mathrm{L}}^{\E[x]} \le a^*x_t + b^*C(H_t) \le D_{\mathrm{U}}^{\E[x]} \right) \ge 1-\delta. 
    \]
\end{definition}

\begin{definition}[Robust safety]
   A controller $C$ is \emph{robustly safe} for $T_0$ time steps for dynamics $\theta^*$ if the following holds for some known distribution $\rho$ with mean $0$ and constant variance $\eta^2> 0$. If $s_t \stackrel{\text{i.i.d.}}{\sim}\rho$ and $u_t = C(H_t) + \frac{s_t}{\log(T)}$, then
   \[
    \P\left(\forall t \in [0,T_0-1] : D_{\mathrm{L}}^{\E[x]} \le a^*x_t + b^*u_t \le D_{\mathrm{U}}^{\E[x]} \right) \ge 1-o_T(1/T^4).
    \]
\end{definition}

\begin{proposition}\label{robustly_feasible}
    The result of Theorem \ref{sufficiently_large_error} hold without Assumption \ref{assum:initial} if we assume access to a controller $C^{\mathrm{rs}}$ that is robustly safe for $\sqrt{T}$ steps. Similarly, the result of Theorem \ref{performance} holds without Assumption \ref{assum:initial} if we assume access to a controller $C^{\mathrm{rs}}$ that is robustly safe for $T^{2/3}$ steps.
\end{proposition}
\begin{proof}
    In the exploration phase of any of the three algorithms, instead of sampling $\phi_t$ from Rademacher distribution we can instead sample i.i.d. from $\rho$ and keep the rest of the algorithm the same. Then the robust feasibility implies that with probability $1-o_T(1/T^4)$ the algorithm will be safe for the warm-up period of the first $\frac{1}{\nu_0^2}$ steps. We can then proof a variation of Lemma \ref{initial_uncertainty} that holds using the distribution $\rho$ instead of the Rademacher distribution.
\end{proof}

By Definition \ref{def:feas}, as $T$ approaches infinity, the existence of a robustly safe controller $C^{\mathrm{rs}}$ becomes intuitively equivalent to $\Theta$ being feasible for boundary $D^{\E[x]}$ with probability $1-o_T(1/T^4)$. Therefore by Proposition \ref{robustly_feasible}, Assumption \ref{assum:initial} is intuitively asymptotically equivalent to the assumption that the problem is feasible for dynamics $\Theta$ and that a controller that achieves feasibility is known.

\section{Generalizations}\label{app:generalizations}
\subsection{Control Constraints}\label{sec:control_constraints}
Our results focus on positional constraints, but we believe that our results with the same rates of regret will also hold with both positional and control constraints under some additional assumptions. While we leave the formal derivations of results for control constraints to future work, we provide a brief discussion of how the algorithm and proofs from this paper could be extended to include control constraints. 

First, we briefly mention how control constraints change the definitions and notation used. Control constraints would be of the form $D_{\mathrm{L}}^u \le u_t \le D_{\mathrm{U}}^u$ for all $t < T$ (for the rest of this section, we will refer to the expected-position constraints as $D^{\E[x]})$. We also define the function $K_{\mathrm{opt}}(\theta, T, D^{\E[x]}, D^u)$ as choosing the optimal parameter $K$ for a controller satisfying both the position constraints $D^{\E[x]}$ and the control constraints $D^u$. We also need the additional assumption that there exists a (non-empty) set of baseline controllers that can satisfy both the position and control constraints. Finally, we need to assume that the controller $C^{\mathrm{init}}$ satisfies both position and control constraints (i.e. an analogue of Assumption \ref{assum:initial}). 

\subsubsection{Theorem \ref{performance} and Algorithm \ref{alg:cap}}\label{app:control_theorem3}

We start with considering how Algorithm \ref{alg:cap} would need to be modified with the addition of control constraints. The key idea behind Algorithm \ref{alg:cap} satisfying the position constraints is that the algorithm sometimes uses controls $u_t^{\mathrm{safeU}}$ and $u_t^{\mathrm{safeL}}$ to enforce positional safety. However, in the presence of control constraints, we can no longer use the controls $u_t^{\mathrm{safeU}}$ and $u_t^{\mathrm{safeL}}$, as these controls may not satisfy the control constraints. The key modification of Algorithm \ref{alg:cap} is to choose the controller $C_s^{\mathrm{alg}}$ in such a way that $C_s^{\mathrm{alg}}$ will satisfy a tighter positional constraint with respect to $D^{\E[x]'} = (D_L^{\E[x]} + \tilde{\Theta}_T(\epsilon_s), D_U^{\E[x]} - \tilde{\Theta}_T(\epsilon_s))$ for dynamics $\hat{\theta}_s$ and a tighter control constraint $D^{u'} = (D_L^u + \tilde{\Theta}_T(\epsilon_s), D_U^u - \tilde{\Theta}_T(\epsilon_s))$.  In other words, choosing $C_s^{\mathrm{alg}} = C_{K_{\mathrm{opt}}(\hat{\theta}_s, T_s, D^{\E[x]'}, D^{u'})}^{\hat{\theta}_s}$. Within each iteration of the safe exploitation phase, the algorithm then can directly use $C_s^{\mathrm{alg}}$.  Because $\norm{\hat{\theta}_s - \theta^*}_{\infty} \le \tilde{O}_T(\epsilon_s)$ with high probability and this $C_s^{\mathrm{alg}}$ is chosen to satisfy the tighter position constraints $D^{\E[x]'}$ for dynamics $\hat{\theta}_s$, the controller $C_s^{\mathrm{alg}}$ will satisfy the true position constraints $D^{\E[x]}$ for dynamics $\theta^*$ with high probability.  Because $C_s^{\mathrm{alg}}$ satisfies the tighter control constraints $D^{u'}$, with the additional assumption that the controller class is continuous, the controls used by $C_s^{\mathrm{alg}}$ under dynamics $\theta^*$ will also satisfy the control constraints with high probability.

Now we will briefly describe what additional results need to be proven in order for the modified version of Algorithm \ref{alg:cap} described above to achieve the same regret rate of $\tilde{O}_T(T^{2/3})$ in the presence of control constraints. We will do this by analyzing each of the terms of regret from the proof of Theorem \ref{performance}.

 The regret term $R_0$, which is the regret from the warm-up period of the first $1/\nu_T^2$ steps, would have the same definition and the same regret bound of $\tilde{O}(T^{2/3})$ as in the analysis of Algorithm \ref{alg:cap}.
   
    To bound the regret term $R_1$, we would need to show that $C_{\mathrm{alg}}^s$ as described above does not have much more expected cost than the true best controller, $C_{K_{\mathrm{opt}}(\theta^*, T, D^{\E[x]}, D^u)}^{\theta^*}$. This can be incorporated into an analogue of Assumption \ref{parameterization_assum2}: assuming that for $\norm{\theta - \theta^*}_{\infty}, \norm{D^u - D^{u'}}_{\infty}, \norm{D^{\E[x]} - D^{\E[x]'}}_{\infty}$ all sufficiently small,
    \begin{align*}
        &|\bar{J}(\theta^*,C_{K_{\mathrm{opt}}(\theta, T, D^{\E[x]'}, D^{u'})}^\theta,t) - \bar{J}(\theta^*,C_{K_{\mathrm{opt}}(\theta^*, T, D^{\E[x]}, D^u)}^{\theta^*},t)| \\
        &= \tilde{O}_T\left(\norm{\theta - \theta^*}_{\infty} + \norm{D^{\E[x]} - D^{\E[x]'}}_{\infty}  + \norm{D^{u} - D^{u'}}_{\infty} + \frac{1}{T^2}\right).
    \end{align*}
    This can be made into a new assumption on the baseline class of controllers that replaces  Assumption \ref{parameterization_assum2}.

We expect that the regret source $R_2$ (converting from expected regret to realized regret) will still be $\tilde{O}_T(\sqrt{T})$, as this was a result of a concentration inequality that will still apply.
 
Regret $R_3$ no longer exists as we no longer use the controls $u_t^{\mathrm{safeU}}$ or $u_t^{\mathrm{safeL}}$, and instead this source of regret is being incorporated into the chosen $C^{\mathrm{alg}}_s$ in regret term $R_1$.

To summarize, the main modification to the algorithm would be the choice of controller $C_s^{\mathrm{alg}}$, and the main change to the proof is moving the burden of bounding the regret term $R_3$ to the version of Assumption \ref{parameterization_assum2} described above that accounts for the tightened constraint arguments to $K_{\mathrm{opt}}$.

\subsubsection{Theorem \ref{sufficiently_large_error} and Algorithm \ref{alg:cap_large}}\label{app:control_theorem2}

In order to show a version of Theorem \ref{sufficiently_large_error} that works for control constraints, Algorithm \ref{alg:cap_large} would need the same modifications as described for Algorithm \ref{alg:cap}. Specifically, instead of using controls $u_t^{\mathrm{safeU}}$ and $u_t^{\mathrm{safeL}}$, the controller $C^{\mathrm{alg}}_s$ is chosen as $C_s^{\mathrm{alg}} = C_{K_{\mathrm{opt}}(\hat{\theta}_s, T_s, D^{\E[x]'}, D^{u'})}^{\hat{\theta}_s}$.

The main way that the proof of regret for Algorithm \ref{alg:cap_large} differs from the regret for Algorithm \ref{alg:cap} is that the proof for Algorithm \ref{alg:cap_large} relies on the faster rate of convergence for $\hat{\theta}_s$ given by Lemma \ref{boundary_uncertainty}. Proving a form of Lemma \ref{boundary_uncertainty} for the modified algorithm would be the main additional step in proving that $\tilde{O}_T(\sqrt{T})$ regret is possible with control constraints. As discussed in the proof sketch of Theorem \ref{sufficiently_large_error}, the proof of Lemma \ref{boundary_uncertainty} comes from the fact that a constant fraction of the time, $u_t^{\mathrm{safeU}}$ is non-linear by an amount larger than a positive constant. The non-linearity of $u_t^{\mathrm{safeU}}$ occurs because enforcing safety constraint satisfaction requires non-linear controls. While the modified controller $C_s^{\mathrm{alg}}$ described in the previuos paragraph does not use the non-linear controls $u_t^{\mathrm{safeU}}$, $C_s^{\mathrm{alg}}$ must still be frequently non-linear in order to satisfy the safety constraints. Therefore, we expect that for noise distributions with large enough support, the modified Algorithm \ref{alg:cap_large} will a constant fraction of the time use a control that is non-linear by a constant amount, which will give that $\epsilon_s$ decreases at a rate of $1/\sqrt{t}$.

\subsection{Higher Dimensions}\label{sec:highd}
This work focuses on the one-dimensional LQR setting, but many LQR applications have higher dimensional positions and controls. We leave the formal extension of our results to higher dimensions for future work, but discuss here when and how we believe our results will extend to higher dimensions. Suppose $x_t \in \mathbb{R}^{n}$ and $u_t \in \mathbb{R}^m$, which implies that the dynamics are a pair of matrices $\theta^* = (A^*,B^*)$ where $A^* \in \mathbb{R}^{n \times n}$ and $B^* \in \mathbb{R}^{n \times m}$. A natural extension of our constraints to higher dimensions is to consider a (origin-containing) polytopal constraint, i.e., the intersection of a finite number of half-spaces that contain the origin. Specifically, we could consider constraints of the form $\Delta(A^*x_t + B^*u_t) \le d$ where $\Delta \in \mathbb{R}^{k \times n}$ and $d \in \mathbb{R}^k$. This still has the interpretation as the expected position at each time is within the convex region $\{x \in \mathbb{R}^n : \Delta x \le d\}$. Analogous to in Appendix \ref{sec:control_constraints}, we define the function $K_{\mathrm{opt}}(\theta, T, \Delta,d)$ as choosing the optimal parameter $K$ for a controller satisfying the constraints $\Delta(Ax_t + Bu_t) \le d$. Before talking about specific algorithms, we first note that we expect that the results of Lemmas \ref{v_to_use} and \ref{initial_uncertainty} generalize directly to higher dimensions. This is necessary for all of our algorithmic results. Note that because the dynamics are matrices, the dynamics estimates will also be matrices denoted $\hat{\theta}_s$.

\subsubsection{Theorem \ref{performance} and Algorithm \ref{alg:cap}}\label{app:highd_theorem1}

In higher dimensions, Assumption \ref{weak_depend} becomes slightly more complicated. Specifically, we define the truncated version of a controller $C$ in higher dimensions as using either control $C(x)$ if $C(x)$ would result in an expected position inside the convex safe region, and otherwise using the smallest magnitude control that takes the position in expectation to inside of the convex safe region. The other assumptions have direct higher dimensional counterparts.

The key modification of Algorithm \ref{alg:cap} is to choose the controller $C_s^{\mathrm{alg}}$ in such a way that $\Delta(\hat{A}_sx_t + \hat{B}_sC_s^{\mathrm{alg}}(x_t)) \le d - \tilde{\Theta}_T(\epsilon_s)$.  In other words, choosing $C_s^{\mathrm{alg}} = C^{\hat{\theta}_s}_{K_{\mathrm{opt}}(\hat{\theta}_s, T_s, \Delta, d - \tilde{\Theta}_T(\epsilon_s))}$. Within each iteration of the main loop of Algorithm \ref{alg:cap}, the algorithm can directly use $C_s^{\mathrm{alg}}$ without the need for $u_t^{\mathrm{safeU}}$ or $u_t^{\mathrm{safeL}}$.  By this construction, $\Delta(\hat{A}_sx_t + \hat{B}_sC_s^{\mathrm{alg}}(x_t)) \le d - \tilde{O}_T(\epsilon)$. Because with high probability $\norm{\hat{\theta}_s - \theta^*}_{\infty} \le \tilde{O}_T(\epsilon_s)$, this will imply that $\Delta(A^*x_t + B^*C_s^{\mathrm{alg}}(x_t)) \le d$ with high probability. This in turn means that the algorithm will satisfy the constraints with high probability.

Analyzing the regret of this algorithm, the regret terms $R_0$, $R_1$, and $R_2$ stay the same as in the proof of Theorem \ref{performance}. The regret term $R_3$ is no longer needed, as we no longer use controls $u_t^{\mathrm{safeU}}$ or $u_t^{\mathrm{safeL}}$. To bound the regret term $R_1$, we want to show that the cost of $C_{\mathrm{alg}}^s$ is close to the cost of $C_{K_{\mathrm{opt}}(\theta^*, T, \Delta,d)}^{\theta^*}$. Like we did in Appendix \ref{sec:control_constraints}, we need an analogue of  Assumption \ref{parameterization_assum2}, which is that for $\norm{\theta - \theta^*}_{\infty}$ and $\norm{d - d'}_{\infty}$ both sufficiently small,
    \begin{align*}
        &|\bar{J}(\theta^*,C_{K_{\mathrm{opt}}(\theta, T, \Delta, d')}^\theta,t) - \bar{J}(\theta^*,C_{K_{\mathrm{opt}}(\theta^*, T, \Delta, d)}^{\theta^*},t)| \\
        &= \tilde{O}_T\left(\norm{\theta - \theta^*}_{\infty} + \norm{d - d'}_{\infty}  + \frac{1}{T^2}\right).
    \end{align*}
 By similar arguments as in our current proof, we expect this assumption will be sufficient to bound $R_1$ for this modified algorithm. We expect that the bound on $R_2$ would be very similar as in the proof of Theorem \ref{performance}, as this regret term corresponds to concentration of the cost. Similarly, the regret term $R_0$ can also be bounded the same as in the proof of Theorem \ref{performance}, as this term corresponds to the warm-up period which still has length $\tilde{O}(T^{2/3})$. Therefore, we expect that the total regret of this modified algorithm can still be bounded by $\tilde{O}(T^{2/3})$.

\subsubsection{Theorem \ref{sufficiently_large_error}}\label{app:highd_theorem2}

We leave whether or not Theorem \ref{sufficiently_large_error} generalizes to higher dimensions in all situations as an open question. However, we will briefly outline a setting in which we do expect the result to generalize. Suppose that $m = n$ and that $A^*$ and $B^*$ are invertible and diagonalizable. Algorithm \ref{alg:cap_large} for higher dimensions would require the same changes as in the previous subsubsection, which means that $C_s^{\mathrm{alg}} = C^{\hat{\theta}_s}_{K_{\mathrm{opt}}(\hat{\theta}_s, T_s, \Delta, d - \tilde{\Theta}_T(\epsilon_s))}$. The main new result that would be necessary is an analogue of Lemma \ref{boundary_uncertainty} for higher dimensions. Intuitively, the result of Lemma \ref{boundary_uncertainty} holds because Algorithm \ref{alg:cap_large} will a constant fraction of the time use the non-linear control $u_t^{\mathrm{safeU}}$ which allows for faster learning. The analogue for higher dimensions is to show that the modified algorithm will a constant fraction of the time use a non-linear control. A difficulty in higher dimension is that it is not sufficient to just be non-linear along one dimension. Instead, there must be sufficient non-linearity in all $m$ dimensions. Therefore, the higher dimensional version of Assumption \ref{assum_sufficiently_large_error} requires that the noise distribution is sufficiently large relative to the constraints in all $m$ dimensions, which for example would be satisfied by the multivariate normal distribution with mean $0$ and constant variance matrix. Under this assumption, the modified algorithm will a constant fraction of the time use controls $u_t$ that satisfy $\Delta_i (A^*x_t + B^*u_t) \ge d_i -O_T(\epsilon_s)$ for some $i \in [1:k]$. Furthermore, if the noise is sufficiently large in all dimensions, then we expect that for every side of the boundary of the convex compact region (corresponding to $\Delta_i$ and $d_i$ for $i \in [1:k]$), $x_t$ will at times be sufficiently far from that side and a point on that side will be the closest point to $x_t$. Because $A^*$ is invertible, the previous sentence will also hold for $A^*x_t$. Because $B^*u_t$ must bring the position back to within the safe region in expectation, for every side of the boundary we must have that $B^*u_t$ is large and perpendicular to that side. Because $B^*$ is invertible, this implies that the $u_t$ used to enforce safety will be sufficiently non-linear in all directions. We believe this would allow the algorithm to learn the matrix $B^*$ up to accuracy $O_T(1/\sqrt{t})$ at time $t$. Equipped with an analogue of Lemma \ref{boundary_uncertainty}, we expect that the rest of the proof will follow directly. If $m > n$ or $A^*$ and $B^*$ are not invertible, then showing that the non-linear controls $u_t$ are sufficient for learning every column of the matrix $B^*$ is more difficult. We leave the details of analyzing this case for future work.

\end{document}